\documentclass{article}

\usepackage[preprint]{neurips_2026}

\usepackage[utf8]{inputenc} 
\usepackage[T1]{fontenc}    
\usepackage{hyperref}       
\usepackage{url}            
\usepackage{booktabs}       
\usepackage{amsfonts}       
\usepackage{nicefrac}       
\usepackage{microtype}      
\usepackage{xcolor}         

\usepackage{amsmath}
\usepackage{amssymb}
\usepackage{mathtools}
\usepackage{amsthm}
\usepackage{subfigure}
\usepackage{caption}
\usepackage{subcaption}
\usepackage{algorithmic}
\usepackage{algorithm}

\theoremstyle{plain}
\newtheorem{theorem}{Theorem}[section]
\newtheorem{proposition}[theorem]{Proposition}
\newtheorem{lemma}[theorem]{Lemma}

\theoremstyle{definition}
\newtheorem{definition}[theorem]{Definition}
\newtheorem{assumption}[theorem]{Assumption}
\theoremstyle{remark}
\newtheorem{remark}[theorem]{Remark}
\newtheorem{claim}[theorem]{Claim}

\title{Large Dimensional Kernel Ridge Regression: Extending to Product Kernels}

\author{%
  Yang Zhou \\
  Department of Statistics and Data Science\\
  Tsinghua University\\
  Beijing, 100084, China\\
  \texttt{yzhou24@mails.tsinghua.edu.cn} \\
   \And
   Yicheng Li \thanks{Yang Zhou and Yicheng Li contributed equally to this work.} \\
   Department of Statistics and Data Science\\
  Tsinghua University\\
  Beijing, 100084, China\\
  \texttt{liyc22@mails.tsinghua.edu.cn}
  \And
   Yuqian Cheng \\
   Department of Mathematical Science\\
  Tsinghua University\\
  Beijing, 100084, China\\
  \texttt{cyq21@mails.tsinghua.edu.cn}
   \And
  Qian Lin \thanks{corresponding author}\\
  Department of Statistics and Data Science\\
  Tsinghua University\\
  Beijing, 100084, China\\
  \texttt{qianlin@tsinghua.edu.cn} 
}

\begin{document}

\maketitle

\begin{abstract}
  Recent studies have reported \textit{saturation effects} and \textit{multiple descent behavior} in large dimensional kernel ridge regression (KRR). However, these findings are predominantly derived under restrictive settings, such as inner product kernels on  sphere or strong eigenfunction assumptions like hypercontractivity. Whether such behaviors hold for other kernels remains an open question. In this paper, we establish a broad, new family of large dimensional kernels and derive the corresponding convergence rates of the generalization error. As a result, we recover key phenomena previously associated with inner product kernels on sphere, including: $i)$ the \textit{minimax optimality} when the source condition $s\le 1$; $ii)$ the \textit{saturation effect} when $s>1$; $iii)$ a $\textit{periodic plateau phenomenon}$ in the convergence rate and a $\textit {multiple-descent behavior}$ with respect to the sample size $n$. 
\end{abstract}

\section{Introduction}

The recent success of deep learning has brought the traditional kernel methods (such as kernel ridge regression(KRR)) back to spotlight, particularly through the connection between wide neural networks and neural tangent kernels (NTKs) established by \cite{jacot2018_NeuralTangent}. Under the fixed dimensional setting, the behavior of KRR is well-characterized under standard assumptions (embedding index, source condition, etc.) by \cite{caponnetto2007optimal, fischer2020_SobolevNorm}. Specifically, \cite{fischer2020_SobolevNorm,zhang2023optimality_2} derived the minimax optimal rate $n^{-\frac{s\beta}{s\beta+1}}$ of KRR when the source condition $0<s\le2$; \cite{bauer2007_RegularizationAlgorithms,li2023_SaturationEffect} reported the saturation effect when $s>2$. Also, a line of work (\cite{simon2023eigenlearning}) considered the eigenframework, deriving conservative equations to generalization metrics of KRR. However, modern applications increasingly involve large dimensional data, where the dimension $d$ ranges from thousands to millions. This shift towards large dimensional applications calls for a thorough understanding of how kernel methods perform when both the sample size $n$ and dimension $d$ grow, particularly in the large dimensional regime where $n \asymp d^\gamma$ for some $\gamma > 0$.

In the large dimensional regime, several new phenomena emerge, including the \textit{saturation effects} when $s>1$, the \textit{periodic plateau phenomenon} and the \textit{multiple descent behavior} (see Figure \ref{figure 5 rates w.r.t d},\ref{figure 3 rates w.r.t n} for details). Specifically, \cite{Ghorbani2019LinearizedTN} identified approximation barriers for large-dimensional kernel regression, while \cite{lu2023optimal} and \cite{zhang2024optimal} derived minimax optimal convergence rates for inner product kernels on the sphere and saturation effects in different regime, revealing intriguing phenomena such as periodic plateau and multiple descent behaviors in the excess risk.  Also, a benign overfitting behavior arises in the large dimensional setting, which has been discussed by \cite{misiakiewicz_spectrum_2022, barzilai2023generalization,zhang2025phase}.

Despite these advances, existing studies under large dimensional setting mostly focus on the  inner product kernels defined on sphere, leaving other commonly used kernels (such as Gaussian kernel) largely unexplored. The scarcity of theoretical understanding for kernels beyond inner product kernels on sphere prompts a fundamental question: \textit{how do other general kernels behave in the large dimensional setting?} In this paper, we extend the applicable kernels to a wider class by investigating the generalization error of KRR in large dimensional, general domains, and  establish exact convergence rates without restrictive eigenfunction assumptions. As a result, we demonstrate that phenomena previously observed for spherical kernels  holds for a large group of large dimensional kernels, extending previous work to a more general class of kernels.

\subsection{Contributions and Related Works}\label{our contributions}

In this paper, we perform an investigation on the kernel ridge regression in large dimensional setting (i.e., $n\propto d^{\gamma}$ for some $\gamma>0$) where  the kernel may be defined on (possibly) unbounded domains. 

\textbf{Reproducing previously observed phenomena on large dimensional product kernel ridge regression}
Previous work on large dimensional kernels mostly focuses on inner product kernels on sphere (e.g., \cite{lu2023optimal},\cite{zhang2024optimal}).
As a result, the following phenomena were reported (see Figure \ref{figure 5 rates w.r.t d}, \ref{figure 3 rates w.r.t n} for details): the \textit{minimax optimality} when $s\le 1$; the \textit{saturation effect} when $s>1$; a \textit{periodic plateau phenomenon} and a \textit{multiple descent behavior}. Although some studies (e.g., \cite{misiakiewicz2024non}) try to extend to other kernels defined on domains beyond the sphere, these studies universally impose explicit assumptions ( {\it e.g.  hypercontractivity}) on the properties of the kernel's eigenfunctions (see Remark \ref{remark miska} for details).  While such assumptions are known to hold in specific cases (such as kernels on the hypercube, kernels on the sphere, or settings where eigenfunctions are polynomials under Gaussian measure), their applicability to other general kernels remains unverified. Recently, \cite{pandit2024universality} considered large dimensional KRR on general domains under the setting $n\asymp d^2$, which is among the first to touch on the problems of general kernels in large dimensions.

In order to study the behaviors of general kernels in large dimensional setting, we provide a \textit{new, broad} family of kernels: the \textit{large dimensional product kernels} (see Assumption \ref{product structure assumption} for details). This family of kernels not only encompasses widely used kernels like the Gaussian kernel, but also generalizes beyond the restrictive setting of the sphere, accommodating a broader range of input spaces and a wider application in reality. As a result, we reproduce the phenomena previously observed for inner product kernels on sphere: the \textit{minimax optimality}, \textit{saturation effects}, \textit{periodic plateau phenomenon} and \textit{multiple descent behavior}.



\textbf{Relief on constraints of eigenfunctions} The reliance on specific eigenfunction properties in prior work (e.g., spherical harmonic polynomials, hypercontractivity) limits the generality of those results, preventing their extension to broader kernel classes. On the contrary, we propose an alternative approach that solely depends on eigenvalues rather than eigenfunctions. This enables us to apply our results to a wide class of kernels with explicit eigenvalues, encompassing both the classical inner product kernels on sphere and the newly introduced family of large dimensional product kernels.

\textbf{Explicit examples of product kernels}
In Appendix \ref{sec:example}, we provide several explicit kernels in the family of product kernels, including: Gaussian kernels on general Gaussian measure; Mehler kernels on general Gaussian measure; Laguerre kernels on $(0,\infty)^d$ and kernels with mixed components. The eigenfunctions of the above examples are generally product of Hermite polynomials, exponential functions and Laguerre polynomials. The significant divergence in eigenfunction structures prevents the simple verification of uniform properties associated with spherical harmonics or hypercontractivity in previous literature, emphasizing the importance of removing the dependency on specific eigenfunction properties.

\textbf{Notations} Let $\mathcal{X}\subseteq \mathbb{R}^d$ be the input space, and $\mathcal{Y}\subseteq \mathbb{R}$ be the output space. Let $ \rho = \rho_{d} $ be an unknown probability distribution on $\mathcal{X} \times \mathcal{Y}$ satisfying $ \int_{\mathcal{X} \times \mathcal{Y}} y^{2} \mathrm{d}\rho({x},y) <\infty$ and denote the corresponding marginal distribution on $ \mathcal{X} $ as $\mu = \mu_{d}$. For the sake of conciseness, we may use $L^2$ as the abbreviation of $L^{2}(\mathcal{X},\mu)$, and $\|\cdot\|_{\mathcal{H}}$, $\|\cdot\|_{L^2}$, $\|\cdot\|_{L^{\infty}}$,  $\|\cdot\|_{2}$ denotes the norm of Hilbert space $\mathcal{H}$, the norm of $L^{2}(\mathcal{X},\mu)$, the norm of $L^{\infty}(\mathcal{X},\mu)$, the Frobenius norm, respectively.

Throughout the paper, we shall use the asymptotic notations $o(\cdot),O(\cdot),\omega(\cdot),\Omega(\cdot)$. Also, we shall define the following notations: $a\lesssim b$ if and only if $a=O(b)$; $a\gtrsim b$ if and only if $b=O(a)$; $a\asymp b$ (or $a=\Theta(b)$) if and only if $a\lesssim b$ and $a\gtrsim b$. We may also define the notations $o_\mathbb{P}(\cdot)$, $O_\mathbb{P}(\cdot)$ as follows. $a_n=o_\mathbb{P}(b_n)$ if and only if $a_n/b_n$ converges to $0$ in probability; $a_n=O_\mathbb{P}(b_n)$ if and only if for any $\epsilon>0$, there exist constants $C_\epsilon, N_\epsilon$ such that for all $n\geq N_\epsilon$, $P(|a_n|>C_\epsilon|b_n|)<\epsilon$.

\section{Preliminaries}
\label{preliminaries}
Suppose that $\{ (x_{i}, y_{i}) \}_{i=1}^{n}$ are i.i.d. sampled from the model $y=f_{\rho}^*(x)+\epsilon$, where $f_{\rho}^*(x)$ is the true function, while $\epsilon$ is the noise. Throughout the paper, we shall consider the large dimensional setting, where $n\asymp d^\gamma, \gamma>0$. We consider the following assumptions.
\begin{assumption}\label{assumption kernel}
    Suppose that $ \mathcal{H} = \mathcal{H}_{d} $ is a separable RKHS on $\mathcal{X} \subset \mathbb{R}^{d}$ with respect to a continuous kernel function $ k = k_{d}$ satisfying 
$$        \sup\limits_{{x} \in \mathcal{X}} k_{d}({x},{x}) \le \kappa^{2},$$ 
    where $ \kappa$ is an absolute constant, which does not depend on the dimension $d$. In the remaining paper, we may omit the subscript $d$ without causing any ambiguity.
\end{assumption}

\begin{remark}
    In most kernel related literature, (see e.g., \cite{Caponnetto2006OptimalRF, caponnetto2007optimal, raskutti2014early, beaglehole2023inconsistency,
buchholz2022_KernelInterpolation, lai2023generalization, li2023kernel}), the input space $\mathcal{X}$ is assumed to be bounded. However, such assumptions are not made in this paper. The relieved restriction on the input space $\mathcal{X}$ allows us to apply our results to more general kernels.
\end{remark}

By the celebrated Mercer's theorem (see,e.g., \cite{steinwart2012_MercerTheorem}), when Assumption \ref{assumption kernel} holds, the kernel $k$ satisfies $k(x,x^\prime)=\sum_{i \in N}\lambda_ie_i(x)e_i(x^\prime),$ where $N$ is an at most countable set, $\{\lambda_i, i=0,1,\cdots \}$ are the non-increasing eigenvalues of $k$, while $\{e_i(x),i=0,1,\cdots\}$ are the corresponding eigenfunctions. 

Next, we shall introduce the interpolation space of $\mathcal{H}.$ For any $s\geq0$, the interpolation space $[\mathcal{H}]^s$ can be defined as 
\[
[\mathcal{H}]^s:=\left\{ \sum_{i\in N}\lambda_i^{s/2}a_ie_i|\sum_{i\in N}a_i^2<\infty\right\},
\]
with the inner product deduced from
$
\langle \lambda_i^{s/2}e_i, \lambda_j^{s/2}e_j \rangle_{[\mathcal{H}]^s}=\delta_{ij}.
$
It is easy to show that $[\mathcal{H}]^s$ is also a separable Hilbert space. Moreover, if we assume $s=1$ or $s=0$, the interpolation space norm $\left\|\cdot\right\|_{[\mathcal{H}]^s}$ will be reduced to $\left\|\cdot\right\|_{\mathcal{H}}$ and $\left\|\cdot\right\|_{L^2}$ respectively.

Kernel Ridge Regression (KRR) aims at minimizing the quadratic cost with a penalty term. Specifically, the estimator $\hat{f}_{\lambda}$ of KRR is defined by the following equation:
\begin{equation}\label{eq:krr}
    \hat{f}_\lambda = \underset{f \in \mathcal{H}}{\arg \min } \left(\frac{1}{n} \sum_{i=1}^n\left(y_i-f\left( {x}_{i} \right)\right)^2+\lambda\|f\|_{\mathcal{H}}^2\right),
\end{equation}
where $\lambda > 0$ is the regularization parameter. 

Denote the samples as $\boldsymbol{X}=\left( {x}_{1},\cdots,{x}_{n} \right)$ and $\boldsymbol{y}=\left( y_{1},\cdots,y_{n} \right)^{\top}$. The explicit expression of the KRR estimator $\hat{f}_{\lambda}$ is given by the representer theorem (see, e.g., \citealt{steinwart2008support}):
\begin{equation}\label{krr estimator}
   \hat{f}_{\lambda}({x}) = \mathbb{K}({x}, \boldsymbol{X})(\mathbb{K}(\boldsymbol{X}, \boldsymbol{X})+n \lambda {I})^{-1} \boldsymbol{y},
\end{equation}
where $$  \mathbb{K}(\boldsymbol{X}, \boldsymbol{X})=\left(k\left({x}_i, {x}_j\right)\right)_{n \times n},~~ \mathbb{K}({x}, \boldsymbol{X})=\left(k\left({x}, {x}_1\right), \cdots, k\left({x}, {x}_n\right)\right).$$

Throughout the paper, we are interested in the generalization error (excess risk) of $\hat{f}_{\lambda}$: 
\begin{equation}\label{eq:gerneralization error}
    \mathbb{E}_{{x} \sim \mu} \left[ \left( \hat{f}_{\lambda}({x}) - f_{\rho}^{*}({x}) \right)^{2} \right] = \left\|\hat{f}_\lambda-f_\rho^*\right\|_{L^2}^2,
\end{equation}
and our aim is to derive the exact convergence rate of the generalization error of KRR.

\section{Main Results}
\label{main results}
In this section, we shall provide the exact  convergence rate of the generalization error of $\hat{f}_\lambda$.
Before giving the main result, we shall introduce the following quantities, which will be frequently used in the following paper.
\begin{definition}
Given $\{\lambda_i, i=0,1,\cdots \}$ are the non-increasing eigenvalues of $k$, denoting $f_{\rho}^{*} = \sum\limits_{i=1}^{\infty} f_{i} e_{i}(x) \in L^{2}(\mathcal{X},\mu)$, we shall define the following quantities.
\begin{equation}\label{n1 n2 m1 m2}
\mathcal{N}_{k}(\lambda):=\sum_{i\in N}\left(\frac{\lambda_i}{\lambda+\lambda_i}\right)^k,\ \ \mathcal{R}_2(\lambda):=\sum_{i \in N}\left(\frac{\lambda}{\lambda_i+\lambda}f_i\right)^2,\ \ k=1,2.
\end{equation}
\end{definition}

Also, we need the following assumptions, which are often used in kernel-related literature.
\begin{assumption}\label{assumption noise}
    Suppose that the noise $\epsilon$ satisfies $\operatorname{Var}(\epsilon)=\sigma_\epsilon^2<\infty.$
\end{assumption}

\begin{assumption}\label{assumption f*}
    Suppose that there exist $s$ and an absolute constant $R_1$ such that $f_\rho^*$ satisfies $\|f_\rho^*\|_{[\mathcal{H}]^s}\leq R_1.$ 
\end{assumption}

Now we are prepared to give one of the main results of this paper.

\begin{theorem}\label{main theorem}
    When $s\ge 1$, suppose that Assumption \ref{assumption kernel}, \ref{assumption noise}, \ref{assumption f*}  hold, and there exist absolute constants $c_1,c_2$ such that $c_1<\lambda_0<c_2$ regardless of $d$. Furthermore, assume that the following conditions hold for some $\lambda=\lambda(d,n)\rightarrow0$:
\begin{align}\label{condition}
    \frac{\mathcal{N}_{1}(\lambda)}{n} \ln{n} = o(1); ~~\frac{1}{n\lambda}=o(1) ; ~~\mathcal{N}_2(\lambda)=\Omega(1); ~~n^{-1} \frac{1}{\lambda^2} \ln{n} = o(\mathcal{N}_{2}(\lambda)).
\end{align}

    Then we have the following equation holds:
    \begin{equation}
    \label{main bound}
        \mathbb{E}[\|\hat{f}_\lambda-f_\rho^*\|_{L^2}^2|\boldsymbol{X}]=\Theta_{\mathbb{P}}\left(\frac{\sigma_\epsilon^{2} \mathcal{N}_{2}(\lambda)}{n}+\mathcal{R}_{2}(\lambda) \right).
    \end{equation}
\end{theorem}

    \textbf{Elaboration of \eqref{main bound}} Theorem \ref{main theorem} provides the precise order of the generalization of KRR when $s\ge 1$. In \eqref{main bound},  $\frac{\sigma_\epsilon^{2} \mathcal{N}_{2}(\lambda)}{n}$ represents the variance term of the generalization error, while $ \mathcal{R}_{2}(\lambda)$ represents the bias term. 
    The tradeoff between the variance term and bias term urges us to choose appropriate $\lambda$. In the following paper, we shall derive the appropriate $\lambda$ when considering large dimensional product kernels (see Theorem \ref{thm:upper bound gaussian},\ref{theorem s<1 gaussian} for details). 


\begin{theorem}\label{main theorem s<1}
    When $0<s<1$, suppose that Assumption \ref{assumption kernel}, \ref{assumption noise}, \ref{assumption f*}  holds. Assume that the following conditions hold for some $\lambda=\lambda(d,n)\rightarrow0$:
\begin{align}\label{condition s<1}
    \frac{\mathcal{N}_{1}(\lambda)}{n} \ln{n} = o(1); ~~\frac{1}{n\lambda}=o(1) ; ~~\mathcal{N}_2(\lambda)=\Omega(1); ~~n^{-1} \frac{1}{\lambda^2} \ln{n} = o(\mathcal{N}_{2}(\lambda)).
\end{align}
Furthermore, define $f_\lambda=\sum_{i \in N}\frac{\lambda_i}{\lambda_i+\lambda}f_ie_i$, if there exists $\epsilon>0$ such that 
\begin{align}\label{add condition s<1}
    \frac{\sqrt{\frac{1}{\lambda}}(n^{\frac{1-s}{2}+\epsilon}+\|f_{\lambda}\|_{L^{\infty}})}{n}=o(\frac{1}{\sqrt{n}}\mathcal{N}_2(\lambda)^{\frac{1}{2}}+\mathcal{R}_2(\lambda)^{\frac{1}{2}}).
\end{align}
    Then we have the following equation holds:
    \begin{equation}
    \label{main bound s<1}
        \mathbb{E}[\|\hat{f}_\lambda-f_\rho^*\|_{L^2}^2|\boldsymbol{X}]=\Theta_{\mathbb{P}}\left(\frac{\sigma_\epsilon^{2} \mathcal{N}_{2}(\lambda)}{n}+\mathcal{R}_{2}(\lambda) \right).
    \end{equation}
\end{theorem}
\begin{remark}
Notice that the most significant distinction between Theorem \ref{main theorem s<1} and Theorem \ref{main theorem} lies in the inclusion of condition \ref{add condition s<1}. This requirement arises because $f_\rho^*$ no longer belongs to $\mathcal{H}$ when $s<1$, thereby complicating the bounding procedure of the bias term.
\end{remark}

\textbf{An approach to bound $\|f_{\lambda}\|_{L^{\infty}}$} Unlike Theorem \ref{main theorem}, we notice that $\|f_{\lambda}\|_{L^{\infty}}$ in Theorem \ref{main theorem s<1} depends on the eigenfunctions $\{e_i\}_{i \in N}$. Thankfully, we are able to provide an upper bound for $\|f_{\lambda}\|_{L^{\infty}}$. Notice that $f_\lambda \in \mathcal{H}$, we can bound $\|f_{\lambda}\|_{L^{\infty}}$ by $\|f_{\lambda}\|_{\mathcal{H}}$, while the latter one can be expressed as $(\sum_{i=0}^{\infty}\frac{\lambda_i}{(\lambda_i+\lambda)^2}f_i^2)^{\frac{1}{2}}$, which is independent of eigenfunctions. (Notice that $\|f_\rho^*\|_{[\mathcal{H}]^s}\leq R_1$, $(\sum_{i=0}^{\infty}\frac{\lambda_i}{(\lambda_i+\lambda)^2}f_i^2)^{\frac{1}{2}}$ can be further bounded by functions of $\{\lambda_i\}_{i \in N}$ and $\lambda$, for a detailed upper bound of $\|f_{\lambda}\|_{L^{\infty}}$, see Lemma \ref{lemma calculation flambda}.) Hence, we are able to claim that the conditions in our Theorem \ref{main theorem s<1} can be easily verified without properties of eigenfunctions.

\begin{remark}\label{remark:difference of main theorem}
    We notice that Theorem 1 in \cite{zhang2024optimal} provides a result similar to our results Theorem \ref{main theorem},\ref{main theorem s<1}. The most important difference is that in our assumptions and conditions, we deliberately avoid reliance on eigenfunctions. On the contrary, Theorem 1 in \cite{zhang2024optimal} highly relies on the properties of eigenfunctions, such as the quantity $\mathcal{M}_{1}(\lambda) = \operatorname*{ess~sup}_{{x} \in \mathcal{X}} \left|\sum\limits_{i=1}^{\infty} \left( \frac{\lambda}{\lambda_{i} + \lambda} f_{i} e_{i}({x}) \right) \right|$ and the Assumption 3 in their paper. Although \cite{zhang2024optimal} verified their assumptions and conditions for spherical harmonics as eigenfunctions, such constraints may not hold for other general kernels. 
\end{remark}

\begin{remark}\label{remark miska}
    Recently, \cite{misiakiewicz2024non} proposed an approximation of the generalization error via a deterministic equivalent $\mathrm{R}_n\left(\boldsymbol{\beta}_*, \lambda\right)=\frac{\lambda_*^2\left\langle\boldsymbol{\beta}_*,\left(\boldsymbol{\Sigma}+\lambda_*\right)^{-2} \boldsymbol{\beta}_*\right\rangle+\sigma_{\varepsilon}^2}{1-\frac{1}{n} \operatorname{Tr}\left(\boldsymbol{\Sigma}^2\left(\boldsymbol{\Sigma}+\lambda_*\right)^{-2}\right)}$, which appears to be valid in the large dimensional setting. Here $\boldsymbol{\Sigma}$ denotes the covariance operator, $\beta^*$ denotes the corresponding coefficients of $f^*$ under the basis $\{e_i\}$'s, $\lambda_*$ satisfies $n-\frac{\lambda}{\lambda_*}=\operatorname{Tr}\left(\boldsymbol{\Sigma}\left(\boldsymbol{\Sigma}+\lambda_*\right)^{-1}\right)$. While this approach offers a valuable theoretical tool for analyzing generalization performance, its applicability relies critically on certain structural assumptions—notably, a hypercontractivity condition (outlined in their Assumptions 3 and 4). This condition requires uniform bounds on higher-order moments of the eigenfunctions ${e_i}$ relative to their second moments, which can be challenging to verify for many commonly used kernels. As a result, even for the inner product kernels on sphere, assumptions in \cite{misiakiewicz2024non} can only be verified in the case $s=0$.
\end{remark}

\section{Applications to Large Dimensional Product KRR}
\label{applications}
In this section, we consider a family of product kernels, whose eigenvalues exhibit staircase descent properties (illustrated in Proposition \ref{prop:staircase}), while eigenfunctions do not generally possess any distinctive properties. As a result, such kernels possess exactly same convergence rates, and phenomena such as saturation effects, periodic plateau phenomenon and multiple descent behavior also exist.  

We assume that the kernel function satisfies the following assumption: 
\begin{assumption}[Large Dimensional Product Kernel]\label{product structure assumption} 
    Assume that the input space $(\mathcal{X},\mu)$ is a product measure space 
$$    (\mathcal{X},\mu) = \prod_{i=1}^{d} (\mathcal{X}_i, \mu_i).$$
For each $i$, $\mathcal{X}_i$ is a subdomain of $\mathbb{R}$ and $\mu_i$ is a probability measure on $\mathcal{X}_i$. We further assume the kernel $k_d: \mathcal{X} \times \mathcal{X} \to \mathbb{R}$ has a product structure:
\begin{equation}\label{product structure formula}
    k_d(x, x') = \prod_{i=1}^{d} \hat{k}_{r_i, i}(x_i, x'_i),
\end{equation}
where, for each index $i$, $\hat{k}_{r_i, i}$ is a kernel defined on $(\mathcal{X}_i, \mu_i)$ with an explicit Mercer's decomposition. Its eigenvalues $\mu^{r_i,i}_{j} , j = 0, 1, 2, \dots$, exhibit exponential decay with parameter $r_i$,
    \begin{equation}\label{eq:eigendecay_1d}
        c_i \cdot (r_i)^{-j} \leq \mu^{r_i,i}_{j} \leq C_i \cdot (r_i)^{-j}.
    \end{equation}
Moreover, we assume that the parameter $r_i$ satisfies $\frac{c}{d}\leq r_i\leq\frac{C}{d}$ for some absolute constants $c,C>0$, and  constants $c_i$, $C_i$ in \eqref{eq:eigendecay_1d} satisfy 
    \begin{equation}\label{constant bound}
        0<c'\leq\prod_{i=1}^d c_i\leq\prod_{i=1}^d C_i\leq C'<\infty
    \end{equation} 
    for some absolute constants $c',C'>0$. 
\end{assumption}

\begin{remark} 
        The parameter $r_i$ in \eqref{eq:eigendecay_1d} is chosen such that $r_i = \Theta(d^{-1})$. This scaling is essential to ensure a non-degenerate asymptotic behavior: it prevents the kernel $k_d$ from either collapsing to trivial cases or diverging as dimension $d \to \infty$. Condition \eqref{constant bound} is imposed to control the overall spectral scaling: it prevents the eigenvalues of $k_d$ from either exploding to infinity or uniformly converging to zero as dimension $d$ increases.
\end{remark}

Assumption \ref{product structure assumption} generally holds for a broad class of kernels, including 
\begin{itemize}
    \item Gaussian kernels on general Gaussian measure;

    \item Mehler kernels on general Gaussian measure; 

    \item Laguerre kernels on $(0,\infty)^d$;

    \item kernels with mixed components.
\end{itemize} We may defer the detailed discussion of the above kernels to Appendix \ref{sec:example}.

The following proposition describes the spectrum of product kernel $k_d$.

\begin{proposition}[staircase descent eigenvalues] \label{prop:staircase}
Under Assumption~\ref{product structure assumption}, 
the Mercer's decomposition of $k_d$ can be expressed as
\begin{equation}\label{Mercer decomposition of product}
k_d(x,x')=\sum_{i=0}^\infty\sum_{j=1}^{N(d,i)}\mu_{i,j} e_{i,j}(x)e_{i,j}(x').
\end{equation}
Moreover, for any fixed $i$, when $d$ is large enough, we have 
\begin{equation} 
        \mathfrak{C}_1 d^{-i}\leq\mu_{i,j}\leq\mathfrak{C}_2 d^{-i},\quad\mathfrak{C}_1 d^{i}\leq N(d,i)\leq\mathfrak{C}_2 d^{i}
\end{equation} 
for some constants $\mathfrak{C}_1,\mathfrak{C}_2>0$ independent of $j$. 
For simplicity, we further rewrite $f_\rho^*=\sum_{i=0}^\infty\sum_{j=1}^{N(d,i)}f_{i,j} e_{i,j}(x).$
\end{proposition}

\begin{remark}
    We observe that inner product kernels on the sphere, although not strictly within the product kernel family, also possess such staircase descent eigenvalue property. Consequently, our results can also be applied to inner product kernels on sphere, thereby surpassing the previous work such as \cite{zhang2024optimal}.
\end{remark}

Proposition \ref{prop:staircase} provides the explicit asymptotic properties of the eigenvalues, and hence we are able to derive the  convergence rates of large dimensional product KRR. 

In order to calculate the quantities in Theorem \ref{main theorem}, \ref{main theorem s<1}, we need the following assumption.
\begin{assumption}\label{assump:lower bound}
    Denote $ q $ as the smallest integer such that $ q > \gamma$ and $\mu_{q} \neq 0$. Suppose that there exists an absolute constant $c_{0} > 0$ such that for any $ d $ and $ i \le q$ , we have
$$       \sum_{j=1}^{N(d,i)}\mu_{i,j}^{-s}f_{i,j}^2 \ge c_{0}.$$
\end{assumption}

\begin{remark}
    Assumption \ref{assump:lower bound} actually implies that $ f_{\rho}^{*} \notin [\mathcal{H}]^{t}$ for any $t > s$. Similar assumptions have been adopted when the lower bound of generalization error needs to be derived in the fixed-dimensional setting, such as equation (8) in \cite{Cui2021GeneralizationER} and Assumption 3 in \cite{li2023asymptotic}.
\end{remark}

Next we shall provide an insightful but informal proposition on the asymptotic properties of the quantities of eigenvalues of large dimensional product kernels. The formal form is deferred to Appendix \ref{section eigenvlue facts}.

\begin{proposition}[informal]\label{prop:gaussian eigenvalues}
    Recall the definitions of $\mathcal{N}_{1}(\lambda)$, $\mathcal{N}_{2}(\lambda)$, $\mathcal{R}_2(\lambda)$ in  \eqref{n1 n2 m1 m2}. Define $\tilde{s}=\min\{s,2\}.$ If we choose $\lambda = d^{-l}$ for some $l > 0$, denote $ p \le l \le p+1$ for some $ p \in \{0,1,2\cdots\}$, when $d$ is large enough, we have: 
    \begin{gather*}
    \mathcal{N}_{1}(\lambda) = O\left(\lambda^{-1}\right), \quad 
    \mathcal{N}_{2}(\lambda) = \Theta\left(d^{p} + \lambda^{-2} d^{-(p+1)}\right), \\
    \mathcal{R}_{2}(\lambda) = \Theta(\lambda^2d^{(2-\tilde{s})p}+d^{-(p+1)\tilde{s}}),\\
    \left\| f_{\lambda} \right\|_{L^{\infty}} = O\left( d^{\frac{(1-s)p}{2}} + \lambda^{-1} d^{-\frac{(1+s)(p+1)}{2}}\right)(\text{when }  s<1).
\end{gather*}
\end{proposition}

Now we are prepared to provide the rates of the generalization error under the large dimensional product kernel case, which are direct applications of Theorem \ref{main theorem}, \ref{main theorem s<1}.

\begin{theorem}[$s\ge1$]
    \label{thm:upper bound gaussian}
    Let $c_{1} d^{\gamma} \le n \le c_{2} d^{\gamma} $ for some fixed $ \gamma >0$ and absolute constants $ c_{1}, c_{2}$. Define $ \tilde{s} = \min\{s,2\} $. Under Assumption \ref{product structure assumption}, when $s\ge 1$,  we have:
    \begin{itemize}
        \item[(i)] When $\gamma \in \left(p+p\tilde{s},~ p+p\tilde{s}+1 \right]$ for some $ p \in \{0,1,2\cdots\}$, by choosing $ \lambda = d^{-\frac{\gamma+p-p\tilde{s}}{2}} \cdot \mathbf{1}_{p>0} + d^{-\frac{\gamma}{2}} \ln{d} \cdot \mathbf{1}_{p=0} $, we have
        \begin{align}
             \mathbb{E}\left[\left\|\hat{f}_\lambda-f_\rho^*\right\|_{L^2}^2 \;\Big|\; \boldsymbol{X} \right]=\begin{cases}
            \Theta_{\mathbb{P}}\left( d^{-\gamma} \ln^{2}{d} \right) = \Theta_{\mathbb{P}}\left( n^{-1} \ln^{2}{n} \right),& p=0, \\[3pt]
            \Theta_{\mathbb{P}}\left( d^{-\gamma + p} \right) = \Theta_{\mathbb{P}}\left( n^{-1 + \frac{p}{\gamma}} \right), & p>0;
            \end{cases}
        \end{align}

        \item[(ii)] When $\gamma \in \left(p+p\tilde{s}+1,~ p+p\tilde{s}+2\tilde{s}-1 \right]$ for some $ p \in \{0,1,2\cdots\}$, by choosing $ \lambda = d^{-\frac{\gamma+3p-p\tilde{s}+1}{4}}$, we have
        \begin{equation}
            \mathbb{E}\left[\left\|\hat{f}_\lambda-f_\rho^*\right\|_{L^2}^2 \;\Big|\; \boldsymbol{X} \right] = \Theta_{\mathbb{P}}\left( d^{-\frac{\gamma-p+p\tilde{s}+1}{2}} \right) = \Theta_{\mathbb{P}}\left( n^{- \frac{\gamma-p+p\tilde{s}+1}{2 \gamma}} \right);
        \end{equation}

        \item[(iii)] When $\gamma \in \left(p+p\tilde{s}+2\tilde{s}-1,~ (p+1)+(p+1)\tilde{s} \right]$ for some $ p \in \{0,1,2\cdots\}$, by choosing $ \lambda = d^{-\frac{\gamma+(p+1)(1-\tilde{s})}{2}}$, we have
        \begin{equation}
             \mathbb{E}\left[\left\|\hat{f}_\lambda-f_\rho^*\right\|_{L^2}^2 \;\Big|\; \boldsymbol{X} \right] = \Theta_{\mathbb{P}}\left( d^{-(p+1)\tilde{s}} \right) = \Theta_{\mathbb{P}}\left( n^{- \frac{(p+1)\tilde{s}}{\gamma}} \right).
        \end{equation}
    \end{itemize}
    Also, there exists no $\lambda$ whose convergence rate of the generalization error of KRR is faster than the above choice.
\end{theorem}
\begin{remark}
    In the case $0<\gamma\le1$, we choose $\lambda=d^{-\gamma/2}\ln{d}$ rather than $d^{-\gamma/2}$ in order to meet the conditions in Theorem \ref{main theorem}. Correspondingly, the convergence rate is slightly smaller than $d^{-\gamma+p}$, which is the case when $2p<\gamma\le 2p+1$, $p>0$. However, we shall show in Theorem \ref{theorem lower bound} that both cases are in accordance with the minimax rate of KRR in large dimensional setting. In the following context we may ignore $\ln^2 n$ without causing any ambiguity.
\end{remark}

\begin{theorem}[$s<1$]\label{theorem s<1 gaussian}Let $c_{1} d^{\gamma} \le n \le c_{2} d^{\gamma} $ for some fixed $ \gamma >0$ and absolute constants $ c_{1}, c_{2}$. Under Assumption \ref{product structure assumption}, when $s< 1$,  we have:
    \begin{itemize}
        \item If $ \frac{1}{2} < s < 1$:
        \begin{itemize}
        \item[(i)] When $\gamma \in \left(p+ps,~ p+ps+s \right]$ for some $p \in \mathbb{N}$, by choosing $ \lambda = d^{-\frac{\gamma+p-ps}{2}} \cdot \mathbf{1}_{p>0} + d^{-\frac{\gamma}{2}} \ln{d} \cdot \mathbf{1}_{p=0} $, we have
        \begin{align}
             \mathbb{E}\left[\left\|\hat{f}_\lambda-f_\rho^*\right\|_{L^2}^2 \;\Big|\; \boldsymbol{X} \right]=\begin{cases}
            \Theta_{\mathbb{P}}\left( d^{-\gamma} \ln^{2}{d} \right) = \Theta_{\mathbb{P}}\left( n^{-1} \ln^{2}{n} \right),& p=0, \\[3pt]
            \Theta_{\mathbb{P}}\left( d^{-\gamma + p} \right) = \Theta_{\mathbb{P}}\left( n^{-1 + \frac{p}{\gamma}} \right), & p>0;
            \end{cases}
        \end{align}

        \item[(ii)] When $\gamma \in \left(p+ps+s,~ (p+1)+(p+1)s \right]$ for some $p \in \mathbb{N}$, by choosing $ \lambda = d^{-\frac{2p+s}{2}}$, we have
        \begin{equation}
            \mathbb{E}\left[\left\|\hat{f}_\lambda-f_\rho^*\right\|_{L^2}^2 \;\Big|\; \boldsymbol{X} \right] = \Theta_{\mathbb{P}}\left( d^{-(p+1)s} \right) = \Theta_{\mathbb{P}}\left( n^{- \frac{(p+1)s}{\gamma}} \right);
        \end{equation}

        \end{itemize}

    \item If $ 0 < s \le \frac{1}{2}$:  we have the same convergence rates as the case $ s \in (\frac{1}{2},1)$ when
         $\gamma > \frac{3s}{2(s+1)}$.
        
    \end{itemize}

\end{theorem}

\begin{remark}
    Note that in Theorem \ref{theorem s<1 gaussian}, we do not claim that the $\lambda$ we choose is optimal. However, we shall show in Theorem \ref{theorem lower bound} that the minimax lower bound is in accordance with the rate given in Theorem \ref{theorem s<1 gaussian}. Hence, the rates in Theorem \ref{theorem s<1 gaussian} are the fastest convergence rates KRR can achieve.
\end{remark}

\begin{remark}
    In Theorem \ref{theorem s<1 gaussian}, we only prove the case $\gamma > \frac{3s}{2(s+1)}$ when $s\le \frac{1}{2}$. However, notice that $\frac{3s}{2(s+1)}\le \frac{1}{2}$ when $s\le \frac{1}{2}$, we actually provide the exact convergence rate for all $n\gtrsim d^{\frac{1}{2}}$. 
\end{remark}

Next, we provide the minimax lower bound of large dimensional product kernels.

\begin{theorem}[minimax lower bound]\label{theorem lower bound}
    Let $c_{1} d^{\gamma} \le n \le c_{2} d^{\gamma} $ for some fixed $ \gamma >0$ and absolute constants $ c_{1}, c_{2}$. L et $k=k_d$ be a product kernel such that Assumption \ref{product structure assumption} holds. Let $\mathcal{P}$ consist of all distributions $\rho$ on $\mathcal{X} \times \mathcal{Y}$ such that Assumption \ref{assumption noise}, \ref{assumption f*}, \ref{assump:lower bound} hold. Then we have:

        \begin{itemize}
            \item[(i)] When $\gamma \in \left(p+ps, p+ps+s \right]$ for some $p \in \mathbb{N}$, for any $\epsilon > 0$, there exist constants $\mathfrak{C}_1$ and $\mathfrak{C}$ only depending on $s, \epsilon, \gamma, \sigma_\epsilon, \kappa, c_{1}$ and $ c_{2} $ such that for any $d \geq \mathfrak{C}$, we have:
            \begin{equation}\label{minimax lower eq 1}
            \min _{\hat{f}} \max _{\rho \in \mathcal{P}} \mathbb{E}_{(\boldsymbol{X}, \boldsymbol{y}) \sim \rho^{\otimes n}}
            \left\|\hat{f} - f_{\rho}^{*}\right\|_{L^2}^2
            \geq \mathfrak{C}_1 d^{-\gamma + p - \epsilon};
            \end{equation}

            \item[(ii)] When $\gamma \in \left(p+ps+s, (p+1)+(p+1)s \right]$ for some $p \in \mathbb{N}$, there exist constants $\mathfrak{C}_1$ and $\mathfrak{C}$ only depending on $s, \gamma, \sigma_\epsilon, \kappa, c_{1}$ and $ c_{2} $ such that for any $d \geq \mathfrak{C}$, we have:
            \begin{equation}\label{minimax lower eq 2}
            \min _{\hat{f}} \max _{\rho \in \mathcal{P}} \mathbb{E}_{(\boldsymbol{X}, \boldsymbol{y}) \sim \rho^{\otimes n}}
            \left\|\hat{f} - f_{\rho}^{*}\right\|_{L^2}^2
            \geq \mathfrak{C}_1 d^{-(p+1)s}.
            \end{equation}
        \end{itemize}

\end{theorem}

Figure \ref{figure 5 rates w.r.t d} shows the theoretical rates of generalization error of KRR and the corresponding minimax rate for different $s$. We also provide an experimental result in Appendix \ref{sec:experiment}. From Figure \ref{figure 5 rates w.r.t d} we shall draw the following conclusions:

\begin{figure}[h]\
\centering
\subfigure[]{\includegraphics[width=0.32\columnwidth]{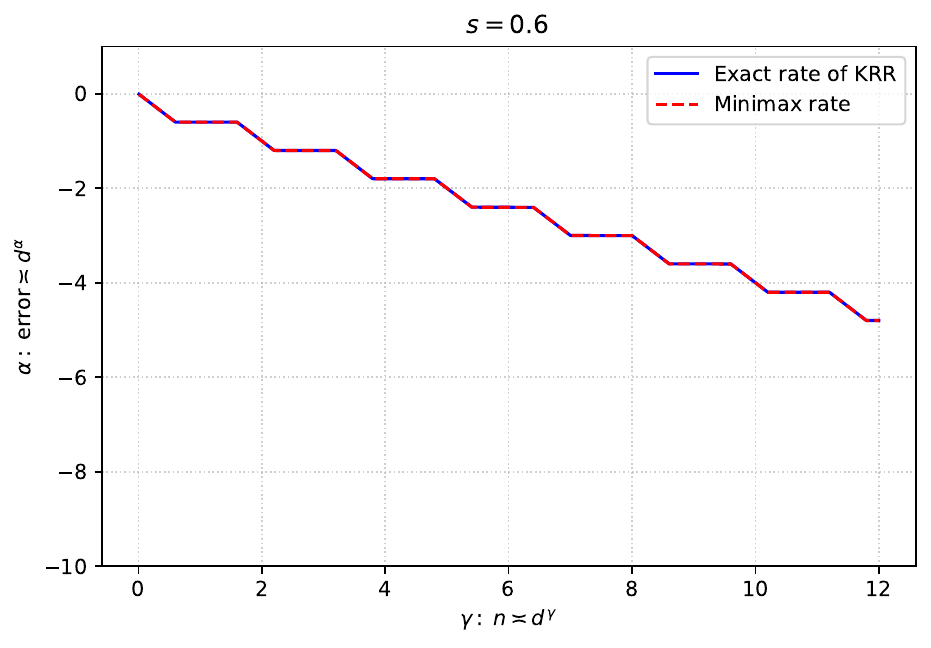}}
\subfigure[]{\includegraphics[width=0.32\columnwidth]{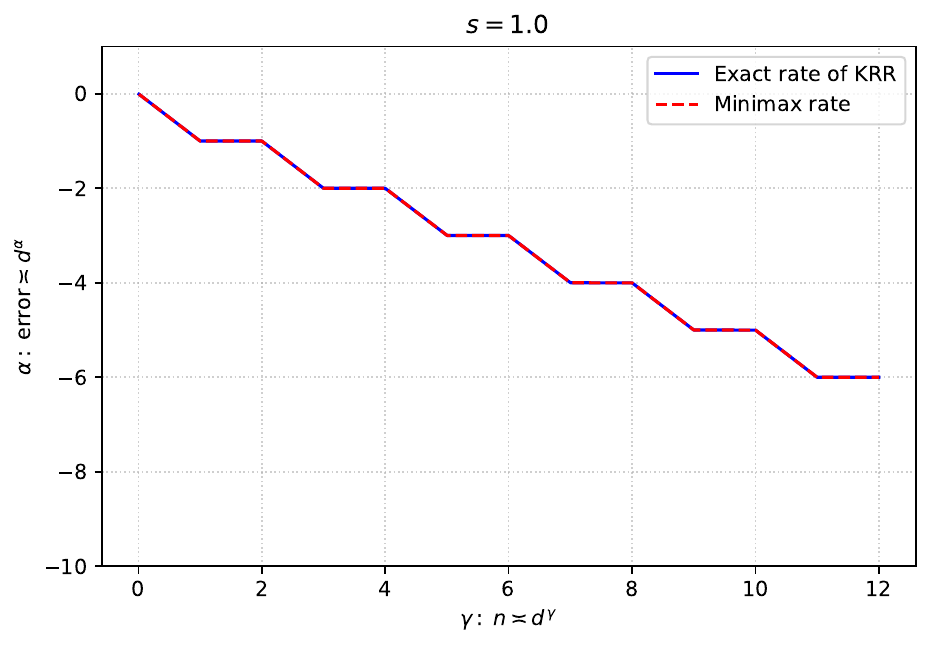}}
\subfigure[]{\includegraphics[width=0.32\columnwidth]{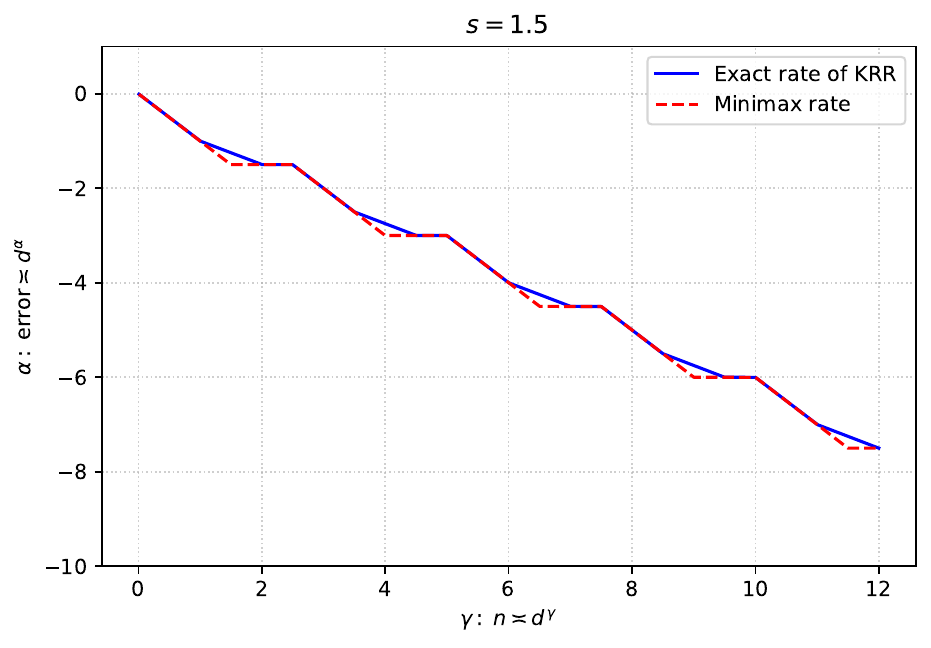}}
\subfigure[]{\includegraphics[width=0.32\columnwidth]{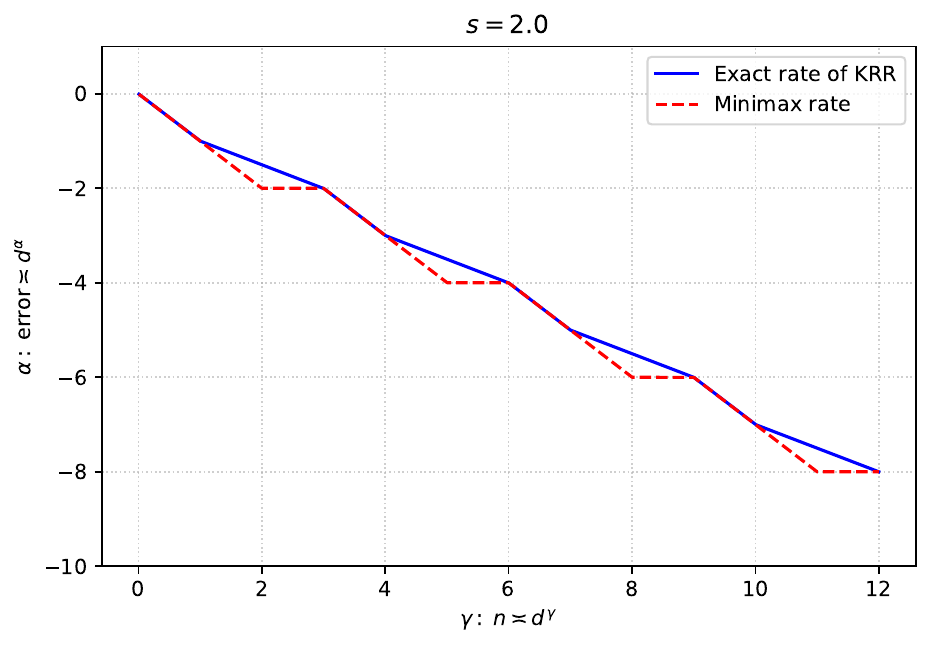}}
\subfigure[]
{\includegraphics[width=0.32\columnwidth]{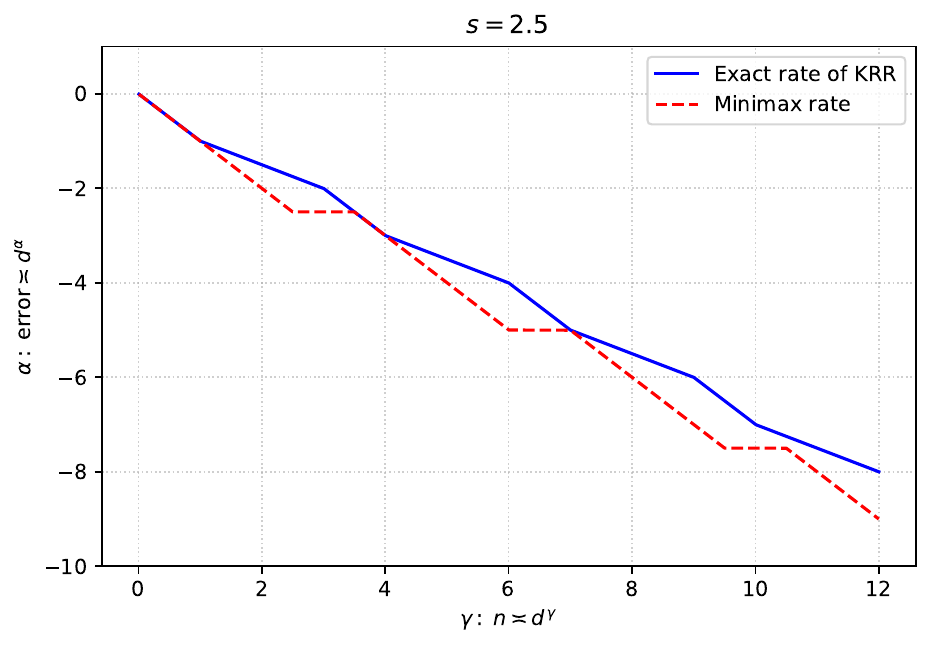}}

\caption{
The exact rates of generalization error of KRR (blue) and the corresponding minimax rate (red) with respect to $d$, $s=0.6,1,1.5,2,2.5$. The difference between two rates suggests the \textit{saturation effect}, the unchanging rates during certain intervals suggest the \textit{periodic plateau phenomenon}.
}
\label{figure 5 rates w.r.t d}
\end{figure}

\textbf{Minimax optimality and saturation effect} From Figure \ref{figure 5 rates w.r.t d}, we conclude:
\begin{itemize}
    \item When $s\leq 1$, the convergence rate of generalization error of KRR reaches the minimax lower bound for most $(s,\gamma)$. (Notice that we only derive the convergence rate of the generalization error of KRR for $\gamma>\frac{3s}{2s+2}$ when $s\le\frac{1}{2}$.) Hence, when $s\leq1$, with a little abusement, we shall claim that large dimensional product kernel ridge regression is \textbf{ minimax optimal}.

    \item When $1<s\le2$, the convergence rate of generalization error of KRR is in accordance with the minimax lower bound for certain ranges of $\gamma$ ($\gamma\in \left(p+p{s},~ p+p{s}+1 \right]\cup\left(p+p{s}+2{s}-1,~ (p+1)+(p+1){s} \right]$), while fails to reach the minimax rate for other $\gamma$. Such gap between the exact rate of KRR and the minimax rate is called the saturation effect, which was reported by \cite{li2023_SaturationEffect} when $s\ge 2$ in the fixed dimensional setting, \cite{zhang2024optimal} for inner product kernels on sphere when $s>1$ in the large dimensional setting. \textbf{We extend the saturation effect from inner product kernels on sphere to other large dimensional kernels such as product kernels,} and hence indicating that saturation effect exists for a large group of kernel for KRR in the large dimensional setting.

    \item When $s>2$, we find that as $\gamma$ increases, the gap between the exact rate of KRR and the minimax rate tends to increase. This is due to the reason that the exact rate of KRR no longer changes when $s\ge2$, while the minimax rate remain to decrease as $s$ increases. As a result, when $\gamma$ is sufficiently large, the convergence rate of generalization error of KRR shall never reach the minimax lower bound, hence we claim that the saturation effect occurs for most of the $\gamma$.

\end{itemize}

\textbf{Periodic plateau phenomenon}
For the minimax rate, we observe a periodic plateau phenomenon for all large dimensional product kernels satisfying Assumption \ref{product structure assumption}: the rate of the excess risk remains unchanged over certain intervals of $\gamma$, and grows faster in other intervals of $\gamma$. Similar periodic plateau phenomenon was reported by \cite{Ghorbani2019LinearizedTN,lu2023optimal,zhang2024optimal}, when considering large dimensional spectral algorithms on the sphere. 

For the exact convergence rate of generalization error of KRR, the periodic plateau phenomenon can also be observed when $0<s<2$. When $s\ge 2$, the plateau period degenerates.

Such periodic plateau phenomenon implies that in order to increase the convergence rate of the excess risk, it might be necessary to increase $\gamma$ (or equivalently, the sample size $n$) beyond a specific threshold. 

We now present Figure \ref{figure 3 rates w.r.t n}, which illustrates the convergence rate of the large dimensional product kernel as a function of sample size $n$ (as opposed to dimension $d$). This figure enables us to draw the following conclusion.

\begin{figure}[h]
\centering
\subfigure[]{\includegraphics[width=0.32\columnwidth]{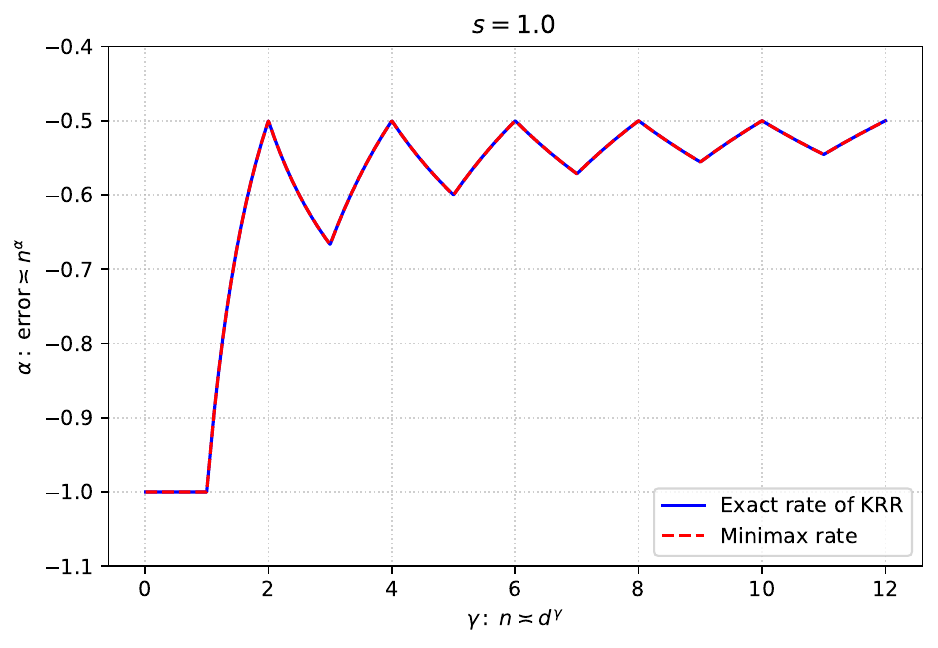}}
\subfigure[]{\includegraphics[width=0.32\columnwidth]{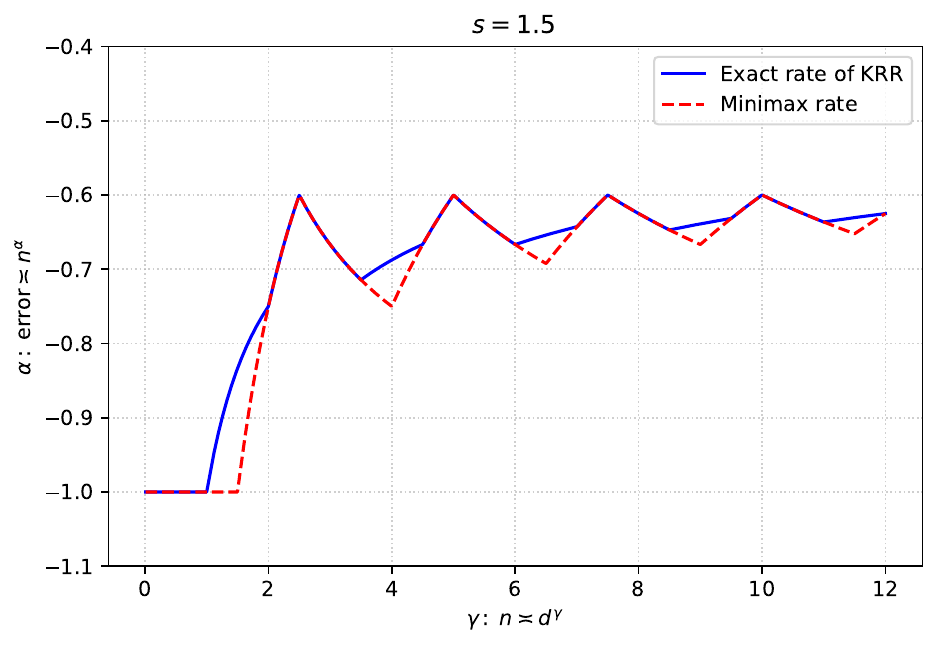}}
\subfigure[]{\includegraphics[width=0.32\columnwidth]{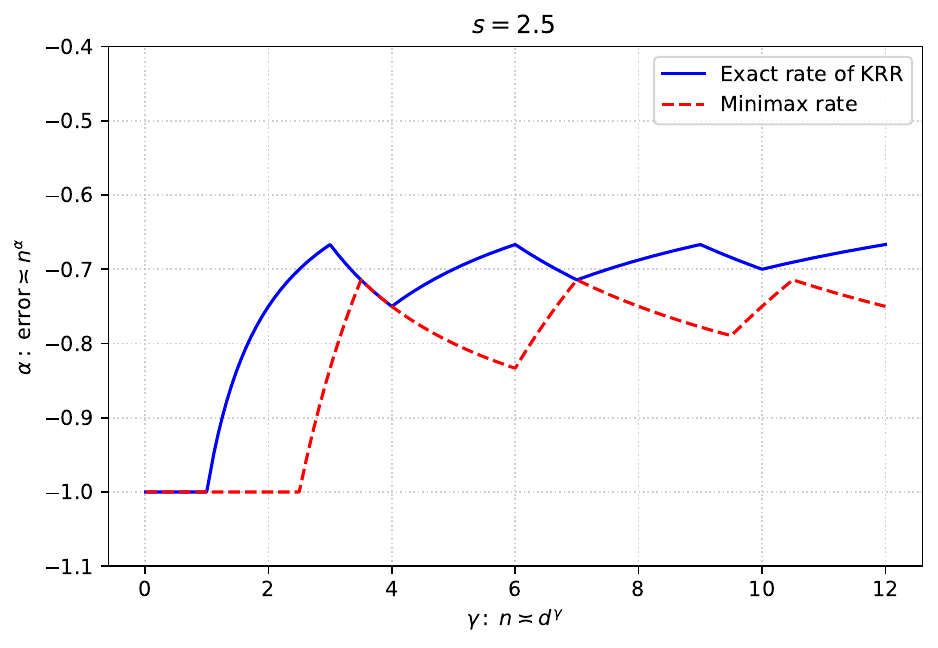}}

\caption{
The exact rates of generalization error of KRR (blue) and the corresponding minimax rate (red) with respect to $n$, $s=1,1.5,2.5$. The vibration of rates suggest the multiple descent behavior.
}
\label{figure 3 rates w.r.t n}
\end{figure}
\textbf{Multiple descent behavior} In Figure  \ref{figure 3 rates w.r.t n}, we find that both the exact convergence rates and minimax rates experience vibration when $\gamma$ varies. For the minimax rate, the curve achieves its local maxima (the slowest rate) at $\gamma = p+ps, p \in \mathbb{N}^{+} $, and achieve its local minima (the fastest rate) at $ \gamma = p+ps+s, p \in \mathbb{N}^{+}$. For the exact convergence rates of KRR, when $0 < s \leq 1$, the curve coincides with that of the minimax lower bound. When $ 1 < s < 2$, the curve achieves its local maxima at $\gamma = p+p\tilde{s}, p \in \mathbb{N}^{+} $, and  achieve its local minima at $ \gamma = p+p\tilde{s}+1, p \in \mathbb{N}^{+} $. When $s \ge 2$, the curve does not change with $s$, which achieves its local maxima at $\gamma = 3p, p \in \mathbb{N}^{+} $, and achieve its local minima at $ \gamma = 3p+1, p \in \mathbb{N}^{+}$. We call this vibration of the convergence rate the multiple descent behavior, which is also reported in \cite{zhang2024optimal,lu2023optimal} for inner product kernels on sphere. Generally speaking, we conclude that when $\gamma$ increases, the convergence rate with respect to $n$ is slower, implying that increasing $n$ is less worthwhile when the sample size is already large enough.

\section{Conclusion}
In this paper, we consider KRR under the large dimensional setting. Unlike most large dimensional kernel related literature, we relieve the constraint on eigenfunctions, and provide an exact rate of the generalization error. We consider the family of large dimensional product kernel as an example, which is not applicable in the previous work, either due to the eigenfunction  constraint \citep{zhang2024optimal} or a hypercontractivity condition \citep{misiakiewicz2024non}. As a result, we prove the minimax optimality of large dimensional product KRR when $s\le 1$, and report a saturation effect when $s>1$. Furthermore, we identify phenomena such as {\it periodic plateau phenomenon} and {\it multiple descent behavior}, similar to those found in previous studies by \cite{lu2023optimal,zhang2024optimal}, which considered the inner product kernel on the sphere. Such similarity suggests that these phenomena may be general to large dimensional kernel methods, and can provide insights into the trade-off between sample size $n$ and convergence rates in practical applications.

Future work may extend this analysis to other kernel related algorithms under the large dimensional setting with product kernels. We believe such researches would significantly advance the understanding of kernel methods in the large dimensional setting, generalizing the restrictive case of inner product kernels on the sphere to a broader class of kernel functions. Recently, \cite{li2025diagonal} considers a diagonal adaptive kernel model under the fixed dimension setting. They claim that by introducing over-parametrization to the learning of kernels, the model can significantly improve its generalization properties compared with the fixed-kernel method. It would be of great interest to determine  whether such phenomenon exists in the large dimensional setting for KRR. Understanding the underlying mechanism of how over-parameterization affects the learning of adaptive kernels in large dimensional setting could shed light on the potential advantages of adaptive kernels in large dimensional environments.

\bibliographystyle{chicago}
\bibliography{colt2026reference.bib}

@article{fischer2020_SobolevNorm,
  title = {Sobolev Norm Learning Rates for Regularized Least-Squares Algorithms},
  author = {Fischer, Simon-Raphael and Steinwart, Ingo},
  year = {2020},
  journal = {Journal of Machine Learning Research},
  volume = {21},
  pages = {205:1-205:38},
  arxivid = {1702.07254}
}

@article{bauer2007_RegularizationAlgorithms,
  title = {On Regularization Algorithms in Learning Theory},
  author = {Bauer, F. and Pereverzyev, S. and Rosasco, L.},
  year = {2007},
  journal = {Journal of complexity},
  volume = {23},
  number = {1},
  pages = {52--72},
  publisher = {{Elsevier}}
}

@article{lin2020_OptimalConvergence,
  title = {Optimal Convergence for Distributed Learning with Stochastic Gradient Methods and Spectral Algorithms.},
  author = {Lin, Junhong and Cevher, Volkan},
  year = {2020},
  journal = {Journal of Machine Learning Research},
  volume = {21},
  pages = {147--1}
}

@article{steinwart2012_MercerTheorem,
  title = {Mercer's Theorem on General Domains: {{On}} the Interaction between Measures, Kernels, and {{RKHSs}}},
  shorttitle = {Mercer's Theorem on General Domains},
  author = {Steinwart, Ingo and Scovel, C.},
  year = {2012},
  journal = {Constructive Approximation},
  volume = {35},
  number = {3},
  pages = {363--417},
  publisher = {{Springer}}
}

@article{Cui2021GeneralizationER,
  title={Generalization error rates in kernel regression: The crossover from the noiseless to noisy regime},
  author={Cui, Hugo and Loureiro, Bruno and Krzakala, Florent and Zdeborov{\'a}, Lenka},
  journal={Advances in Neural Information Processing Systems},
  volume={34},
  pages={10131--10143},
  year={2021}
}

@inproceedings{li2023_SaturationEffect,
  title = {On the Saturation Effect of Kernel Ridge Regression},
  booktitle = {International {{Conference}} on {{Learning Representations}}},
  author = {Li, Yicheng and Zhang, Haobo and Lin, Qian},
  year = {2023},
  month = feb,
  langid = {english}
}

@inproceedings{jacot2018_NeuralTangent,
  title = {Neural Tangent Kernel: {{Convergence}} and Generalization in Neural Networks},
  booktitle = {Advances in Neural Information Processing Systems},
  author = {Jacot, Arthur and Gabriel, Franck and Hongler, Clement},
  editor = {Bengio, S. and Wallach, H. and Larochelle, H. and Grauman, K. and {Cesa-Bianchi}, N. and Garnett, R.},
  year = {2018},
  volume = {31},
  publisher = {{Curran Associates, Inc.}}
}

@article{beaglehole2023inconsistency,
  title={On the inconsistency of kernel ridgeless regression in fixed dimensions},
  author={Beaglehole, Daniel and Belkin, Mikhail and Pandit, Parthe},
  journal={SIAM Journal on Mathematics of Data Science},
  volume={5},
  number={4},
  pages={854--872},
  year={2023},
  publisher={SIAM}
}

@article{Ghorbani2019LinearizedTN,
author = {Behrooz Ghorbani and Song Mei and Theodor Misiakiewicz and Andrea Montanari},
title = {{Linearized two-layers neural networks in high dimension}},
volume = {49},
journal = {The Annals of Statistics},
number = {2},
publisher = {Institute of Mathematical Statistics},
pages = {1029 -- 1054},
year = {2021}
}

@article{Yang_Density_1999,
author = {Yuhong Yang and Andrew Barron},
title = {{Information-theoretic determination of minimax rates of convergence}},
volume = {27},
journal = {The Annals of Statistics},
number = {5},
publisher = {Institute of Mathematical Statistics},
pages = {1564 -- 1599},
year = {1999}
}

@inproceedings{zhang2023optimality_2,
  title={On the optimality of misspecified kernel ridge regression},
  author={Zhang, Haobo and Li, Yicheng and Lu, Weihao and Lin, Qian},
  booktitle={International Conference on Machine Learning},
  pages={41331--41353},
  year={2023},
  organization={PMLR}
}

@article{caponnetto2010cross,
  title={Cross-validation based adaptation for regularization operators in learning theory},
  author={Caponnetto, Andrea and Yao, Yuan},
  journal={Analysis and Applications},
  volume={8},
  number={02},
  pages={161--183},
  year={2010},
  publisher={World Scientific}
}

@article{misiakiewicz_spectrum_2022,
  title={Spectrum of inner-product kernel matrices in the polynomial regime and multiple descent phenomenon in kernel ridge regression},
  author={Misiakiewicz, Theodor},
  journal={arXiv preprint arXiv:2204.10425},
  year={2022}
}

@inproceedings{li2023asymptotic,
  title={On the Asymptotic Learning Curves of Kernel Ridge Regression under Power-law Decay},
  author={Li, Yicheng and Zhang, Haobo and Lin, Qian},
  booktitle={Thirty-seventh Conference on Neural Information Processing Systems},
  year={2023}
}

@inproceedings{buchholz2022_KernelInterpolation,
  title = {Kernel Interpolation in {{Sobolev}} Spaces Is Not Consistent in Low Dimensions},
  booktitle = {Proceedings of Thirty Fifth Conference on Learning Theory},
  author = {Buchholz, Simon},
  year = {2022},
  month = jul,
  series = {Proceedings of Machine Learning Research},
  volume = {178},
  pages = {3410--3440},
  publisher = {{PMLR}},
  issn = {2640-3498},
  langid = {english}
}

@book{10.7551/mitpress/3206.001.0001,
    author = {Rasmussen, Carl Edward and Williams, Christopher K. I.},
    title = {Gaussian Processes for Machine Learning},
    publisher = {The MIT Press},
    year = {2005},
    month = {11},
    isbn = {9780262256834},
    eprint = {https://direct.mit.edu/book-pdf/2514321/book\_9780262256834.pdf},
}

@article{lu2023optimal,
  title={Optimal Rate of Kernel Regression in Large Dimensions},
  author={Lu, Weihao and Zhang, Haobo and Li, Yicheng and Xu, Manyun and Lin, Qian},
  journal={arXiv preprint arXiv:2309.04268},
  year={2023}
}

@book{steinwart2008support,
  title={Support vector machines},
  author={Steinwart, Ingo and Christmann, Andreas},
  year={2008},
  publisher={Springer Science \& Business Media}
}

@article{zhang2024optimal,
  title={Optimal rates of kernel ridge regression under source condition in large dimensions},
  author={Zhang, Haobo and Li, Yicheng and Lu, Weihao and Lin, Qian},
  journal={Journal of Machine Learning Research},
  volume={26},
  number={219},
  pages={1--63},
  year={2025}
}

@article{caponnetto2007optimal,
  title={Optimal rates for the regularized least-squares algorithm},
  author={Caponnetto, Andrea and De Vito, Ernesto},
  journal={Foundations of Computational Mathematics},
  volume={7},
  number={3},
  pages={331--368},
  year={2007},
  publisher={Springer}
}

@techreport{Caponnetto2006OptimalRF,
  author = {Caponnetto, Andrea},
  title = {Optimal rates for regularization operators in learning theory},
  institution = {Massachusetts Institute of Technology},
  number = {CBCL Paper \#264/AI Technical Report \#062},
  year = {2006},
  month = {September},
}

@article{raskutti2014early,
  author  = {Garvesh Raskutti and Martin J. Wainwright and Bin Yu},
  title   = {Early Stopping and Non-parametric Regression: An Optimal Data-dependent Stopping Rule},
  journal = {Journal of Machine Learning Research},
  year    = {2014},
  volume  = {15},
  number  = {11},
  pages   = {335--366}
}

@article{li2023kernel,
  title={Kernel interpolation generalizes poorly},
  author={Li, Yicheng and Zhang, Haobo and Lin, Qian},
  journal={arXiv preprint arXiv:2303.15809},
  year={2023}
}

@article{lai2023generalization,
  title={Generalization ability of wide neural networks on $\mathbb{R}$},
  author={Lai, Jianfa and Xu, Manyun and Chen, Rui and Lin, Qian},
  journal={arXiv preprint arXiv:2302.05933},
  year={2023}
}

@article{barzilai2023generalization,
  title={Generalization in Kernel Regression Under Realistic Assumptions},
  author={Barzilai, Daniel and Shamir, Ohad},
  journal={arXiv preprint arXiv:2312.15995},
  year={2023}
}

@article{misiakiewicz2024non,
  title={A non-asymptotic theory of kernel ridge regression: deterministic equivalents, test error, and gcv estimator},
  author={Misiakiewicz, Theodor and Saeed, Basil},
  journal={arXiv preprint arXiv:2403.08938},
  year={2024}
}

@article{simon2023eigenlearning,
  title={The eigenlearning framework: A conservation law perspective on kernel ridge regression and wide neural networks},
  author={Simon, James B and Dickens, Madeline and Karkada, Dhruva and DeWeese, Michael R},
  journal={Transactions on Machine Learning Research},
  year={2023}
}

@article{li2025diagonal,
  title={Diagonal Over-parameterization in Reproducing Kernel Hilbert Spaces as an Adaptive Feature Model: Generalization and Adaptivity},
  author={Li, Yicheng and Lin, Qian},
  journal={arXiv preprint arXiv:2501.08679},
  year={2025}
}

@article{pandit2024universality,
  title={Universality of kernel random matrices and kernel regression in the quadratic regime},
  author={Pandit, Parthe and Wang, Zhichao and Zhu, Yizhe},
  journal={arXiv preprint arXiv:2408.01062},
  year={2024}
}

@article{zhang2025phase,
  title={The phase diagram of kernel interpolation in large dimensions},
  author={Zhang, Haobo and Lu, Weihao and Lin, Qian},
  journal={Biometrika},
  volume={112},
  number={1},
  pages={asae057},
  year={2025},
  publisher={Oxford University Press}
}

@article{laguerre_formula,
  title={Bilinear generating functions for Laguerre and Lauricella polynomials in several variables},
  author={Carlitz, L},
  journal={Rendiconti del Seminario Matematico della Universita di Padova},
  volume={43},
  pages={269--276},
  year={1970}
}

@article{mehler_formula,
  title={Notes on generating functions of polynomials:(2) Hermite polynomials},
  author={Watson, George Neville},
  journal={Journal of the London Mathematical Society},
  volume={1},
  number={3},
  pages={194--199},
  year={1933},
  publisher={Oxford University Press}
}
\appendix
\section{Examples of Large Dimensional Product Kernels}\label{sec:example}
\subsection{Large Dimensional Gaussian Kernels}
For every index $i$, consider the $1$-dimensional Gaussian kernel $\hat{k}_i(x_i,x_i')=\exp\left(\frac{(x_i-x_i')^2}{2\ell_i^2d}\right)$ on Gaussian measure $N(0,\sigma_i^2)$. Assume that for all $i$, there exist absolute constants $0<c\leq\ell_i,\sigma_i\leq C$. The respective large dimensional product kernel as 
$$
k_d(x, x') = \prod_{i=1}^{d} \hat{k}_{ i}(x_i, x'_i).
$$
The following proposition provides the explicit form of the Mercer's decomposition of Gaussian kernel, which can be found in \cite{10.7551/mitpress/3206.001.0001}, Section 4.3.
\begin{proposition}\label{prop mercer decomposition}
    Denote $a_i^{-1}=4\sigma_i^2,b_i^{-1}=2\ell_i^2 d$, $c_i=\sqrt{a_i^2+2a_ib_i},A_i=a_i+b_i+c_i,B_i=b_i/A_i.$ Then, the Mercer's decomposition of large dimensional Gaussian kernel is given as follows.
    \begin{displaymath}
        \hat{k}_i(x_i,x_i') = \sum_{k=0}^{\infty} \mu_{ik} Y_{ik}({x_i})Y_{ik}({x_i^{\prime}}),
    \end{displaymath}
    where the eigenvalues $\{\mu_{ik},k=0,1,\cdots\} $ are $        \mu_{ik}=(2a_i/A_i)^\frac{1}{2}\cdot B_i^k,$ 
    and the eigenfunction $ Y_{ik}$ is 
    \begin{displaymath}
        Y_{ik}=c_{0k}\exp{\left(-(c_i-a_i)^2x_i^2\right)}H_{k}(\sqrt{2c}x_i),
    \end{displaymath}
    Here $H_{k}(x)=(-1)^k\exp(x^2)\frac{d^k}{dx^k}\exp(-x^2)$ denotes the $k^{\text{th}}$ Hermite polynomial, and $\{c_{0k}\}_{k\geq0}$ are constants satisfying $\|c_{0k}\exp{\left(-(c_i-a_i)^2x_i^2\right)}H_{k}(\sqrt{2c_i}x_i)\|_{L^2}=1$. 
\end{proposition}

Notice that 
\begin{equation}
        \begin{aligned}
            &\sqrt{2a_i/A_i}
            =\sqrt{\frac{2a_i}{a_i+b_i+\sqrt{a_i^2+2a_ib_i}}}\\
            &=\sqrt{\frac{2a_i}{2a_i+2b_i+\sqrt{a_i^2+2a_ib_i}-a_i-b_i}}\\
            &=\sqrt{\frac{2a_i}{2a_i+2b_i+\Theta(d^{-2})}}\\
            &=\sqrt{\frac{1}{1+\frac{1}{2\ell_i^2a_id}+\Theta(d^{-2})}}\\
            &=\sqrt{1-\frac{1}{2\ell_i^2a_id}+\Theta(d^{-2})},\\
        \end{aligned}
    \end{equation}
while 
\begin{equation}
        \begin{aligned}
            & B_i
            =\frac{b_i}{A_i}\\
            &=\frac{b_i}{a_i+b_i+\sqrt{a_i^2+2a_ib_i}}\\
            &=\frac{b_i}{2a_i+2b_i+\Theta(d^{-2})}\\
            &=\Theta (d^{-1}).
        \end{aligned}
    \end{equation}
Hence, the conditions in Assumption \ref{product structure assumption} clearly holds.

\subsection{The Laguerre Kernels}
Consider $x_1,\dots,x_d$ are independently distributed, with $x_i\sim\Gamma(\alpha_i+1,1)$ for some $\alpha_i>-1$, $i=1,\dots,d$. The Laguerre kernel is defined as  
\begin{equation}\label{laguerre kernel}
    L_{\mathbf{\theta},d}(x,y)=\prod_{i=1}^d l_{\theta_i/d,i}(x_i,y_i),
\end{equation} 
where $c<\theta_i<C$ is a bounded preselected parameter ($c$ and $C$ are absolute constants), and 
\begin{equation} 
    \begin{aligned} 
        l_{r,i}(x_i,y_i)=&\frac{\Gamma(\alpha_i+1)}{1-r}e^{-\frac{r}{1-r}(x_i+y_i)}(rx_iy_i)^{-\alpha_i/2}
        \cdot I_{\alpha_i}(\frac{2\sqrt{rx_iy_i}}{1-r}),
    \end{aligned}
\end{equation} 
is the $1$-dimensional Laguerre kernel with parameter $0<r<1$, and $I_{\alpha_i}$ is the modified Bessel function of the first kind with parameter $\alpha_i$. 

By the Laguerre formula (see (1.3) and (1.6) of \citet{laguerre_formula}), the $1$-dimensional kernel $l_{r,i}(x_i,y_i)$ has the following decomposition: 
\begin{equation} 
    l_{r,i}(x_i,y_i)=\sum_{k=1}^\infty r^k\frac{p_k^{\alpha_i}(x_i)p_k^{\alpha_i}(y_i)}{\binom{k+\alpha_i}{k}},
\end{equation} 
where $p_k^{\alpha_i}$ denotes the $k$-th Laguerre polynomial (see (1.1) and (1.2) of \citet{laguerre_formula}). As $\{p_k^{\alpha_i}(x_i)/\sqrt{\binom{k+\alpha_i}{k}}\}_{k=1}^\infty$ forms an orthonormal basis of $L^2((0,\infty),\Gamma(\alpha_i+1,1))$, the $k$-th eigenvalue of $l_{r,i}$ is $r^k$. Therefore, the multi-dimensional Laguerre kernel $L_{\theta,d}$ satisfies Assumption \ref{product structure assumption}.

\subsection{The Mehler Kernels}
Consider $x_1,\dots,x_d$ are independently distributed, with $x_i\sim N(0,\sigma_i^2)$ .
The Mehler kernel is defined by 
\begin{equation}\label{mehler kernel definition} 
    M_{\theta,d}(x,y)=\prod_{i=1}^d m_{\theta_i/d,i}(x_i,y_i),
\end{equation} 
where $c<\theta_i<C$ is a bounded preselected parameter ($c$ and $C$ are absolute constants), and 
\begin{equation} 
    m_{r,i}(x_i,y_i)=\frac{1}{(1-r^2)^{\frac{d}{2}}}\exp\left(\frac{2rx_iy_i-r^2(x_i^2+y_i^2)}{2\sigma_i^2(1-r^2)}\right),
\end{equation} 
is the $1$-dimensional Mehler kernel. By Mehler's formula \citep{mehler_formula}, the Mercer decomposition of $m_{r,i}$ is 
\begin{equation} 
    m_{r,i}(x_i,y_i)=\sum_{k=0}^\infty r^k\frac{H_k'(x_i/\sigma_i)H_k'(y_i/\sigma_i)}{n!}
\end{equation} 
where $H_k'(z)=(-1)^n e^{z^2/2}\frac{d^k}{dx^k}e^{-z^2/2}$ is the probabilists' Hermite polynomial, $k\in\mathbb{N}$. Since $\{H_k(x_i/\sigma_i)/\sqrt{n!}\}_{k=1}^\infty$ forms an orthonormal basis of $L^2(\mathbb{R},N(0,\sigma_i^2))$, the $k$-th eigenvalue of $m_{r,i}$ is $r^k$. Therefore, the multi-dimensional Mehler kernel $M_{\theta,d}$ satisfies Assumption \ref{product structure assumption}. 

\subsection{Other Kernels with Mixed Components}
One may also consider kernels with mixed components, such as 
$$
k_{\text{mixed}}(x,y)=\prod_{i=1}^{d_1} \hat{k}_{ i}(x_i, y_i)\prod_{i=d_1+1}^{d_1+d_2} l_{\theta_i/d,i}(x_i,y_i)\prod_{i=d_1+d_2+1}^{d_1+d_2+d_3} m_{\theta_i/d,i}(x_i,y_i)
$$
on the respective distribution.

\section{Proofs}
\subsection{Proof of Theorem \ref{main theorem}}
\label{proof of main theorem}

The proof of Theorem \ref{main theorem} relies on the classical bias-variance decomposition framework. We begin by establishing key notations and definitions.

Denote $T_\lambda = T + \lambda$ and $T_{\mathbf{X}\lambda} = T_{\mathbf{X}} + \lambda$, where $\lambda > 0$ is the regularization parameter. Define $S_k$ as the embedding operator $S_k:\mathcal{H}\rightarrow L^2(\mathcal{X},\mu)$, and its adjoint operator as $S_{k}^{*}: L^{2}(\mathcal{X},\mu) \to \mathcal{H}$, satisfying 
\begin{displaymath}
\left(S_{k}^{*} f\right)({x})=\int_\mathcal{X} k\left({x}, {x}^{\prime}\right) f\left({x}^{\prime}\right) \mathrm{d} \mu\left({x}^{\prime}\right),
\end{displaymath}
and denote $T=S_k^* S_k$.

For operator norms, we rewrite $L^2(\mathcal{X}, \mu)$ as $L^2$ and $L^\infty(\mathcal{X}, \mu)$ as $L^\infty$ for brevity, and use the following conventions:
\begin{itemize}
\item[(i)] $\|\cdot\|_{\mathcal{B}(B_1, B_2)}$ denotes the operator norm between Banach spaces $B_1$ and $B_2$, and is defined as $\|A\|_{\mathcal{B}(B_1, B_2)} = \sup_{\|f\|_{B_1}=1} \|Af\|_{B_2}.$ When the context is clear, we write $\|\cdot\|$ for the operator norm.
\item[(ii)] $\operatorname{tr} A$ and $\|A\|_1$ denote the trace and trace norm of operator $A$, while $\|A\|_2$ represents the Hilbert-Schmidt norm.
\end{itemize}

Next, we shall introduce the following operators. Given the sample set $\mathbf{Z} = \{(x_i, y_i)\}_{i=1}^n$, we define the sampling operator $K_{x}: \mathbb{R} \rightarrow \mathcal{H},$ $y \mapsto yk(x, \cdot)$. Its adjoint operator $K_{x}^*: \mathcal{H} \rightarrow \mathbb{R}$ satisfies $K_{x}^*f = f(x)$. Also, we define the operator $T_{x} = K_{x}K_{x}^*$, and the empirical covariance operator is defined as:
\[
T_{\mathbf{X}} := \frac{1}{n}\sum_{i=1}^n K_{x_i}K_{x_i}^*
\]
Note that $\|T_{\mathbf{X}}\| \leq \|T_{\mathbf{X}}\|_1 \leq \kappa^2$, and $T_{\mathbf{X}}$ is a compact trace-class operator.

Further, if we define the sample basis function as:
\[
g_{\mathbf{Z}} := \frac{1}{n}\sum_{i=1}^n K_{x_i}y_i \in \mathcal{H},
\]
following \cite{caponnetto2007optimal}, the KRR estimator has the operator representation:
\[
\hat{f}_\lambda = (T_{\mathbf{X}} + \lambda)^{-1}g_{\mathbf{Z}}.
\]
To analyze the bias term, we introduce conditional expectations:
\[
\tilde{g}_{\mathbf{Z}} := \mathbb{E}(g_{\mathbf{Z}} | \mathbf{X}) = \frac{1}{n}\sum_{i=1}^n K_{x_i}f_\rho^*(x_i),
\]
\[
\tilde{f}_\lambda := \mathbb{E}(\hat{f}_\lambda | \mathbf{X}) = (T_{\mathbf{X}} + \lambda)^{-1}\tilde{g}_{\mathbf{Z}}.
\]
Also, the expectation of ${g}_{\mathbf{Z}}$ can be defined as:
\[
g := \mathbb{E}g_{\mathbf{Z}} = \int_{\mathcal{X}} k({x}, \cdot)f_\rho^*({x})d\mu({x}) = S_k^*f_\rho^*,
\]
and
\[
f_\lambda = (T + \lambda)^{-1}g = (T + \lambda)^{-1}S_k^*f_\rho^*.
\]

Now we are ready to give the bias-variance decomposition. The estimation error decomposes as:
\begin{align*}
\hat{f}_\lambda - f_\rho^* &= \frac{1}{n}(T_{\mathbf{X}} + \lambda)^{-1}\sum_{i=1}^n K_{x_i}y_i - f_\rho^* \\
&= (T_{\mathbf{X}} + \lambda)^{-1}\left[\frac{1}{n}\sum_{i=1}^n K_{x_i}(f_\rho^*(x_i) + \epsilon_i)\right] - f_\rho^* \\
&= (T_{\mathbf{X}} + \lambda)^{-1}\tilde{g}_{\mathbf{Z}} + \frac{1}{n}\sum_{i=1}^n (T_{\mathbf{X}} + \lambda)^{-1}K_{x_i}\epsilon_i - f_\rho^* \\
&= (\tilde{f}_\lambda - f_\rho^*) + \frac{1}{n}\sum_{i=1}^n (T_{\mathbf{X}} + \lambda)^{-1}K_{x_i}\epsilon_i
\end{align*}

Taking conditional expectation with respect to the noise $\epsilon$ given $\mathbf{X}$, and noting that $\{\epsilon_i|x_i\}$ are independent with $\mathbb{E}[\epsilon_i] = 0$ and $\operatorname{Var}(\epsilon_i) = \sigma_\epsilon^2$, we obtain:

\[
\mathbb{E}\left[\|\hat{f}_\lambda - f_\rho^*\|_{L^2}^2 \big| \mathbf{X}\right] = \operatorname{Bias}^2(\lambda) + \operatorname{Var}(\lambda)
\]

where
\begin{equation}
\label{decomposition}
\operatorname{Bias}^2(\lambda) = \|\tilde{f}_\lambda - f_\rho^*\|_{L^2}^2, 
\operatorname{Var}(\lambda) = \frac{\sigma_\epsilon^2}{n^2}\sum_{i=1}^n \|(T_{\mathbf{X}} + \lambda)^{-1}k(x_i, \cdot)\|_{L^2}^2
\end{equation}

The subsequent analysis will establish bounds for both $\operatorname{Bias}^2(\lambda)$ and $\operatorname{Var}(\lambda)$.

\subsubsection{Variance Term}

In this subsection, we aim at deriving Theorem \ref{theorem variance approximation}, which establishes the bound of the variance term. We shall proceed through several preparatory steps.
First, we define the sample subspace 
\[
\mathcal{H}_n=\operatorname{span}\left\{k\left( {x}_1, \cdot\right), \ldots, k\left( {x}_n, \cdot\right)\right\} \subset \mathcal{H}.
\]
Also, recall the notation for kernel matrices:
\[
\mathbb{K}( \mathbf{X},  \mathbf{X})=\left(k\left( {x}_i,  {x}_j\right)\right)_{n \times n}, \quad \mathbb{K}( \mathbf{X}, \cdot)=\left\{k\left( {x}_1, \cdot\right), \ldots, k\left( {x}_n, \cdot\right)\right\},
\]
and define the normalized sample kernel matrix as
\[
 {K}=\frac{1}{n} \mathbb{K}( \mathbf{X},  \mathbf{X}).
\]
Notice that $\operatorname{Ran}\left(T_{ \mathbf{X}}\right)=\mathcal{H}_n$, and $ {K}$ represents $T_{ \mathbf{X}}$ under the basis $\left\{k\left( {x}_1, \cdot\right), \ldots, k\left( {x}_n, \cdot\right)\right\}$. Hence, for any continuous function $\varphi$, the following equation holds:
\begin{equation}\label{eq:TXMatrixForm}
\varphi\left(T_{ \mathbf{X}}\right) \mathbb{K}( \mathbf{X}, \cdot)=\varphi( {K})\mathbb{K}(\mathbf{X}, \cdot).
\end{equation}
Consider the inner product elementwise between \ref{eq:TXMatrixForm} and $f$, we get
\begin{equation}
\label{eq:TXActionF}
\left(\varphi\left(T_{ \mathbf{X}}\right) f\right)[ \mathbf{X}]=\varphi( {K}) f[ \mathbf{X}].
\end{equation}
Also, we shall define the following inner product:
\[
\langle f, g\rangle_{L^2, n}=\frac{1}{n} \sum_{i=1}^n f\left( {x}_i\right) g\left( {x}_i\right)=\frac{1}{n} f[ \mathbf{X}]^{\top} g[ \mathbf{X}].
\]
Throughout the proof, we shall keep the abbreviation $k_x(\cdot)=k(x,\cdot)$ for brevity. The following lemma provides a transformation of the variation term.

\begin{lemma}\label{lemma var transform}
The variance term satisfies:
\[
\operatorname{Var}(\lambda)=\frac{\sigma_\epsilon^2}{n} \int_\mathcal{X}\left\|\left(T_{ \mathbf{X}}+\lambda\right)^{-1} k_x(\cdot)\right\|_{L^2, n}^2 {~d} \mu( {x}).
\]
\end{lemma}

\begin{proof}
\begin{align*}
\operatorname{Var}(\lambda) &= \frac{\sigma_\epsilon^2}{n^2} \sum_{i=1}^n\left\|\left(T_{ \mathbf{X}}+\lambda\right)^{-1} k\left( {x}_i, \cdot\right)\right\|_{L^2}^2 \\
&= \frac{\sigma_\epsilon^2}{n^2}\left\|(T_{ \mathbf{X}}+\lambda)^{-1} \mathbb{K}( \mathbf{X}, \cdot)\right\|_{L^2\left(\mathbb{R}^n\right)}^2 \\
&= \frac{\sigma_\epsilon^2}{n^2}\left\|( {K}+\lambda)^{-1} \mathbb{K}( \mathbf{X}, \cdot)\right\|_{L^2\left(\mathbb{R}^n\right)}^2 \quad \text{(by \eqref{eq:TXMatrixForm})} \\
&= \frac{\sigma_\epsilon^2}{n^2} \int_{\mathcal{X}} \mathbb{K}( {x},  \mathbf{X})( {K}+\lambda)^{-2} \mathbb{K}( \mathbf{X},  {x}) {d} \mu( {x}).
\end{align*}

Using \eqref{eq:TXActionF}, noticing that $k_{ {x}}[ \mathbf{X}]=\mathbb{K}( \mathbf{X},  {x})$, we get:
\[
\left(\left(T_{ \mathbf{X}}+\lambda\right)^{-1} k_x\right)[ \mathbf{X}]=( {K}+\lambda)^{-1} \mathbb{K}( \mathbf{X},  {x}),
\]
and hence 
\begin{align*}
\frac{1}{n} \mathbb{K}( {x},  \mathbf{X})( {K}+\lambda)^{-2} \mathbb{K}( \mathbf{X},  {x}) &= \frac{1}{n}\left\|( {K}+\lambda)^{-1} \mathbb{K}( \mathbf{X},  {x})\right\|_{\mathbb{R}^n}^2 \\
&= \left\|\left(T_{ \mathbf{X}}+\lambda\right)^{-1} k_{ {x}}\right\|_{L^2, n}^2.
\end{align*}

Thus we conclude:
\[
\operatorname{Var}(\lambda)=\frac{\sigma_\epsilon^2}{n} \int_{\mathcal{X}}\left\|\left(T_{\mathbf{X}}+\lambda \right)^{-1} k_x(\cdot)\right\|_{L^2, n}^2 {~d} \mu( {x}).
\]
\end{proof}

In the remaining part of this subsection, we shall derive the following two approximations.
\begin{align*}
&\int_\mathcal{X}\left\|(T_{ \mathbf{X}}+\lambda)^{-1} k_x\right\|_{L^2, n}^2 {~d} \mu( {x})= \int_\mathcal{X}\left\|T_{ \mathbf{X} \lambda}^{-1} k_x\right\|_{L^2, n}^2 {~d} \mu( {x})\\
&\stackrel{A}{\approx} \int_\mathcal{X}\left\|T_\lambda^{-1} k_{ {x}}\right\|_{L^2, n}^2 {~d} \mu( {x}) \stackrel{B}{\approx} \int_\mathcal{X}\left\|T_\lambda^{-1} k_{ {x}}\right\|_{L^2}^2 {~d} \mu( {x}).
\end{align*}

\subsubsection{Approximation B}
\begin{lemma}[Approximation B]\label{lemma approximation B}
 Suppose that Assumption \ref{assumption kernel}, \ref{assumption noise}, \ref{assumption f*} hold. Furthermore, suppose that 
\begin{displaymath}
    \frac{\sup_{i\in N}\frac{\lambda_i}{(\lambda_i+\lambda)^2}}{n}=o(\mathcal{N}_2(\lambda)),
\end{displaymath}
in which $\lambda=\lambda(d, n) \rightarrow 0$. Then for any fixed $\delta \in(0,1)$, when $n$ is sufficiently large, with probability at least $1-\delta$, we have

\[
\frac{1}{2} \int_{\mathcal{X}}\left\|T_\lambda^{-1} k_x\right\|_{L^2}^2 {~d} \mu( {x})-\Delta_2 \leq \int_\mathcal{X}\left\|T_\lambda^{-1} k_x\right\|_{L^2, n}^2 {~d} \mu( {x}) \leq \frac{3}{2} \int_\mathcal{X}\left\|T_\lambda^{-1} k_{ {z}}\right\|_{L^2}^2 {~d} \mu( {x})+\Delta_2
\]

where

\[
\Delta_2=\frac{5\kappa^2 \sup_{i\in N}\frac{\lambda_i}{(\lambda_i+\lambda)^2}}{3 n} \ln \frac{2}{\delta}
\]
    
\end{lemma} 

\begin{proof}

We define the function $f({z})$ as follows:
\begin{equation}
\begin{aligned}
f({z}) &= \int_{\mathcal{X}}\left(T_\lambda^{-1} k_{{x}}({z})\right)^2 d\mu({x}) \\
&= \int_{\mathcal{X}} \sum_{i=1}^{\infty}\left(\frac{\lambda_i}{\lambda_i+\lambda}\right)^2 e_i^2({x}) e_i^2({z}) d\mu({x}) \\
&= \sum_{i=1}^{\infty}\left(\frac{\lambda_i}{\lambda_i+\lambda}\right)^2 e_i^2({z})
\end{aligned}
\end{equation}

The norms of $f$ satisfy:
\begin{equation}
\begin{aligned}
\|f\|_{L^{\infty}} &\leq \sup_{i\in\mathbb{N}}\frac{\lambda_i}{(\lambda_i+\lambda)^2} \|\sum_{i=1}^{\infty}\lambda_ie_i^2(x)\|_{L^{\infty}} \leq \kappa^2\sup_{i\in\mathbb{N}}\frac{\lambda_i}{(\lambda_i+\lambda)^2} \\
\|f\|_{L^1} &= \sum_{i=1}^{\infty}\left(\frac{\lambda_i}{\lambda_i+\lambda}\right)^2 = \mathcal{N}_2(\lambda)
\end{aligned}
\end{equation}

Applying Proposition \ref{prop:SampleNormEstimation} to $\sqrt{f}$ and noting that:
\[
\|\sqrt{f}\|_{L^{\infty}} = \sqrt{\|f\|_{L^{\infty}}} = \left(\kappa^2\sup_{i\in\mathbb{N}}\frac{\lambda_i}{(\lambda_i+\lambda)^2}\right)^{\frac{1}{2}}
\]
we obtain with probability at least $1-\delta$:
\begin{equation}
\frac{1}{2}\|\sqrt{f}\|_{L^2}^2 - \frac{5\kappa^2\sup_{i\in\mathbb{N}}\frac{\lambda_i}{(\lambda_i+\lambda)^2}}{3n} \ln\frac{2}{\delta} 
\leq \|\sqrt{f}\|_{L^2,n}^2 
\leq \frac{3}{2}\|\sqrt{f}\|_{L^2}^2 + \frac{5\kappa^2\sup_{i\in\mathbb{N}}\frac{\lambda_i}{(\lambda_i+\lambda)^2}}{3n} \ln\frac{2}{\delta}
\end{equation}

The empirical norm can be expressed as:
\begin{equation}
\begin{aligned}
\|\sqrt{f}\|_{L^2,n}^2 &= \int_\mathcal{X} f({y}) dP_n({y}) \\
&= \int_\mathcal{X}\left[\int_\mathcal{X}\left(T_\lambda^{-1} k_{{x}}({y})\right)^2 d\mu({x})\right] dP_n({y}) \\
&= \int_\mathcal{X}\left[\int_\mathcal{X}\left(T_\lambda^{-1} k_{{x}}({y})\right)^2 dP_n({y})\right] d\mu({x}) \\
&= \int_\mathcal{X}\left\|T_\lambda^{-1} k_x\right\|_{L^2,n}^2 d\mu({x})
\end{aligned}
\end{equation}

Similarly, the $L^2$ norm satisfies:
\begin{equation}
\begin{aligned}
\|\sqrt{f}\|_{L^2}^2 &= \int_{\mathcal{X}} f(z) d\mu(z) \\
&= \int_{\mathcal{X}}\left[\int_\mathcal{X}\left(T_\lambda^{-1} k_x(z)\right)^2 d\mu({x})\right] d\mu(z) \\
&= \int_{\mathcal{X}}\left\|T_\lambda^{-1} k_x\right\|_{L^2}^2 d\mu({x})
\end{aligned}
\end{equation}

These equalities establish the desired results.

\end{proof}

\subsubsection{Approximation A}

The proof of Approximation A relies on the following key observation by \citet{li2023_SaturationEffect}:

\begin{proposition}\label{proposition observation}
For any $f,g \in \mathcal{H}$, we have
\begin{equation}
\label{eq:SampleInnerProductsRelation}
\langle f,g \rangle_{L^2,n} = \langle T_{\mathbf{X}} f,g\rangle_{\mathcal{H}} = \langle T_{\mathbf{X}}^{1/2} f, T_{\mathbf{X}}^{1/2} g\rangle_{\mathcal{H}}.
\end{equation}
\end{proposition}

\begin{proof}
Since $T_{\mathbf{X}} f = \frac{1}{n}\sum_{i=1}^n f({x}_i) k({x}_{i},\cdot)$, we have:
\begin{align*}
\langle T_{\mathbf{X}} f,g \rangle_{\mathcal{H}} 
&= \frac{1}{n}\sum_{i=1}^n f({x}_i) \langle k({x}_{i},\cdot), g\rangle_{\mathcal{H}} \\
&= \frac{1}{n}\sum_{i=1}^n f({x}_i) g({x}_i) \\
&= \langle f,g \rangle_{L^2,n}.
\end{align*}
The second equality follows from the definition of $T_{\mathbf{X}}^{1/2}$.
\end{proof}
\begin{lemma}[Approximation A]\label{lemma approximation A}
    Suppose that Assumption \ref{assumption kernel}, \ref{assumption noise}, \ref{assumption f*} hold. Define $\Delta_1 = \Delta_1(\lambda, X)$ as a function of $\lambda$ and $X$:
    \begin{equation}\begin{aligned}\label{approximation A}
    \Delta_1 := \left| \int_{\mathcal{X}} \left\|{T_{\mathbf{X} \lambda}^{-1} k_{{x}}}\right\|_{L^2, n}^{2} {d}\mu({x}) - \int_{\mathcal{X}} \left\|{T_{\lambda}^{-1} k_{{x}}}\right\|_{L^2, n}^{2} {d}\mu({x})  \right|.
\end{aligned}\end{equation}
 Suppose that $ \lambda = \lambda(d,n)$ satisfies $\frac{\mathcal{N}_{1}(\lambda)}{n} \ln{n} = o(1)$, $\frac{\sup_{i\in N}\frac{\lambda_i}{(\lambda_i+\lambda)^2}}{n}=o(\mathcal{N}_2(\lambda))$. Then for any fixed $\delta \in (0,1)$, when $n$ is sufficiently large, with probability at least $1-\delta$, we have
    \begin{equation}\begin{aligned}
        \Delta_1 &\le 36 n^{-1} \mathcal{N}_{1}(\lambda)^{2} \ln{n} + 12 n^{-\frac{1}{2}} \mathcal{N}_{1}(\lambda) \mathcal{N}_{2}(\lambda)^{\frac{1}{2}} \left( \ln{n} \right)^{\frac{1}{2}}.
    \end{aligned}\end{equation}
\end{lemma}

\begin{proof}
We decompose $\Delta_1$ as:
\begin{equation}\begin{aligned}\label{proof appr A 1}
    \Delta_1 &= \left| \int_{\mathcal{X}} \left\|{T_{\mathbf{X} \lambda}^{-1} k_{{x}}}\right\|_{L^2, n}^{2} {d}\mu({x}) - \int_{\mathcal{X}} \left\|{T_{\lambda}^{-1} k_{{x}}}\right\|_{L^2, n}^{2} {d}\mu({x})  \right| \notag \\
    &\le  \int_{\mathcal{X}} \left| \left\|{T_{\mathbf{X}}^{\frac{1}{2}} T_{\mathbf{X} \lambda}^{-1} k_{{x}}}\right\|_{\mathcal{H}}^{2}  - \left\|{ T_{\mathbf{X}}^{\frac{1}{2}} T_{\lambda}^{-1} k_{{x}}}\right\|_{\mathcal{H}}^{2}  \right| {d}\mu({x}) \notag \\
    &= \int_{\mathcal{X}} \left| \left\|{T_{\mathbf{X}}^{\frac{1}{2}} T_{\mathbf{X} \lambda}^{-1} k_{{x}}}\right\|_{\mathcal{H}} -  \left\|{ T_{\mathbf{X}}^{\frac{1}{2}} T_{\lambda}^{-1} k_{{x}}}\right\|_{\mathcal{H}} \right| \cdot \left| \left\|{T_{\mathbf{X}}^{\frac{1}{2}} T_{\mathbf{X} \lambda}^{-1} k_{{x}}}\right\|_{\mathcal{H}} + \left\|{ T_{\mathbf{X}}^{\frac{1}{2}} T_{\lambda}^{-1} k_{{x}}}\right\|_{\mathcal{H}} \right| {d}\mu({x}). \notag \\
    &:= \int_{\mathcal{X}} \left| X_{1} - X_{2} \right| \cdot \left| X_{1} + X_{2} \right| {d}\mu({x}).
\end{aligned}\end{equation}
where we use Proposition \ref{proposition observation} in the second line.

\noindent\textbf{Part I: Bounding $|X_1 - X_2|$}
\begin{equation}\begin{aligned}\label{x1-x2 bound}
    \left| X_{1} - X_{2} \right| &\le \left\|{T_{\mathbf{X}}^{1/2}T_{\mathbf{X} \lambda}^{-1} \left( T - T_{\mathbf{X}} \right) T_{\lambda}^{-1}k_{{x}}}\right\|_{\mathcal{H}} \notag \\
    &\le \left\|{T_{\mathbf{X}}^{1/2}T_{\mathbf{X} \lambda}^{-1/2}}\right\| \cdot
      \left\|{T_{\mathbf{X} \lambda}^{-1/2} T_{\lambda}^{1/2}}\right\|
      \cdot \left\|{T_{\lambda}^{-1/2} \left( T - T_{\mathbf{X}} \right)T_{\lambda}^{-1/2} }\right\|
      \cdot \left\|{T_{\lambda}^{-1/2}k_{{x}}}\right\|_{\mathcal{H}}.
\end{aligned}\end{equation}

\noindent\textbf{Step (i):} Denote $A_i = T_\lambda^{-1/2}(T - T_{\mathbf{x}_i})T_\lambda^{-1/2}$. By Lemma \ref{emb norm}:
\[
\|A_i\| \le \mathcal{N}_1(\lambda)+\kappa^2/\lambda, \quad \mu\text{-a.e. } {x} \in \mathcal{X}.
\]
Again, by Lemma \ref{emb norm}, we have:
\[
\mathbb{E} A_i^2 \preceq \mathbb{E}[T_\lambda^{-1/2}T_{{x}_i}T_\lambda^{-1/2}]^2 \preceq (\kappa^2/\lambda)\mathbb{E}[T_\lambda^{-1/2}T_{{x}_i}T_\lambda^{-1/2}]\preceq(\kappa^2/\lambda)T_\lambda^{-1}T.
\]
Define $V := (\kappa^2/\lambda)T_\lambda^{-1}T$, we have:
\begin{equation}\begin{aligned}
      \| V \| &= \frac{\kappa^2}{\lambda} \frac{\lambda_{0}}{\lambda_{0} + \lambda} = \frac{\kappa^2}{\lambda} \frac{\|T\|}{\|T\| + \lambda} \le \frac{\kappa^2}{\lambda}; \notag \\
      \text{tr}V &= \frac{\kappa^2}{\lambda}\mathcal{N}_{1}(\lambda); \notag \\
      \frac{\text{tr}V}{ \| V \|} &= \frac{\mathcal{N}_{1}(\lambda) (\|T\| + \lambda)}{\|T\|}, \notag
  \end{aligned}\end{equation}
  where $\lambda_0$ is the biggest eigenvalue.
  Adopt Lemma \ref{lemma concentration of operator} to $A_{i}$, $V$, we get the following claim. 
  
  For any fixed $\delta \in (0,1)$, with probability at least $1-\delta$, we have
    \begin{displaymath}
        \Vert T_\lambda^{-\frac{1}{2}} (T - T_{\mathbf{X}}) T_\lambda^{-\frac{1}{2}} \Vert
        \le \frac{2 \left(\mathcal{N}_{1}(\lambda)+\frac{\kappa^2}{\lambda}\right)}{3n} \beta + \sqrt {\frac{ {2\kappa^2}}{n\lambda} \beta},
   \end{displaymath}
   where 
   \begin{displaymath}
       \beta = \ln{\frac{4 \mathcal{N}_{1}(\lambda) (\|T\| + \lambda) }{\delta \|T\|}}.
   \end{displaymath}
Recall that the condition $ \frac{1}{\lambda} \ln{n}/n = o(1)$ implies that $\frac{1}{\lambda} = O(n)$. Furthermore, notice that $\mathcal{N}_{1}(\lambda)\leq\sum_{i\in N}\frac{\lambda_i}{\lambda}\leq\frac{\kappa^2}{\lambda}$, we have  $\beta = O(\ln n)$, so when $n$ is sufficiently large, we can conclude that
\begin{equation}\label{x1-x2-1}
    \Vert T_\lambda^{-\frac{1}{2}} (T - T_{\mathbf{X}}) T_\lambda^{-\frac{1}{2}} \Vert \lesssim \sqrt {\frac{2\kappa^2}{n\lambda}\beta}
        \lesssim n^{-\frac{1}{2}} \frac{1}{\lambda^{\frac{1}{2}}} \left( \ln{n} \right)^{\frac{1}{2}}.
\end{equation}

\noindent\textbf{Step (ii):} By Lemma \ref{due embedding bound}:
\begin{equation}
\|T_\lambda^{-1/2}k({x},\cdot)\|_{\mathcal{H}} \le \sqrt{\kappa^2/\lambda}, \quad \mu\text{-a.e. }
\label{x1-x2-2}
\end{equation}

\noindent\textbf{Step (iii):} We now bound the first two terms in \eqref{x1-x2 bound}. Under the assumption that $\frac{1}{\lambda}\frac{\ln n}{n} = o(1)$, equation \eqref{x1-x2-1} implies that for sufficiently large $n$:

\begin{equation}\label{eq:a-bound}
a := \|T_\lambda^{-\frac{1}{2}}(T - T_{\mathbf{X}})T_\lambda^{-\frac{1}{2}}\| \leq \frac{2}{3}.
\end{equation}

This allows us to establish the following key operator norm bounds:

\begin{lemma}\label{lemma:operator-norms}
For sufficiently large $n$, the following inequalities hold with probability at least $1-\delta$:
\begin{enumerate}
\item $\|T_\lambda^{-\frac{1}{2}}T_{\mathbf{X}\lambda}^{\frac{1}{2}}\|^2 \leq 2$
\item $\|T_\lambda^{\frac{1}{2}}T_{\mathbf{X}\lambda}^{-\frac{1}{2}}\|^2 \leq 3$
\end{enumerate}
\end{lemma}

\begin{proof}
For the first bound:
\begin{equation}\begin{aligned}\label{x1-x2 3}
\|T_\lambda^{-\frac{1}{2}}T_{\mathbf{X}\lambda}^{\frac{1}{2}}\|^2 
&= \|T_\lambda^{-\frac{1}{2}}T_{\mathbf{X}\lambda}T_\lambda^{-\frac{1}{2}}\| \notag \\
&= \|T_\lambda^{-\frac{1}{2}}(T_{\mathbf{X}} + \lambda)T_\lambda^{-\frac{1}{2}}\| \notag \\
&= \|T_\lambda^{-\frac{1}{2}}(T_{\mathbf{X}} - T + T + \lambda)T_\lambda^{-\frac{1}{2}}\| \notag \\
&= \|T_\lambda^{-\frac{1}{2}}(T_{\mathbf{X}} - T)T_\lambda^{-\frac{1}{2}} + I\| \notag \\
&\leq a + 1 \leq 2.
\end{aligned}\end{equation}

For the second bound:
\begin{equation}\begin{aligned}\label{x1-x2 4}
\|T_\lambda^{\frac{1}{2}}T_{\mathbf{X}\lambda}^{-\frac{1}{2}}\|^2 
&= \|T_\lambda^{\frac{1}{2}}T_{\mathbf{X}\lambda}^{-1}T_\lambda^{\frac{1}{2}}\| \notag \\
&= \|(T_\lambda^{-\frac{1}{2}}T_{\mathbf{X}\lambda}T_\lambda^{-\frac{1}{2}})^{-1}\| \notag \\
&= \|(I - T_\lambda^{-\frac{1}{2}}(T - T_{\mathbf{X}})T_\lambda^{-\frac{1}{2}})^{-1}\| \notag \\
&\leq \sum_{k=0}^{\infty} \|T_\lambda^{-\frac{1}{2}}(T - T_{\mathbf{X}})T_\lambda^{-\frac{1}{2}}\|^k \notag \\
&\leq \sum_{k=0}^{\infty} \left(\frac{2}{3}\right)^k \leq 3.
\end{aligned}\end{equation}
\end{proof}

Combining the estimates \eqref{x1-x2-1}, \eqref{x1-x2-2}, \eqref{x1-x2 3}, and \eqref{x1-x2 4} into \eqref{x1-x2 bound}, we obtain that with high probability, the following inequality holds:

\begin{equation}\label{x1-x2 bound final}
|X_1 - X_2| \lesssim 6n^{-\frac{1}{2}}\frac{1}{\lambda}(\ln n)^{\frac{1}{2}}, \quad \mu\text{-a.e. } \mathbf{x} \in \mathcal{X}.
\end{equation}

\noindent\textbf{Part II: Bounding $\int X_2 d\mu$}

Now we shall bound $\int X_2 d\mu$. When $n$ is sufficiently large, with probability at least $1-\delta$, we have
\begin{equation}\begin{aligned}\label{x2 int}
    \int_{\mathcal{X}} X_{2} {d} \mu({x}) &= \int_{\mathcal{X}} \left\|{T_{\lambda}^{-1} k_{{x}}}\right\|_{L^2, n} {d} \mu({x}) \notag \\
    &\le \left[  \int_{\mathcal{X}} \left\|{T_{\lambda}^{-1} k_{{x}}}\right\|_{L^2, n}^{2}{d} \mu({x}) \right]^{\frac{1}{2}} \notag \\
    &\le \left(\frac{3}{2}\mathcal{N}_{2}(\lambda) + \Delta_2\right)^{\frac{1}{2}} \notag \\
    &\le \left(2 \mathcal{N}_{2}(\lambda)\right)^{\frac{1}{2}},
\end{aligned}\end{equation}
where the third line follows from Lemma \ref{lemma approximation B}, and the forth line follows from the assumption that $\frac{\sup_{i\in N}\frac{\lambda_i}{(\lambda_i+\lambda)^2}}{n}=o(\mathcal{N}_2(\lambda)).$

\noindent\textbf{Part III: Final Bound} 

Combining \eqref{x1-x2 bound final} and \eqref{x2 int}, we get
\begin{align*}
\Delta_1 &\leq\int_{\mathcal{X}} \left| X_{1} - X_{2} \right| \cdot \left| X_{1} + X_{2} \right| {d}\mu({x}) \lesssim 36n^{-1}\lambda^{-2}\ln n + 24n^{-1/2}\lambda^{-1}\mathcal{N}_2(\lambda)^{1/2}(\ln n)^{1/2}.
\end{align*}
\end{proof}

\subsubsection{Final Proof of the Variance Term}
Now we are ready to state the theorem about the variance term. 

\begin{theorem}[Variance term]\label{theorem variance approximation}
    Suppose that Assumption \ref{assumption kernel}, \ref{assumption noise}, \ref{assumption f*} hold. If the following approximation conditions hold for some $\lambda = \lambda(d,n) \to 0$:
    \begin{equation}\begin{aligned}\label{var conditions}
    \frac{\mathcal{N}_{1}(\lambda)}{n} \ln{n} = o(1); ~~\frac{\sup_{i\in N}\frac{\lambda_i}{(\lambda_i+\lambda)^2}}{n}=o(\mathcal{N}_2(\lambda)); ~~n^{-1} \frac{1}{\lambda^2} \ln{n} = o(\mathcal{N}_{2}(\lambda)),
\end{aligned}\end{equation}
then we have
\begin{equation}
    \operatorname{Var}(\lambda) = \Theta_{\mathbb{P}}\left(\frac{\mathcal{N}_{2}(\lambda)}{n}\right).
\end{equation}
\end{theorem}
\begin{proof}
    Lemma \ref{lemma var transform} has shown that
    \begin{displaymath}
      \operatorname{Var}(\lambda) = \frac{\sigma_\epsilon^2}{n} \int_{\mathcal{X}} \left\|(T_{\mathbf{X}}+\lambda)^{-1}k_{{x}}(\cdot)\right\|_{L^2,n}^2 {d} \mu({x}).
    \end{displaymath}
Denote $\Delta_1$ as in Lemma \ref{lemma approximation A}, then conditions \eqref{var conditions} and Lemma \ref{lemma approximation A} imply that 
\begin{displaymath}
    \Delta_1 = o_{\mathbb{P}}\left(\mathcal{N}_{2}(\lambda)\right).
\end{displaymath}
Further recall that in Lemma \ref{lemma approximation B}, we have defined
\begin{displaymath}
     \Delta_2 = \frac{5\kappa^2 \sup_{i\in N}\frac{\lambda_i}{(\lambda_i+\lambda)^2}}{3 n} \ln \frac{2}{\delta}= o(\mathcal{N}_{2}(\lambda)).
\end{displaymath}
Then for any $\delta \in (0,1)$, when $n$ is sufficiently large, with probability at least $1-\delta$, we have
\begin{equation}\begin{aligned}
    n \operatorname{{Var}}(\lambda) / \sigma_\epsilon^{2} = \int_{\mathcal{X}} \left\|{T_{\mathbf{X} \lambda}^{-1} k_{{x}}}\right\|_{L^2, n}^{2} {d}\mu({x}) &\le  \int_{\mathcal{X}} \left\|{T_{\lambda}^{-1} k_{{x}}}\right\|_{L^2, n}^{2} {d}\mu({x}) + \Delta_1 \notag \\
    &\le \frac{3}{2} \int_{\mathcal{X}} \left\|{T_{\lambda}^{-1} k_{{x}}}\right\|_{L^2}^{2} {d}\mu({x}) + \Delta_1 + \Delta_2 \notag \\
    &= \frac{3}{2} \mathcal{N}_{2}(\lambda) + \Delta_1 + \Delta_2, \notag
\end{aligned}\end{equation}
and 
\begin{equation}\begin{aligned}
    n \operatorname{{Var}}(\lambda) / \sigma_\epsilon^{2} = \int_{\mathcal{X}} \left\|{T_{\mathbf{X} \lambda}^{-1} k_{{x}}}\right\|_{L^2, n}^{2} {d}\mu({x}) &\ge  \int_{\mathcal{X}} \left\|{T_{\lambda}^{-1} k_{{x}}}\right\|_{L^2, n}^{2} {d}\mu({x}) - \Delta_1 \notag \\
    &\ge \frac{1}{2} \int_{\mathcal{X}} \left\|{T_{\lambda}^{-1} k_{{x}}}\right\|_{L^2}^{2} {d}\mu({x}) - \Delta_1 - \Delta_2 \notag \\
    &= \frac{1}{2} \mathcal{N}_{2}(\lambda) - \Delta_1 - \Delta_2, \notag
\end{aligned}\end{equation}
which further implies
\begin{equation}\label{var appr upper}
    n \operatorname{{Var}}(\lambda) / \sigma_\epsilon^{2} = \Theta_{\mathbb{P}}\left(\mathcal{N}_{2}(\lambda)\right).
\end{equation}
\end{proof}

\subsubsection{Bias term}\label{section bias term}
In this subsection, we shall derive Theorem \ref{theorem bias approximation}, which shows the upper bound of bias under some approximation conditions. 

The triangle inequality implies that 
\begin{equation}\label{proof bias thm-1}
    \operatorname{{Bias}}(\lambda) = \left\| \tilde{f}_{\lambda} - f_{\rho}^{*}\right\|_{L^{2}} \le \left\| f_{\lambda} - f_{\rho}^{*}\right\|_{L^{2}} + \left\| \tilde{f}_{\lambda} - f_{\lambda}\right\|_{L^{2}},
\end{equation}
and we shall bound the two terms separately.

The following lemma characterizes the first term of ${{Bias}}(\lambda) $.
\begin{lemma}\label{lemma bias main term}
    Suppose that Assumption \ref{assumption kernel}, \ref{assumption noise}, \ref{assumption f*} hold. Then for any $ \lambda = \lambda(d,n) \to 0$, we have
    \begin{equation}\label{goal bias bound}
    \left\| f_{\lambda} - f_{\rho}^{*}\right\|_{L^{2}} = \mathcal{R}_{2}(\lambda)^{\frac{1}{2}}.
\end{equation}
\end{lemma}
\begin{proof}
    Recall that we have defined $ f_{\rho}^{*} = \sum\limits_{i=0}^{\infty} f_{i} e_{i}( {x}) \in L^{2}(\mathcal{X},\mu)$ and $ f_{\lambda} = \left( T + \lambda \right)^{-1} S_{k}^{*} f_{\rho}^{*}$. Therefore, we have
    \begin{equation}\begin{aligned}
        \left\| f_{\lambda} - f_{\rho}^{*}\right\|_{L^{2}}^{2} &= \left\|\sum\limits_{i=0}^{\infty} f_{i} e_{i}( {x}) - \sum\limits_{i=0}^{\infty} \frac{\lambda_{i}}{\lambda_{i} +  \lambda} f_{i} e_{i}( {x}) \right\|_{L^{2}}^{2} \notag \\
        &= \left\|\sum\limits_{i=0}^{\infty} \frac{\lambda}{\lambda_{i} +  \lambda} f_{i} e_{i}( {x}) \right\|_{L^{2}}^{2} \notag \\
        &= \sum\limits_{i=0}^{\infty} \left(\frac{\lambda}{\lambda_{i} +  \lambda} f_{i}\right)^{2} \notag \\
        &= \mathcal{R}_{2}(\lambda). \notag
    \end{aligned}\end{equation}
\end{proof}

The following lemma characterizes the second term of ${{Bias}}(\lambda) $.
\begin{lemma}\label{lemma bias appr term}
Suppose that Assumption \ref{assumption kernel}, \ref{assumption noise}, \ref{assumption f*} hold. Further more, suppose that there exist constants $c_1,c_2$ regardless of $d$ such that $c_1\leq\lambda_0\leq c_2$. If the following conditions hold for some $\lambda = \lambda(d,n) \to 0$:
    \begin{align}\label{bias conditions}
    \frac{\mathcal{N}_{1}(\lambda)}{n} \ln{n} = o(1); ~~\frac{1}{n\lambda}=o(1) ; ~~\mathcal{N}_2(\lambda)=\Omega(1); ~~n^{-1} \frac{1}{\lambda^2} \ln{n} = o(\mathcal{N}_{2}(\lambda)).
\end{align}
then we have
    \begin{equation}
        \left\| \tilde{f}_{\lambda} - f_{\lambda}\right\|_{L^{2}} \le n^{-\frac{1}{2}}o((\mathcal{N}_2(\lambda))^\frac{1}{2})+o(\mathcal{R}_2(\lambda)^{1/2}).
    \end{equation}
\end{lemma}

\begin{proof}
We begin by decomposing the $L^2$ norm difference:

\begin{equation}\begin{aligned}\label{appr proof-0}
\left\|\tilde{f}_\lambda - f_{\lambda}\right\|_{L^{2}} 
&= \left\| S_{k} (\tilde{f}_\lambda-f_\lambda)\right\|_{L^{2}} \notag \\
&= \left\| S_{k} T_\lambda^{-\frac{1}{2}} \cdot T_\lambda^{\frac{1}{2}} T_{\mathbf{X}\lambda}^{-1} T_\lambda^{\frac{1}{2}} \cdot T_\lambda^{-\frac{1}{2}} T_{\mathbf{X}\lambda} (\tilde{f}_\lambda-f_\lambda)\right\|_{L^{2}} \notag \\
&\leq \left\| S_{k} T_\lambda^{-\frac{1}{2}}\right\| \cdot \left\| T_{\lambda}^{\frac{1}{2}} T_{\mathbf{X}\lambda}^{-1} T_{\lambda}^{\frac{1}{2}} \right\| \cdot \left\| T_{\lambda}^{-\frac{1}{2}} (\tilde{g}_{\mathbf{Z}} - T_{\mathbf{X}\lambda} f_{\lambda}) \right\|_{\mathcal{H}}.
\end{aligned}\end{equation}

\noindent\textbf{Part I: First Term Bound}
For any $f \in \mathcal{H}$ with $\|f\|_{\mathcal{H}}=1$, let $f = \sum_{i=1}^{\infty} a_i \lambda_i^{1/2} e_i$ where $\sum_{i=1}^{\infty} a_i^2 = 1$. Then:

\begin{equation}\begin{aligned}\label{bias appr plug 1}
\left\| S_k T_\lambda^{-\frac{1}{2}}\right\| 
&= \sup_{\|f\|_{\mathcal{H}}=1} \left\|S_k T_\lambda^{-\frac{1}{2}} f\right\|_{L^2} \notag \\
&\leq \sup_{\|f\|_{\mathcal{H}}=1} \left\|\sum_{i\in N} \frac{\lambda_i^{1/2}}{(\lambda_i+\lambda)^{1/2}} a_i e_i \right\|_{L^2} \notag \\
&\leq \sup_{i \in N} \frac{\lambda_i^{1/2}}{(\lambda_i+\lambda)^{1/2}} \cdot \sup_{\|f\|_{\mathcal{H}}=1} \left\|\sum_{i\in N} a_i e_i \right\|_{L^2} \notag \\
&\leq 1.
\end{aligned}\end{equation}

\noindent\textbf{Part II: Second Term Bound}
Under the assumption $\mathcal{N}_1(\lambda)\ln n/n = o(1)$, for any $\delta \in (0,1)$ and sufficiently large $n$, with probability $\geq 1-\delta$:

\begin{equation}\label{bias appr plug 2}
\left\| T_{\lambda}^{\frac{1}{2}} T_{\mathbf{X}\lambda}^{-1} T_{\lambda}^{\frac{1}{2}} \right\| \leq \left\| T_{\lambda}^{\frac{1}{2}} T_{\mathbf{X}\lambda}^{-\frac{1}{2}} \right\| \cdot \left\| T_{\mathbf{X}\lambda}^{-\frac{1}{2}} T_{\lambda}^{\frac{1}{2}} \right\| \leq \frac{3}{2}.
\end{equation}

\noindent\textbf{Part III: Third Term Bound}
The third term decomposes as:

\begin{align}\label{appr proof-1}
\left\| T_{\lambda}^{-\frac{1}{2}} (\tilde{g}_{\mathbf{Z}} - T_{\mathbf{X}\lambda} f_\lambda) \right\|_{\mathcal{H}}
&= \left\|T_\lambda^{-\frac{1}{2}}[(\tilde{g}_{\mathbf{Z}} - (T_{\mathbf{X}} + \lambda + T - T) f_\lambda]\right\|_{\mathcal{H}} \notag \\
&= \left\|T_\lambda^{-\frac{1}{2}}[(\tilde{g}_{\mathbf{Z}} - T_{\mathbf{X}} f_\lambda) - (T + \lambda) f_\lambda + T f_\lambda]\right\|_{\mathcal{H}} \notag \\
&= \left\|T_\lambda^{-\frac{1}{2}}[(\tilde{g}_{\mathbf{Z}}-T_{\mathbf{X}} f_\lambda)-(g-T f_\lambda)]\right\|_{\mathcal{H}}.
\end{align}

Define $\xi_i = \xi(x_i) = T_\lambda^{-\frac{1}{2}}(K_{x_i} f_\rho^*(x_i) - T_{x_i} f_\lambda)$. For the $m$-th moment:

\begin{align}\label{proof of 4.9-1}
\mathbb{E} \|\xi(x)\|_{\mathcal{H}}^{m} 
&= \mathbb{E} \| T_\lambda^{-\frac{1}{2}} K_{x}(f_\rho^* - f_\lambda(x)) \|_{\mathcal{H}}^{m} \notag \\
&\leq \mathbb{E} \left( \| T_\lambda^{-\frac{1}{2}}k(x,\cdot)\|_{\mathcal{H}}^{m} \mathbb{E}[|f_\rho^* - f_\lambda(x)|^{m} | x] \right).
\end{align}

From Lemma \ref{due embedding bound}, we have:

\begin{equation*}
\| T_\lambda^{-\frac{1}{2}} k(x,\cdot)\|_{\mathcal{H}} \leq \sqrt{\kappa^2/\lambda}, \quad \mu\text{-a.e. } x \in \mathcal{X}
\end{equation*}

Using Lemma \ref{lemma bias main term}, we are able to provide the following bound:

\begin{align*}
(\ref{proof of 4.9-1}) &\leq \left(\frac{\kappa^2}{\lambda}\right)^{m/2} \| f_\lambda - f_\rho^* \|_{L^{\infty}}^{m-2} \mathcal{R}_2(\lambda) \notag \\
&\leq \left( \sqrt{\frac{\kappa^2}{\lambda}} \|f_\lambda - f_\rho^*\|_{L^{\infty}} \right)^{m-2} \left( \sqrt{\frac{\kappa^2}{\lambda}} \mathcal{R}_2(\lambda)^{1/2} \right)^2.
\end{align*}

Applying Lemma \ref{bernstein} with $L = \sqrt{\frac{\kappa^2}{\lambda}} \|f_\lambda - f_\rho^*\|_{L^{\infty}}$ and $\sigma_\epsilon = \sqrt{\frac{\kappa^2}{\lambda}} \mathcal{R}_2(\lambda)^{1/2}$, with probability $\geq 1-\delta$:

\begin{equation*}
(\ref{appr proof-1}) \leq 4\sqrt{2} \log \frac{2}{\delta} \left( \frac{\sqrt{\frac{\kappa^2}{\lambda}} \|f_\lambda - f_\rho^*\|_{L^{\infty}}}{n} + \frac{\sqrt{\frac{\kappa^2}{\lambda}} \mathcal{R}_2(\lambda)^{1/2}}{\sqrt{n}} \right).
\end{equation*}

\noindent\textbf{Norm Bound on $\|f_\lambda - f_\rho^*\|_{L^{\infty}}$}:
Notice that $\|f_\lambda - f_\rho^*\|_{L^{\infty}}\leq\|f_\lambda\|_{L^{\infty}}+\|f_\rho^*\|_{L^{\infty}}$. For every $f$ in $\mathcal{H}$, we have $f(x)=\langle k(x,\cdot), f\rangle_{\mathcal{H}}\leq\|k(x,\cdot)\|_{\mathcal{H}}\|f\|_{\mathcal{H}}\leq\kappa\|f\|_{\mathcal{H}}.$ Hence, when $ s\ge 1$, $f_\rho^*\in \mathcal{H}$, which indicates that there exists a constant $R_2$ such that $\|f_\rho^*\|_{L^{\infty}}\leq R_2$. Also,  
\begin{align*}
    \|f_\lambda\|_{L^{\infty}}
    &\lesssim\|f_\lambda\|_{\mathcal{H}}\\
    &=\left\|\sum_{i=0}^{\infty}\frac{\lambda_i}{\lambda_i+\lambda}f_ie_i\right\|_\mathcal{H} \\
    &= \left(\sum_{i=0}^{\infty}\frac{\lambda_i^{1+s}}{(\lambda_i+\lambda)^2}\lambda_i^{-s}f_i^2\right)^{1/2} \\
    &\leq \left(\sup_{i \in N}\frac{\lambda_i^{1+s}}{(\lambda_i+\lambda)^2}\right)^{\frac{1}{2}}\cdot\left(\sum_{i=0}^{\infty}\lambda_i^{-s}f_i^2\right)^{1/2} \\
    &\leq \left(\sup_{i \in N}\frac{\lambda_i^{1+s}}{(\lambda_i+\lambda)^2}\right)^{\frac{1}{2}}\cdot R_1\\
    &\leq \lambda_0^\frac{s-1}{2} \cdot R_1\\
    &\lesssim 1.
\end{align*}
Hence we get 
\begin{equation*}
    \|f_\lambda - f_\rho^*\|_{L^{\infty}}\lesssim1+R_2
\end{equation*}

Thus the final bound becomes:
\begin{equation*}
\begin{aligned}
(\ref{appr proof-1}) 
&\lesssim 4\sqrt{2} \log \frac{2}{\delta} \left( \frac{\sqrt{\frac{\kappa^2}{\lambda}}\left(1+R_2\right)}{n} + \frac{\sqrt{\frac{\kappa^2}{\lambda}} \mathcal{R}_2(\lambda)^{1/2}}{\sqrt{n}} \right)\\
&=O(\frac{1}{n\sqrt{\lambda}}+\frac{\mathcal{R}_2(\lambda)^{1/2}}{\sqrt{n\lambda}})\\
&=o(n^{-\frac{1}{2}})+o(\mathcal{R}_2(\lambda)^{1/2})\\
&=o\left(\left(\frac{\mathcal{N}_{2}(\lambda)}{n}\right)^{\frac{1}{2}}\right)+o(\mathcal{R}_2(\lambda)^{1/2}).
\end{aligned}
\end{equation*}
\end{proof}

\textit{Final proof of the bias term.} Now we are ready to state the theorem about the bias term.
\begin{theorem}\label{theorem bias approximation}
    Suppose that Assumption \ref{assumption kernel}, \ref{assumption noise}, \ref{assumption f*} hold. If the following condition holds for some $\lambda = \lambda(d,n) \to 0$:
    \begin{equation}\begin{aligned}\label{bias conditions final}
    \frac{\mathcal{N}_{1}(\lambda)}{n} \ln{n} = o(1); ~~\frac{1}{n\lambda}=o(1) ; ~~\mathcal{N}_2(\lambda)=\Omega(1); ~~n^{-1} \frac{1}{\lambda^2} \ln{n} = o(\mathcal{N}_{2}(\lambda)),
\end{aligned}\end{equation}
then we have
\begin{equation}\label{bias theorem proof condition 2}
    \operatorname{Bias}^{2}(\lambda) =\Theta_{\mathbb{P}}\left( \mathcal{R}_{2}(\lambda)+o\left(\frac{\mathcal{N}_{2}(\lambda)}{n}\right) +\frac{ 1 }{n^2\lambda} + \frac{ \mathcal{R}_{2}(\lambda)}{n\lambda} \right) .
\end{equation}
\end{theorem}

\subsubsection{Final proof of Theorem \ref{main theorem}}\label{section final proof of main theorem}
Now we are ready to prove Theorem \ref{main theorem}. Note that
\begin{equation*}
    \frac{\sup_{i\in N}\frac{\lambda_i}{(\lambda_i+\lambda)^2}}{n}\leq\frac{1}{4\lambda n}=o(1)=o(\mathcal{N}_2(\lambda)),
\end{equation*}
$ \lambda = \lambda(d,n) \to 0$ in Theorem \ref{main theorem} satisfies all the conditions required in Theorem \ref{theorem variance approximation} and Theorem \ref{theorem bias approximation}. Therefore, Theorem \ref{theorem variance approximation} and Theorem \ref{theorem bias approximation} show that
\begin{displaymath}
    \operatorname{{Var}}(\lambda) =\Theta_{\mathbb{P}}\left(\frac{ \mathcal{N}_{2}(\lambda)}{n}\right);~~ \operatorname{Bias}^{2}(\lambda) = \Theta_{\mathbb{P}}\left(\mathcal{R}_{2}(\lambda)+o\left(\frac{\mathcal{N}_{2}(\lambda)}{n}\right) \right).
\end{displaymath}
Recalling the bias-variance decomposition , we finish the proof.
\subsection{Proof of Theorem \ref{main theorem s<1}}
The only difference lies in the \textbf{Part III: Third Term Bound} in Lemma \ref{lemma bias appr term}. 

Define $\xi_i = \xi(x_i) = T_\lambda^{-\frac{1}{2}}(K_{x_i} f_\rho^*(x_i) - T_{x_i} f_\lambda)$. Consider the subset $\Omega_{1} = \{x \in \mathcal{X}: |f_{\rho}^{*}(x)| \le t \}$ and $\Omega_{2} = \mathcal{X} \backslash \Omega_{1}$, where $t$ will be chosen later. We have the following decomposition of \eqref{appr proof-1}:
\begin{align}\label{decomposition of bias}
    \left\|\frac{1}{n} \sum_{i=1}^n \xi_i-\mathbb{E} \xi(x)\right\|_\mathcal{H} &\le \left\|\frac{1}{n} \sum_{i=1}^n \xi_i I_{x_{i} \in \Omega_{1}}-\mathbb{E} \xi(x) I_{x \in \Omega_{1}} \right\|_\mathcal{H} + \| \frac{1}{n} \sum_{i=1}^n \xi_i I_{x_{i} \in \Omega_{2}} \|_{_\mathcal{H}} + \| \mathbb{E} \xi(x) I_{x \in \Omega_{2}} \|_{_\mathcal{H}} \notag \\
    &:= \text{\uppercase\expandafter{\romannumeral1}} + \text{\uppercase\expandafter{\romannumeral2}} + \text{\uppercase\expandafter{\romannumeral3}}.
\end{align}
Next we choose $t = n^{\frac{1-s}{2} + \epsilon_{1}}, q = \frac{2}{1-s}-\epsilon_{2} $ such that 
\begin{equation}\label{choose t q}
   \epsilon_{1} < \epsilon;~~\text{and}~~ \frac{1-s}{2} + \epsilon_{1} > 1 / \left( \frac{2}{1-s}-\epsilon_{2} \right),
\end{equation}
where $ \epsilon $ is given in \eqref{add condition s<1}. Then we can bound the three terms in \eqref{decomposition of bias} as follows:

$\left( \text{\lowercase\expandafter{\romannumeral1}}\right)~$For the first term in \eqref{decomposition of bias}, denoted as $\text{\uppercase\expandafter{\romannumeral1}}$, notice that
\begin{align}
     \left\| \left(f_{\lambda} - f_{\rho}^{*}\right)I_{x_{i} \in \Omega_{1}} \right\|_{L^{\infty}} \le \left\| f_{\lambda}\right\|_{L^{\infty}} + n^{\frac{1-s}{2}+\epsilon_{1}}.
\end{align}
Imitating the third part in the proof of Lemma \ref{lemma bias appr term}, we have
\begin{equation}\label{plug s le 1-1}
    \text{\uppercase\expandafter{\romannumeral1}} \lesssim 4\sqrt{2} \log \frac{2}{\delta} \left( \frac{\sqrt{\frac{\kappa^2}{\lambda}}(n^{\frac{1-s}{2}+\epsilon_{1}}+\|f_{\lambda}\|_{L^{\infty}})}{n} + \frac{\sqrt{\frac{\kappa^2}{\lambda}} \mathcal{R}_2(\lambda)^{1/2}}{\sqrt{n}} \right).
\end{equation}

$\left(\text{\lowercase\expandafter{\romannumeral2}}\right)~$ For the second term in \eqref{decomposition of bias}, denoted as $\text{\uppercase\expandafter{\romannumeral2}}$. Since $ q = \frac{2}{1-s}-\epsilon_{2} < \frac{2}{1-s}$, Lemma \ref{integrability of Hs constants} shows that,
\begin{align}
      [\mathcal{H}]^{s} \hookrightarrow L^{q}(\mathcal{X}, \mu).
\end{align}
Hence, there exists $0 < C_{q} < \infty$ such that $\| f_{\rho}^{*} \|_{L^{q}(\mathcal{X},\mu)} \le C_{q}$. Using the Markov inequality, we have
\begin{displaymath}
       P(x \in \Omega_{2}) = P\Big(|f_{\rho}^{*}(x)| > t \Big) \le \frac{\mathbb{E} |f_{\rho}^{*}(x)|^{q}}{t^{q}} \le \frac{(C_{q})^{q}}{t^{q}}.
\end{displaymath}
Further, since \eqref{choose t q} guarantees $ t^{q} \gg n$, we have 
\begin{align}\label{plug s le 1-2}
    \tau_{n} := P\left(\text{\uppercase\expandafter{\romannumeral2}} \ge \mathcal{R}_{2}(\lambda)^{\frac{1}{2}}\right) 
    &\le P\Big( ~\exists x_{i} ~\text{s.t.}~ x_{i} \in \Omega_{2}, \Big) \notag \\
    &= 1 - P\Big(x \notin \Omega_{2}\Big)^{n} \notag \\
    & \le 1 - \Big( 1 - \frac{(C_q)^{q}}{t^{q}}\Big)^{n} \to 0.
\end{align}
$\left(\text{\lowercase\expandafter{\romannumeral3}}\right)~$ For the third term in \eqref{decomposition of bias}, denoted as $\text{\uppercase\expandafter{\romannumeral3}}$. Since Lemma \ref{due embedding bound} implies that $\| T_{\lambda}^{-\frac{1}{2}} k(x,\cdot)\|_{\mathcal{H}} \le \sqrt{\frac{\kappa^2}{\lambda}}, \mu \text {-a.e. } x \in \mathcal{X},$ so
\begin{align}\label{third term}
    \text{\uppercase\expandafter{\romannumeral3}} &\le \mathbb{E}\| \xi(x) I_{x \in\Omega_{2}} \|_{\mathcal{H}} \le \mathbb{E}\Big[ \| T_{\lambda}^{-\frac{1}{2}} k(x,\cdot) \|_{\mathcal{H}} \cdot \big| \big(f_{\rho}^{*}-f_{\lambda}(x) \big) I_{x \in\Omega_{2}}\big| \Big] \notag \\
    &\le \sqrt{\frac{\kappa^2}{\lambda}} \mathbb{E} \big| \big(f_{\rho}^{*}-f_{\lambda}(x) \big) I_{x \in\Omega_{2}}\big| \notag \\
    &\le \sqrt{\frac{\kappa^2}{\lambda}} \left\| f_{\rho}^{*} - f_{\lambda}\right\|_{L^{2}}^{\frac{1}{2}} \cdot P\left( x \in \Omega_{2} \right)^{\frac{1}{2}} \notag \\
    &\lesssim \sqrt{\frac{\kappa^2}{\lambda}} \mathcal{R}_{2}(\lambda)^{\frac{1}{2}} t^{-\frac{q}{2}},
\end{align}
where we use Cauchy-Schwarz inequality for the third inequality and Lemma \ref{lemma bias main term} for the forth inequality. Recalling that the choices of $t, q$ satisfy $ t^{-q} \ll n^{-1}$ and we have assumed $\frac{1}{n\lambda}= o(1) $, we have 
\begin{equation}\label{plug s le 1-3}
    \text{\uppercase\expandafter{\romannumeral3}} = o\left( \mathcal{R}_{2}(\lambda)^{\frac{1}{2}} \right).
\end{equation}
Plugging \eqref{plug s le 1-1}, \eqref{plug s le 1-2} and \eqref{plug s le 1-3} into \eqref{decomposition of bias}, we finish the proof.

\subsection{Proof of Theorem \ref{thm:upper bound gaussian}}
\subsubsection{Proof of Proposition \ref{prop:staircase}}

\begin{proof}
Suppose the Mercer decomposition of the 1-dimensional kernel $\tilde{k}_{r_i,i}(x_i,x_i')$ is given by
\begin{equation}
    \hat{k}_{r_i,i}(x_i,x_i')=\sum_{j=1}^\infty\mu_j^{r_i,i}e_j^{r_i,i}(x_i)e_j^{r_i,i}(x_i'),
\end{equation}
where $e_j^{r_i,i}$ is the eigenfunction corresponding to the eigenvalue $\mu_j^{r_i,i}$. Then we have
\begin{equation} 
    \begin{aligned} 
        k_d(x,x') &= \prod_{i=1}^d\hat{k}_{r_i,i}(x_i,x_i') \\
        &= \prod_{i=1}^d\sum_{j=1}^\infty\mu_j^{r_i,i}e_j^{r_i,i}(x_i)e_j^{r_i,i}(x_i') \\
        &= \sum_{k=0}^\infty\sum_{|\alpha|=k}\mu_\alpha E_\alpha(x)E_\alpha(x'),
    \end{aligned} 
\end{equation}
where the second sum is taken over all multi-indices $\alpha=(\alpha_1,\alpha_2,\dots,\alpha_d)$ with $|\alpha|=\alpha_1+\dots+\alpha_d=k$, $\alpha_i\in\mathbb{N}$ for $i=1,\dots,d$, and
\begin{equation} 
    \mu_\alpha = \prod_{i=1}^d \mu^{r_i,i}_{\alpha_i}, \quad E_\alpha(x) = \prod_{i=1}^d e^{r_i,i}_{\alpha_i}(x_i).
\end{equation}
This is equivalent to the form  (\ref{Mercer decomposition of product}), with the index $i$ in (\ref{Mercer decomposition of product}) corresponding to $|\alpha|$, and
\begin{equation}
    N(d,k) = \binom{k+d-1}{d-1} = \frac{(k+d-1)!}{k!(d-1)!}.
\end{equation}

For the eigenvalue, by Assumption~\ref{product structure assumption} and condition (\ref{eq:eigendecay_1d}), we have
\begin{equation} 
    \mu_\alpha = \prod_{i=1}^d \mu_{\alpha_i}^{r_i,i} = \Theta\left(\prod_{i=1}^d r_i^{\alpha_i}\right).
\end{equation}
Since $r_i = \Theta(1/d)$, we have 
\begin{equation} 
    \mu_\alpha = \Theta\left(\prod_{i=1}^d d^{-\alpha_i}\right) = \Theta\left(d^{-|\alpha|}\right) = \Theta\left(d^{-k}\right).
\end{equation}
The condition (\ref{constant bound}) ensures that the constants in the $\Theta$-notation are independent of $d$ and the specific multi-index $\alpha$ (only depending on $|\alpha|$).

For $N(d,k)$, we apply Stirling's formula:
\begin{equation}
n! = \sqrt{2\pi n}\left(\frac{n}{e}\right)^n\left(1 + \Theta(n^{-1})\right), \quad n=1,2,\dots
\end{equation}
When $k$ is fixed and $d \to \infty$, we obtain
\begin{equation}
\begin{aligned} 
    N(d,k) &= \frac{1}{k!} \cdot \frac{(k+d-1)!}{(d-1)!} \\
    &= \frac{1}{k!} \cdot \sqrt{\frac{k+d-1}{d-1}} \cdot \frac{(k+d-1)^{k+d-1}e^{d-1}}{(d-1)^{d-1}e^{k+d-1}} \cdot \frac{1 + \Theta((k+d-1)^{-1})}{1 + \Theta((d-1)^{-1})} \\
    &= \Theta\left(\frac{(k+d-1)^{k+d-1}}{(d-1)^{d-1}}\right) \\
    &= \Theta\left(\left(1+\frac{k}{d-1}\right)^{d-1} \cdot (k+d-1)^k\right) \\
    &= \Theta(d^k).
\end{aligned}
\end{equation}
Thus, for any fixed $k$, there exist constants $\mathfrak{C}_1,\mathfrak{C}_2>0$ such that for sufficiently large $d$,
\begin{equation}
    \mathfrak{C}_1 d^{-k} \leq \mu_{\alpha} \leq \mathfrak{C}_2 d^{-k} \quad \text{and} \quad \mathfrak{C}_1 d^{k} \leq N(d,k) \leq \mathfrak{C}_2 d^{k}
\end{equation}
for all multi-indices $\alpha$ with $|\alpha|=k$, which completes the proof.
\end{proof}

\subsubsection{Facts of the Eigenvalues of Large Dimensional Product Kernels}\label{section eigenvlue facts}

\begin{lemma}\label{lemma calculation n1 n2}
   Let $k = k_{d}$ be a kernel satisfying Assumption \ref{assumption f*},\ref{product structure assumption}. By choosing $\lambda = d^{-l}$ for some $l > 0$, if $ p \le l \le p+1$ for some $ p \in \{0,1,2\cdots\}$, we have:
    \begin{equation}\label{inner product N1}
        \mathcal{N}_{1}(\lambda) = O\left(\lambda^{-1}\right);
    \end{equation}
    \begin{equation}\label{inner product N2}
        \mathcal{N}_{2}(\lambda) = \Theta\left(d^{p} + \lambda^{-2} d^{-(p+1)}\right).
    \end{equation}
    The notation $O,\Theta$ involves constants only depending on $p$.
\end{lemma}
\begin{proof}
    For $\mathcal{N}_1$, the following inequality holds:
\begin{equation}
\begin{aligned}
        \mathcal{N}_1
        \leq\frac{1}{\lambda}\sup k(x,x)
        \leq\frac{\kappa^2}{\lambda}
\end{aligned}
\end{equation}

For $\mathcal{N}_2$, the following inequality holds:
\begin{equation}
    \begin{aligned}
        &\mathcal{N}_2
        =\sum_{k\in N}\left(\frac{\lambda_k}{\lambda+\lambda_k}\right)^2\\
        &\leq \sum_{k\leq p}N(d,k)+\lambda^{-2}\sum_{k>p}\sum_{j=1}^{N(d,k)}\mu_{k,j}^2 \\
        &\lesssim \sum_{k\leq p}N(d,k)+\lambda^{-2}d^{-(p+1)}\sum_{k>p}\sum_{j=1}^{N(d,k)}\mu_{k,j}\\
        &\lesssim d^p+\lambda^{-2} d^{-(p+1)}
    \end{aligned}
\end{equation}

On the other hand, we have the lower bound:
\begin{equation}
    \begin{aligned}
        &\mathcal{N}_2
        =\sum_{k\in N}\left(\frac{\lambda_k}{\lambda+\lambda_k}\right)^2\\
        &\gtrsim \left(\frac{d^{-p}}{d^{-p}+\lambda}\right)^2N(d,p)+(d^{-p-1}+\lambda)^{-2}d^{-2p-2} N(d,p+1)\\
        &\geq\frac{1}{4}N(d,p)+(2\lambda)^{-2}d^{-2p-2}N(d,p+1)\\
        &\gtrsim d^p+\lambda^{-2} d^{-(p+1)}
    \end{aligned}
\end{equation}
\end{proof}
\begin{lemma}\label{lemma calculation m2}
    Let $k =  k_{d}$ be a sequence of kernels satisfying Assumption \ref{assumption f*},\ref{product structure assumption}. Define $\tilde{s}=\min\{s,2\}.$ By choosing $\lambda = d^{-l}$ for some $l > 0$, if $ p \le l < p+1$ for some $ p \in \{0,1,2\cdots\}$, we have:
    \begin{equation}\label{inner product R2}
        \mathcal{R}_{2}(\lambda) = \Theta(\lambda^2d^{(2-\tilde{s})p}+d^{-(p+1)\tilde{s}}).
    \end{equation}
    The notation $\Theta$ involves constants only depending on $p$.
\end{lemma}

\begin{proof}
When $s\le 2$, 
\begin{align*}
    \mathcal{R}_2(\lambda)
    &=\lambda^2\sum_{k=0}^\infty\sum_{i=1}^{N(d,k)}\frac{\mu_{k,i}^s}{(\mu_{k,i}+\lambda)^2}\mu_{k,i}^{-s}f_{k,i}^2\\
    &\lesssim \lambda^2\left(\sum_{k=0}^p\frac{d^{-sk}}{(d^{-k}+\lambda)^2}R_1^2+\sum_{k=p+1}^\infty\sum_{i=1}^{N(d,k)}\frac{\mu_{k,i}^s}{(\mu_{k,i}+\lambda)^2}\mu_{k,i}^{-s}f_{k,i}^2\right)\\
    &\lesssim \lambda^2\left(d^{(2-s)p}+\lambda^{-2}\sum_{k=p+1}^\infty\sum_{i=1}^{N(d,k)}d^{-(p+1)s}\mu_{k,i}^{-s}f_{k,i}^2\right)\\
    &\lesssim \lambda^2\left(d^{(2-s)p}+\lambda^{-2}d^{-(p+1)s} R_1^2\right)\\
    &\lesssim\lambda^2d^{(2-s)p}+d^{-(p+1)s}.
\end{align*}
On the other hand,
\begin{align*}
    \mathcal{R}_2(\lambda)
    &\ge\lambda^2\left(\sum_{i=1}^{N(d,p)} \frac{\mu_{p,i}^s}{(\mu_{p,i}+\lambda)^2}\mu_{p,i}^{-s}f_{p,i}^2+\sum_{i=1}^{N(d,p+1)} \frac{\mu_{p+1,i}^s}{(\mu_{p+1,i}+\lambda)^2}\mu_{p+1,i}^{-s}f_{p+1,i}^2\right)\\
    &\gtrsim\lambda^2 d^{(2-s)p}+d^{-(p+1)s}.
\end{align*}
When $s>2$, without loss of generality, we may assume that $\lambda_0\leq 1,\lambda\leq 1$. Hence, 
\begin{equation*}
    \mathcal{R}_2(\lambda)\geq\lambda^{2} \sum\limits_{i=0}^{\infty} \frac{f_{i}^{2}}{\left(\lambda_{i} + \lambda\right)^{2}} \ge \frac{1}{4} \lambda^{2} \sum\limits_{i=0}^{\infty} f_{i}^{2} \gtrsim \lambda^{2}.
\end{equation*}
On the other hand,
\begin{align*}
        \mathcal{R}_{2}(\lambda) &= \lambda^2\sum_{k=0}^\infty\sum_{i=1}^{N(d,k)}\frac{\mu_{k,i}^s}{(\mu_{k,i}+\lambda)^2}\mu_{k,i}^{-s}f_{k,i}^2 \notag \\
        &\le \lambda^{2} \left( \sup\limits_{k ,i} \frac{\mu_{k,i}^{s}}{\left( \mu_{k,i} + \lambda\right)^{2}} \cdot \sum_{k=0}^\infty\sum_{i=1}^{N(d,k)}\mu_{k,i}^{-s}f_{k,i}^2  \right) \notag \\
        &\le \lambda^{2} \cdot \sup\limits_{k ,i} \frac{\mu_{k,i}^{s}}{\left( \mu_{k,i} + \lambda\right)^{2}} \cdot R_{1}^{2} \notag \\
        &\lesssim \lambda^{2}.
    \end{align*}
    Hence, we get the desired result.
\end{proof}

\begin{lemma}\label{lemma calculation flambda}
    Let $k =  k_{d}$ be a sequence of kernels satisfying Assumption \ref{assumption f*},\ref{product structure assumption}. When $0<s<1$, by choosing $\lambda = d^{-l}$ for some $l > 0$, if $ p \le l < p+1$ for some $ p \in \{0,1,2\cdots\}$, we have:
    \begin{equation}\label{flambda}
    \left\| f_{\lambda} \right\|_{L^{\infty}} = O\left( d^{\frac{(1-s)p}{2}} + \lambda^{-1} d^{-\frac{(1+s)(p+1)}{2}}\right).
    \end{equation}
    The notation $\Theta$ involves constants only depending on $p$.
\end{lemma}

\begin{proof}
Notice that $f(x)=\langle k(x,\cdot), f\rangle_{\mathcal{H}}\leq\|k(x,\cdot)\|_{\mathcal{H}}\|f\|_{\mathcal{H}}\leq\kappa\|f\|_{\mathcal{H}}.$ We shall now bound $\|f_\lambda\|_{\mathcal{H}}$.
    \begin{align*}
    \|f_\lambda\|_{\mathcal{H}}^2
    &=\left\|\sum_{i=0}^{\infty}\frac{\lambda_i}{\lambda_i+\lambda}f_i e_i\right\|_\mathcal{H}^2 \\
    &= \sum_{i=0}^{\infty}\frac{\lambda_i^{1+s}}{(\lambda_i+\lambda)^2}\lambda_i^{-s}f_i^2 \\
    &=\sum_{k=0}^\infty\sum_{i=1}^{N(d,k)}\frac{\mu_{k,i}^{1+s}}{(\mu_{k,i}+\lambda)^2}\mu_{k,i}^{-s}f_i^2\\
    &=\sum_{k=0}^p\sum_{i=1}^{N(d,k)}\frac{\mu_{k,i}^{1+s}}{(\mu_{k,i}+\lambda)^2}\mu_{k,i}^{-s}f_i^2+\sum_{k=p+1}^\infty\sum_{i=1}^{N(d,k)}\frac{\mu_{k,i}^{1+s}}{(\mu_{k,i}+\lambda)^2}\mu_{k,i}^{-s}f_i^2\\
    &\lesssim d^{(1-s)p}R_1^2+\sum_{k=p+1}^\infty\sum_{i=1}^{N(d,k)}\frac{\mu_{k,i}^{1+s}}{(\mu_{k,i}+\lambda)^2}\mu_{k,i}^{-s}f_i^2\\
    &\lesssim d^{(1-s)p}R_1^2+\lambda^{-2}\sum_{k=p+1}^\infty\sum_{i=1}^{N(d,k)}{\mu_{k,i}^{1+s}}\mu_{k,i}^{-s}f_i^2\\
    &\lesssim p\lambda_p^{s-1}R_1^2+\lambda^{-2}d^{-(p+1)(s+1)} R_1^2\\
    &\lesssim d^{(1-s)p}+\lambda^{-2}d^{-(p+1)(s+1)}.
    \end{align*}
    Hence, we get the desired result.
\end{proof}

\subsubsection{Proof of Theorem \ref{thm:upper bound gaussian}}
Now we are able to provide the proof of Theorem \ref{thm:upper bound gaussian}.

\begin{proof}
In order to prove the theorem, we will (i) locate the balancing regularization index $l_{\mathrm{balance}}$ (so $\lambda_{\mathrm{balance}}=d^{-l_{\mathrm{balance}}}$), (ii) verify the four technical conditions in \eqref{condition} for $l=l_{\mathrm{balance}}$, and (iii) show that no other choice of $\lambda$ yields a strictly better rate (hence $\lambda_{\mathrm{balance}}$ is optimal).  Throughout we use Lemma \ref{lemma calculation n1 n2} and Lemma \ref{lemma calculation m2} for asymptotic orders of $\mathcal{N}_1,\mathcal{N}_2,\mathcal{R}_2$.

\medskip

\noindent\textbf{Step 1. Determine $l_{\mathrm{balance}}$.}  
Assume $s\ge1$, let $\tilde s=\min\{s,2\}$ and write $\lambda=d^{-l}$ with $0<l<\gamma$.  Using Lemma \ref{lemma calculation n1 n2} and Lemma \ref{lemma calculation m2} we separate three regimes for $l$ (indexed by the integer $p\ge0$):

\begin{itemize}
    \item If $l\in\big(p,\;p+\tfrac12\big]$ then
    \[
       \frac{\mathcal{N}_2(\lambda)}{n}\asymp d^{\,p-\gamma},\qquad
       \mathcal{R}_2(\lambda)\asymp d^{-2l+(2-\tilde s)p}.
    \]
    Balancing variance and bias,
    \[
       d^{\,p-\gamma}\asymp d^{-2l+(2-\tilde s)p}
    \quad\Longrightarrow\quad
       l_{\mathrm{balance}}=\frac{\gamma+p-p\tilde s}{2},
    \]
    and requiring $l_{\mathrm{balance}}\in\big(p,p+\tfrac12\big]$ yields
    \[
       \gamma\in\big(p+p\tilde s,\;p+p\tilde s+1\big].
    \]

    \item If $l\in\big(p+\tfrac12,\;p+\tfrac{\tilde s}{2}\big]$ then
    \[
       \frac{\mathcal{N}_2(\lambda)}{n}\asymp d^{\,2l-p-1-\gamma},\qquad
       \mathcal{R}_2(\lambda)\asymp d^{-2l+(2-\tilde s)p},
    \]
    and balancing gives
    \[
       d^{\,2l-p-1-\gamma}\asymp d^{-2l+(2-\tilde s)p}
    \quad\Longrightarrow\quad
       l_{\mathrm{balance}}=\frac{\gamma+3p-p\tilde s+1}{4},
    \]
    with admissible $\gamma$ satisfying
    \[
       \gamma\in\big(p+p\tilde s+1,\;p+p\tilde s+2\tilde s-1\big].
    \]

    \item If $l\in\big(p+\tfrac{\tilde s}{2},\;p+1\big]$ then
    \[
       \frac{\mathcal{N}_2(\lambda)}{n}\asymp d^{\,2l-p-1-\gamma},\qquad
       \mathcal{R}_2(\lambda)\asymp d^{-(p+1)\tilde s},
    \]
    and balancing yields
    \[
       d^{\,2l-p-1-\gamma}\asymp d^{-(p+1)\tilde s}
    \quad\Longrightarrow\quad
       l_{\mathrm{balance}}=\frac{\gamma+(p+1)(1-\tilde s)}{2},
    \]
    with admissible $\gamma$ satisfying
    \[
       \gamma\in\big(p+p\tilde s+2\tilde s-1,\;(p+1)+(p+1)\tilde s\big].
    \]
\end{itemize}

From now on we fix $l_{\mathrm{balance}}$ as above according to which $\gamma$-regime we are in.

\medskip

\noindent\textbf{Step 2. Verification of the hypotheses of Theorem \ref{main theorem} (conditions \eqref{condition}).}  
To apply Theorem \ref{main theorem} we must check the four conditions
\begin{equation}\label{eq:thm34-conds}
   \frac{\mathcal{N}_1(\lambda)}{n}\ln n = o(1),\qquad
   \frac{1}{n\lambda}=o(1),\qquad
   \mathcal{N}_2(\lambda)=\Omega(1),\qquad
   \frac{\ln n}{n\lambda^2}=o\big(\mathcal{N}_2(\lambda)\big)
\end{equation}
for $\lambda=\lambda_{\mathrm{balance}}$.  We verify these case by case. 

\medskip

\noindent\emph{Case A:} $\gamma\in\big(p+p\tilde s,\;p+p\tilde s+1\big]$ (so $l_{\mathrm{balance}}=(\gamma+p-p\tilde s)/2\in(p,p+\tfrac12]$).\\
(1) $\frac{1}{n\lambda}=d^{l-\gamma}=o(1)$ is equivalent to $l<\gamma$. Substituting $l_{\mathrm{balance}}$ we get
\[
\frac{\gamma+p-p\tilde s}{2}<\gamma \iff \gamma>p-p\tilde s,
\]
which holds because $\gamma>p+p\tilde s\ge p-p\tilde s$.  Thus the second condition in \eqref{eq:thm34-conds} holds.

(2) For $\frac{\mathcal{N}_1(\lambda)}{n}\ln n=o(1)$, by Lemma \ref{lemma calculation n1 n2} (case $l\in(p,p+\tfrac12]$) we have
\[
\frac{\mathcal{N}_1(\lambda)}{n}\ln n\lesssim \frac{1}{n\lambda}\ln n=d^{l-\gamma}\ln n=o(1).
\]
Thus the first condition in \eqref{eq:thm34-conds} holds.

(3) For  $\mathcal{N}_2(\lambda)=\Omega(1)$, from Lemma \ref{lemma calculation n1 n2} we have $\mathcal{N}_2(\lambda)\gtrsim d^p$, hence $\mathcal{N}_2(\lambda)=\Omega(1)$ obviously holds.

(4) For $(\ln n)/(n\lambda^2)=o(\mathcal{N}_2(\lambda))$,  using $n\asymp d^\gamma$ and $\lambda=d^{-l_{\mathrm{balance}}}$, this condition is equivalent to
\[
d^{-\gamma}\cdot d^{2l_{\mathrm{balance}}}\cdot \gamma\ln d \ll d^p.
\]
Substituting $l_{\mathrm{balance}}=(\gamma+p-p\tilde s)/2$ and we find that the condition can be verified when $p>0$. When $p=0$, a simple calculation also implies that the condition holds.

\medskip

\noindent\emph{Case B:} $\gamma\in\big(p+p\tilde s+1,\;p+p\tilde s+2\tilde s-1\big]$ (so $l_{\mathrm{balance}}=(\gamma+3p-p\tilde s+1)/4\in(p+\tfrac12,p+\tfrac{\tilde s}{2}]$).\\

(1) $\frac{1}{n\lambda}=d^{\,l-\gamma}=o(1)$ reduces to
\[
\frac{\gamma+3p-p\tilde s+1}{4}<\gamma \iff \gamma> p-\frac{p\tilde s}{3}+\frac13,
\]
which holds on the stated $\gamma$-interval.

(2) By Lemma \ref{lemma calculation n1 n2} in this regime we have
\[
\frac{\mathcal{N}_1(\lambda)}{n}=O\!\big(d^{\frac{\gamma+3p-p\tilde s+1}{2}-\gamma}\big),
\]
hence
\[
\frac{\mathcal{N}_1(\lambda)}{n}\ln n
=O\!\big(d^{\frac{\gamma+3p-p\tilde s+1}{2}-\gamma}\cdot\gamma\ln d\big)=o(1).
\]

(3) For  $\mathcal{N}_2(\lambda)=\Omega(1)$, from Lemma \ref{lemma calculation n1 n2} we have $\mathcal{N}_2(\lambda)\gtrsim d^p$, hence $\mathcal{N}_2(\lambda)=\Omega(1)$ obviously holds.

(4) For $(\ln n)/(n\lambda^2)=o(\mathcal{N}_2(\lambda))$, we have  
\[
\frac{\ln n}{n\lambda^2}\asymp d^{2l-\gamma}\ln n =o(d^{2l-p-1})=o(\mathcal{N}_2(\lambda)).
\]

\medskip

\noindent\emph{Case C:} $\gamma\in\big(p+p\tilde s+2\tilde s-1,\;(p+1)+(p+1)\tilde s\big]$ (so $l_{\mathrm{balance}}=(\gamma+(p+1)(1-\tilde s))/2\in(p+\tfrac{\tilde s}{2},p+1]$).

(1) $1/(n\lambda)=o(1)$ is equivalent to
\[
\frac{\gamma+(p+1)(1-\tilde s)}{2}<\gamma \iff \gamma> p-p\tilde s+1-\tilde s,
\]
which is true for the stated $\gamma$-range.

(2) By Lemma \ref{lemma calculation n1 n2} in this regime we have
\[
\frac{\mathcal{N}_1(\lambda)}{n}=O\!\big(d^{\frac{\gamma+(p+1)(1-\tilde s)}{2}-\gamma}\big),
\]
hence
\[
\frac{\mathcal{N}_1(\lambda)}{n}\ln n
=O\!\big(d^{\frac{\gamma+(p+1)(1-\tilde s)}{2}-\gamma}\cdot\gamma\ln d\big)=o(1).
\]

(3) For  $\mathcal{N}_2(\lambda)=\Omega(1)$, from Lemma \ref{lemma calculation n1 n2} we have $\mathcal{N}_2(\lambda)\gtrsim d^p$, hence $\mathcal{N}_2(\lambda)=\Omega(1)$ obviously holds.

(4) For $(\ln n)/(n\lambda^2)=o(\mathcal{N}_2(\lambda))$, we have  
\[
\frac{\ln n}{n\lambda^2}\asymp d^{2l-\gamma}\ln n =o(d^{2l-p-1})=o(\mathcal{N}_2(\lambda)).
\]
\medskip

Thus in each admissible $\gamma$-regime the four conditions in \eqref{eq:thm34-conds} are satisfied for $\lambda=\lambda_{\mathrm{balance}}$.  Moreover, by simple calculation, the conditions of Theorem \ref{main theorem} also hold for all $l\le l_{\mathrm{balance}}$.

\medskip

\noindent\textbf{Step 3. Conclude the optimal rate and show $\lambda_{\mathrm{balance}}$ is best.}  
By Theorem \ref{main theorem} (applied with the verified conditions \eqref{eq:thm34-conds}) we have for $\lambda=\lambda_{\mathrm{balance}}$
\begin{equation}\label{eq:rate-balance}
    \mathbb{E}\left[\|\hat f_{\lambda_{\mathrm{balance}}}-f_\rho^*\|_{L^2}^2\mid\mathbf X\right]
    = \Theta_{\mathbb{P}}\!\Big(\frac{\sigma^2\mathcal{N}_2(\lambda_{\mathrm{balance}})}{n}+\mathcal{R}_2(\lambda_{\mathrm{balance}})\Big),
\end{equation}
and the computations in Step 1 (together with Lemma \ref{lemma calculation m2}) give the three displayed rates corresponding to the three $\gamma$-regimes.

It remains to show no other $\lambda$ yields a strictly smaller rate.  For $\lambda\gtrsim\lambda_{\mathrm{balance}}$ the bias term $\mathcal{R}_2(\lambda)$ cannot decrease below $\mathcal{R}_2(\lambda_{\mathrm{balance}})$ , while $\mathcal{R}_2(\lambda_{\mathrm{balance}})=\Theta_{\mathbb{P}}\!\Big(\frac{\sigma^2\mathcal{N}_2(\lambda_{\mathrm{balance}})}{n}+\mathcal{R}_2(\lambda_{\mathrm{balance}})\Big)$ by the definition of $\lambda_{\mathrm{balance}}$. Hence,  the RHS of \eqref{eq:rate-balance} gives a lower bound (up to constants) for all such $\lambda$.  For $\lambda\lesssim\lambda_{\mathrm{balance}}$, note the matrix inequality
\[
(\mathbf K+\lambda_1)^{-2}\succeq(\mathbf K+\lambda_2)^{-2},\qquad\text{whenever }\lambda_1\le\lambda_2,
\]
(which implies $\mathbf{Var}(\lambda_1)\ge\mathbf{Var}(\lambda_2)$ by the proof in Lemma \ref{lemma var transform}).  Hence for $\lambda\le\lambda_{\mathrm{balance}}$, we have
\[
\mathbb{E}\left[\|\hat f_{\lambda}-f_\rho^*\|_{L^2}^2\mid\mathbf X\right]
\ge \mathbf{Var}(\lambda)\ge \mathbf{Var}(\lambda_{\mathrm{balance}})
\asymp \mathbb{E}\left[\|\hat f_{\lambda_{\mathrm{balance}}}-f_\rho^*\|_{L^2}^2\mid\mathbf X\right].
\]
Combining the two sides shows that $\lambda_{\mathrm{balance}}$ attains (up to constants) the minimal possible conditional mean-squared error among all admissible regularization parameters, that is, it is the optimal choice and the rates given in the theorem are sharp.

This completes the proof.
\end{proof}

\subsubsection{Proof of Theorem \ref{theorem s<1 gaussian}}
\begin{proof}
We follow the steps above: (i) locate the balancing index \(l_{\mathrm{balance}}\) (so \(\lambda_{\mathrm{balance}}=d^{-l_{\mathrm{balance}}}\)); (ii) verify all hypotheses of Theorem \ref{theorem s<1 gaussian} at \(l=l_{\mathrm{balance}}\) (including the extra approximation condition involving the \(L^\infty\)-norm of \(f_\lambda\)).

We use Lemma \ref{lemma calculation n1 n2}  and Lemma \ref{lemma calculation m2} for the asymptotic orders of \(\mathcal N_1,\mathcal N_2,\mathcal R_2\). The \(L^\infty\)-bound for \(f_\lambda\) is provided by Lemma \ref{lemma calculation flambda}.

\medskip

\noindent\textbf{Step 1. Determination of \(l_{\mathrm{balance}}\).}  
Write \(\lambda=d^{-l}\) with \(0<l<\gamma\). As in the main text, separate three regimes indexed by \(p\in\{0,1,2,\dots\}\):

\begin{itemize}
  \item If \(l\in\big(p,\;p+\frac{s}{2}\big]\),  then (Lemma \ref{lemma calculation n1 n2}, Lemma \ref{lemma calculation m2})
  \[
    \frac{\mathcal N_2(\lambda)}{n}\asymp d^{p-\gamma},\qquad
    \mathcal R_2(\lambda)\asymp d^{-2l+(2-s)p}.
  \]
  Balancing variance and bias gives
  \[
    d^{p-\gamma}\asymp d^{-2l+(2-s)p}\quad\Longrightarrow\quad
    l_{\mathrm{balance}}=\frac{\gamma+p-p s}{2},
  \]
  and the admissible \(\gamma\)-range is \(\gamma\in\big(p+p s,\; p+p s+s\big]\).

 \item If \(l\in\big(p+\frac{s}{2},\;p+\frac{1}{2}\big]\), then
 \[
    \frac{\mathcal N_2(\lambda)}{n}\asymp d^{p-\gamma},\qquad
    \mathcal R_2(\lambda)\asymp d^{-(p+1)s},
  \]
which is balanced if and only if $\gamma=p+ps+s$.

  \item If \(l\in\big(p+\tfrac{1}{2},\;p+1\big]\), then
  \[
    \frac{\mathcal N_2(\lambda)}{n}\asymp d^{2l-p-1-\gamma},\qquad
    \mathcal R_2(\lambda)\asymp d^{-(p+1)s},
  \]
  hence
  \[
    l_{\mathrm{balance}}'=\frac{\gamma+(p+1)(1-s)}{2},\qquad
    \gamma\in\big(p+p s+s,\; (p+1)+(p+1)s\big].
  \]
\end{itemize}

Notice that the second circumstance is covered by the first one, hence there are only two intervals of $\gamma$. When $\gamma\in\big(p+p s+s,\; (p+1)+(p+1)s\big]$, by choosing $l=p+\frac{s}{2}$, the convergence rate remains the same. Hence, we substitute $l_{\mathrm{balance}}'=\frac{\gamma+(p+1)(1-s)}{2}$ with $l_{\mathrm{balance}}=p+\frac{s}{2}$. From now on \(l_{\mathrm{balance}}\) denotes the formula appropriate to the \(\gamma\)-regime above.

\medskip

\noindent\textbf{Step 2. Verification of the hypotheses of Theorem \ref{main theorem s<1}.}  
Theorem \ref{main theorem s<1} requires, for the chosen \(\lambda\), the four standard conditions \eqref{condition s<1} 
\begin{equation}\label{eq:thm36-conds-repeat}
\frac{\mathcal N_1(\lambda)}{n}\ln n = o(1),\qquad
\frac{1}{n\lambda}=o(1),\qquad
\mathcal N_2(\lambda)=\Omega(1),\qquad
\frac{\ln n}{n\lambda^2}=o\big(\mathcal N_2(\lambda)\big),
\end{equation}
together with the extra approximation condition \eqref{add condition s<1}: there exists \(\varepsilon>0\) such that
\begin{equation}\label{eq:extra-condition}
\frac{\sqrt{\frac{1}{\lambda}}(n^{\frac{1-s}{2}+\epsilon}+\|f_{\lambda}\|_{L^{\infty}})}{n}=o(\frac{1}{\sqrt{n}}\mathcal{N}_2(\lambda)^{\frac{1}{2}}+\mathcal{R}_2(\lambda)^{\frac{1}{2}})
\end{equation}
We check \eqref{eq:thm36-conds-repeat} and \eqref{eq:extra-condition} at \(\lambda=\lambda_{\mathrm{balance}}\). (Because each of the left-hand quantities in \eqref{eq:thm36-conds-repeat} and the left-hand side of \eqref{eq:extra-condition} is non-decreasing in \(l\) when \(\lambda=d^{-l}\), verification at \(l_{\mathrm{balance}}\) implies the conditions hold for all \(l\le l_{\mathrm{balance}}\).)

The verification of the four relations in \eqref{eq:thm36-conds-repeat} is identical to the checks carried out in the proof of Theorem \ref{thm:upper bound gaussian} (replace \(\tilde s\) by \(s\) here). It remains to verify the extra condition \eqref{eq:extra-condition}.  By Lemma C.5 we have the explicit bound
\begin{equation}\label{eq:lemC5}
\|f_\lambda\|_{L^\infty} \;=\; O\!\Big( d^{\frac{(1-s)p}{2}} \;+\; \lambda^{-1} d^{-\frac{(1+s)(p+1)}{2}} \Big),
\qquad\text{whenever } p\le l < p+1,
\end{equation}

Set \(n\asymp d^\gamma\), \(\lambda=d^{-l_{\mathrm{balance}}}\), using \eqref{eq:lemC5} and noticing that $\gamma\geq p$, we have
\begin{equation}
    LHS\lesssim d^{l_{\mathrm{balance}}/2-\gamma}(d^{(\frac{1-s}{2}+\epsilon)\gamma}+d^{l_{\mathrm{balance}}-\frac{(1+s)(p+1)}{2}})
\end{equation}

On the right-hand side of \eqref{eq:extra-condition} we have
\[
\frac{1}{\sqrt{n}}\mathcal N_2(\lambda)^{1/2}\asymp d^{(p-\gamma)/2}+d^{(2l_{\mathrm{balance}}-p-1-\gamma)/2},\qquad
\mathcal R_2(\lambda)^{1/2}\asymp d^{-l_{\mathrm{balance}} + \frac{(2-s)p}{2}},
\]
so
\[
\mathrm{RHS}\asymp \max\{ d^{(p-\gamma)/2}, d^{(2l_{\mathrm{balance}}-p-1-\gamma)/2},d^{-l_{\mathrm{balance}} + \frac{(2-s)p}{2}}\}.
\]

\medskip

\noindent\emph{Case A: \(l_{\mathrm{balance}}=\dfrac{\gamma+p-ps}{2}\) (so \(l\in(p,p+\tfrac12]\), \(\gamma\in(p+ps,p+ps+s]\)).}
By simple calculation, we may find that condition \eqref{add condition s<1} is equivalent to 
\[
(1-2s)\gamma<p(s+1), 
\]
which naturally holds when $s>\frac{1}{2}$. When $s\le\frac{1}{2}$, $(1-2s)\gamma<p(s+1)$ only holds for $p>0$. Hence, we are not able to give the convergence rate when $\gamma<\frac{1}{2}$.
\medskip

\noindent\emph{Case B: \(l_{\mathrm{balance}}=p+\frac{s}{2}\) (so  \(\gamma\in(p+ps+s,(p+1)+(p+1)s]\)).}

By simple calculation, we may find that condition \eqref{add condition s<1} is equivalent to 
\[
\gamma > \frac{2p+3s+2ps}{2(s+1)}.
\]
When $s> \frac{1}{2}$, we may find that the condition holds naturally for $p>0$. When $p=0$, the extra condition \eqref{eq:extra-condition} can also be verified.

When $s\le\frac{1}{2}$, we may find that the condition only holds for $p>0$. Hence, we are not able to give the convergence rate when $\gamma\le\frac{3s}{2(s+1)}$.
\medskip

In summary, by the explicit bound of Lemma C.5 for \(\|f_\lambda\|_{L^\infty}\) (used in \eqref{eq:lemC5}) and the exponent comparisons above, the extra approximation condition \eqref{eq:extra-condition} required by Theorem \ref{theorem s<1 gaussian} is satisfied at \(\lambda=\lambda_{\mathrm{balance}}\) in every admissible \(\gamma\)-regime. Together with the four conditions \eqref{eq:thm36-conds-repeat} this verifies all hypotheses of Theorem \ref{theorem s<1 gaussian} at \(l=l_{\mathrm{balance}}\).

This completes the proof.
\end{proof}

\subsection{Proof of Theorem \ref{theorem lower bound}}
\subsubsection{More preliminaries about minimax lower bound}
We shall first introduce several concepts about minimax lower bound which can be frequently found in related literature \cite{Yang_Density_1999,lu2023optimal}, etc..

Suppose that $(\mathbf{Z},d)$ is a topological space with a compatible loss function $d$, which are mappings from $\mathbf{Z} \times \mathbf{Z}$ to  $\mathbb{R}_{\geq 0}$ with $d(f, f)=0$ and $d(f, f^{\prime}) >0$ for $f \neq f^{\prime}$. We call such a loss function a \textit{distance}. We introduce the packing entropy and covering entropy below:

\begin{definition}[Packing entropy]
A finite set $N_{\epsilon} \subset\mathbf{Z}$ is said to be an $\epsilon$-packing set in $\mathbf{Z}$ with separation $\epsilon>0$, if for any $f, f^{\prime} \in N_{\epsilon}, f \neq f^{\prime}$, we have $d\left(f, f^{\prime}\right)>\epsilon$. The logarithm of the maximum cardinality of $\epsilon$-packing set is called the $\epsilon$-packing entropy of $\mathbf{Z}$ with distance $d$ and is denoted by 
$M_{d}(\epsilon,\mathbf{Z})$.
\end{definition}

\begin{definition}[Covering entropy]\label{def:covering_entropy}
A set $G_{\epsilon} \subset\mathbf{Z}$ is said to be an $\epsilon$-net for $\mathbf{Z}$ if for any $\tilde{f} \in\mathbf{Z}$, there exists an $f_0 \in G_{\epsilon}$ such that $d(\tilde{f}, f_0) \leq \epsilon$. The logarithm of the minimum cardinality of $\epsilon$-net is called the $\epsilon$-covering entropy of $\mathbf{Z}$ with distance $d$ and is denoted by
$V_{d}(\epsilon,\mathbf{Z})$.
\end{definition}

Let $\mathcal{B} = \left\{ f \in \mathcal{H},~ \| f \|_{\mathcal{H}} \le R\right\}$, where $ R$ is the constant from Assumption \ref{assumption f*}.  Without loss of generality, we can consider $\mathcal{B}$ be the unit ball in $\mathcal{H}$. Let $M_{2}(\epsilon,\mathcal{B})$ be the $\epsilon$-packing entropy of $(\mathcal{B}, d^2=\|\cdot\|_{L^2}^2)$ and $V_{2}(\epsilon,\mathcal{B})$ be the $\epsilon$-covering entropy of $(\mathcal{B}, d^2=\|\cdot\|_{L^2}^2)$. Recalling that $\mu$ is the marginal distribution on $\mathcal{X},$ we further define
\begin{align*}
    \mathcal{D}=\left\{ \rho_{f}~\bigg|~ \mbox{ joint distribution of $(y,{x}$) where } {x}\sim \mu, y=f( {x})+\epsilon, \epsilon\sim N(0,\sigma_\epsilon^{2}),
    f\in \mathcal{B} \right\},
\end{align*}
and let $V_{K}(\epsilon,\mathcal{D})$ be the $\epsilon$-covering entropy of $(\mathcal{D}, d^2=\text{ KL divergence })$. It is easy to see that $\mathcal{D} $ is an subset of $\mathcal{P}$ which is defined in Theorem \ref{theorem lower bound}, i.e., $\mathcal{D} \subset \mathcal{P} $.

The following lemmas give useful characterizations of $M_{2}(\epsilon,\mathcal{B}), V_{2}(\epsilon,\mathcal{B}) $ and $ V_{K}(\epsilon,\mathcal{D}) $. We refer to Lemma A.5, Lemma A.7 and Lemma A.8 in \cite{lu2023optimal} for their proofs.

\begin{lemma}\label{lemma_M_2_and_V_2} 
For any $\epsilon>0$, we have $M_{2}(2\epsilon,\mathcal{B}) \leq V_{2}(\epsilon,\mathcal{B}) \leq M_{2}(\epsilon,\mathcal{B}).$

\end{lemma}

\begin{lemma}\label{claim:d_K_and_d_2}
$V_{2}\left(\epsilon, \mathcal{B} \right) = V_{K}\left(\frac{\epsilon}{\sqrt{2}\sigma_\epsilon}, \mathcal{D}\right)$.    
\end{lemma}

\begin{lemma}
\label{lemma_entropy_of_RKHS}
Let $\{\lambda_{j}\}_{j=1}^{\infty} $ be the eigenvalues of $\mathcal{H}$. For any $\epsilon>0$, let $K(\epsilon)=\frac{1}{2}\sum\limits_{j: \lambda_j > \epsilon^2} \ln\left(\lambda_j/{\epsilon^2}\right)$. We have
\begin{equation}\label{eqn:137}
	\begin{aligned}
 V_{2}(6\epsilon, \mathcal{B})\leq K(\epsilon) \leq  V_{2}(\epsilon, \mathcal{B}).
	\end{aligned}
\end{equation}
\end{lemma}

The following important lemma is a modification of Theorem 1 and Corollary 1 in \cite{Yang_Density_1999}. We refer to Lemma 4.1 in \cite{lu2023optimal} for the proof.
\begin{lemma}\label{thm_lower_ultimate_tech}
Let $\mathfrak{c}\in (0,1)$ be  a constant only depending on $c_{1}$, $c_{2}$, and $\gamma$, where $c_{1}, c_{2} $ are the constants given in Theorem \ref{theorem lower bound}. For any $0<\tilde\epsilon_1, \tilde\epsilon_2<\infty$ 
only depending on $n$, $d$, $\{\lambda_j\}$, $c_{1}$, $c_{2}$, and $\gamma$
and satisfying
\begin{equation}
    \frac{V_K(\tilde\epsilon_2, \mathcal{D}) + n\tilde\epsilon_2^2 + \ln{2}}{V_2(\tilde\epsilon_1, \mathcal{B})} \leq \mathfrak{c},
\end{equation}
 we have 
\begin{equation}
\min _{\hat{f}} \max _{\rho_{f^{*}} \in \mathcal{D}} \mathbb{E}_{(\mathcal{X}, \mathbf{y}) \sim \rho_{f^{*}}^{\otimes n}}
\left\|\hat{f} - f^{*}\right\|_{L^2}^2
\geq \frac{1 - \mathfrak{c}}{4} \tilde\epsilon_1^2.
\end{equation}
\end{lemma}

\subsubsection{Proof of Theorem \ref{theorem lower bound}}
Now we are ready to use the lemmas in the last subsection to prove Theorem \ref{theorem lower bound}. The proof is divided into two parts, dealing with the two cases of the interval in which $\gamma$ falls into. Notice that the proof is adapted from \cite{zhang2024optimal}\\
\textit{Proof of Theorem \ref{theorem lower bound} $\left( \text{\lowercase\expandafter{\romannumeral1}}\right)$. } In this case, we assume $ \gamma \in \left( p+ps, p+ps+s \right]$ for some integer $ p \ge 0$. Let $\tilde\epsilon_{2}^{2} = C_{2} d^{-(\gamma-p)}$, in which we will choose the constant $C_{2}$ later. Note that $ \gamma - p \in (ps, (p+1)s ]$. Proposition \ref{prop:staircase} implies that there exists $C_{2}$ only depending on $p$ such that for any $j$ and sufficiently large $d$, we have
\begin{equation}
    \mu_{p+1,j}^{s} < \tilde \epsilon_{2}^{2} < \mu_{p,j}^{s}.
\end{equation}
Next we choose $\tilde\epsilon_{1}^{2} = d^{-(\gamma-p+\epsilon)},\epsilon>0 $. Since $ \gamma-p+\epsilon > ps$, for any $j$, when $d \ge \mathfrak{C}$, we have
\begin{equation}
    \tilde \epsilon_{1}^{2} < \mu_{p,j}^{s}.
\end{equation}
Hence, using Lemma \ref{lemma_entropy_of_RKHS} and Proposition \ref{prop:staircase}, for any $d \ge \mathfrak{C}$, we have
\begin{align}\label{lower case 1 eq 1}
    V_{2}\left( \tilde \epsilon_{1},\mathcal{B} \right) &\ge K(\tilde \epsilon_{1}) \\
    &\ge \frac{1}{2}\sum_{j=1}^{N(d,p)}\ln\frac{\mu_{p,j}^s}{\tilde\varepsilon_1^2} \notag \\
    &\ge \frac{1}{2} N(d,p) \ln{\left(\frac{\mathfrak{C}_{1} d^{-ps}}{d^{-(\gamma-p+\epsilon )}}\right)} \notag \\
    &= \frac{1}{2} N(d,p)  \left( \ln{\mathfrak{C}_{1}} + (\gamma-ps-p+\epsilon) \ln{d}     \right).
\end{align}
In addition, using Proposition \ref{prop:staircase}, we have the following claim.
\begin{claim}\label{claim_1}
Suppose that $ \gamma \in \left( p+ps, p+ps+s \right]$ for some integer $ p \ge 0$. Let $\tilde\epsilon_2^2$ be defined as above. For any $\epsilon_{0} > 0$, there exists a sufficiently large constant $\mathfrak{C}$ 
such that for any $d \geq \mathfrak{C}$, we have
\begin{equation*}
\begin{aligned}
&K\left( \sqrt{2}\sigma \tilde\epsilon_2 / 6 \right) \leq
 (1+\varepsilon_0)\frac{1}{2}\sum_{j=1}^{N(d,p)}\ln\frac{18\mu_{p,j}^s}{\sigma^2\tilde\varepsilon_2^2}.
\end{aligned}
\end{equation*}
\end{claim}
Therefore, for any $d \geq \mathfrak{C}$, we have
\begin{align}\label{lower case 1 eq 2}
    V_{K}\left(\tilde \epsilon_{2}, \mathcal{D} \right) &= V_{2}\left( \sqrt{2} \sigma \tilde \epsilon_{2}, \mathcal{B} \right) \\
    &\le K\left(\frac{\sqrt{2} \sigma \tilde \epsilon_{2}}{6}\right) \notag \\
    &\le (1+\varepsilon_0)\frac{1}{2}\sum_{j=1}^{N(d,p)}\ln\frac{18\mu_{p,j}^s}{\sigma^2\tilde\varepsilon_2^2} \notag \\
    &\le \left(1+\epsilon_{0} \right) \frac{1}{2} N(d,p)\ln\left(\frac{18 \mathfrak{C}_{2} d^{-ps}}{C_{2} \sigma^2  d^{-(\gamma-p)} }\right) \notag \\
    &= \left(1+\epsilon_{0} \right) \frac{1}{2} N(d,p) \left( \ln{\frac{18 \mathfrak{C}_{2}}{C_{2} \sigma^2 }} + (\gamma-ps-p) \ln{d} \right),
\end{align}
where we use Lemma \ref{claim:d_K_and_d_2} and Lemma \ref{lemma_entropy_of_RKHS} for the first line and use Proposition \ref{prop:staircase} for the third line.

Using \eqref{lower case 1 eq 1} and \eqref{lower case 1 eq 2}, recalling that $ c_{1} d^{\gamma} \le n \le c_{2} d^{\gamma}$, we have
\begin{equation}\label{lower case 1 eq 3}
    \frac{V_K(\tilde\epsilon_2, \mathcal{D}) + n\tilde\epsilon_2^2 + \ln{2}}{V_2(\tilde\epsilon_1, \mathcal{B})} \le \frac{ \left(1+\epsilon_{0} \right) \frac{1}{2} N(d,p) \left( \ln{\frac{18 \mathfrak{C}_{2}}{C_{2} \sigma^2 }} + (\gamma-ps-p) \ln{d} \right) + c_{2} d^{\gamma} \cdot C_{2} d^{-(\gamma-p)} + \ln{2} }{ \frac{1}{2} N(d,p)  \left( \ln{\mathfrak{C}_{1}} + (\gamma-ps-p+\epsilon) \ln{d} \right) }.
\end{equation}
The dominant terms in \eqref{lower case 1 eq 3} are:
\begin{equation}
    \frac{ \frac{1}{2} \left(1+\epsilon_{0} \right)(\gamma-ps-p) N(d,p)  \ln{d}  }{ \frac{1}{2} (\gamma-p+\epsilon-p s) N(d,p)  \ln{d} }.
\end{equation}
Hence, for any $ \epsilon > 0$, we can choose $ \epsilon_{0}$ small enough such that 
\begin{equation}
    \frac{V_K(\tilde\epsilon_2, \mathcal{D}) + n\tilde\epsilon_2^2 + \ln{2}}{V_2(\tilde\epsilon_1, \mathcal{B})} \le \ref{lower case 1 eq 3} := \mathfrak{c}< 1.
\end{equation}
Then using Lemma \ref{thm_lower_ultimate_tech}, we have 
\begin{displaymath}
    \min _{\hat{f}} \max _{\rho_{f^{*}} \in \mathcal{D}} \mathbb{E}_{(x, \mathbf{y}) \sim \rho_{f^{*}}^{\otimes n}}
\left\|\hat{f} - f^{*}\right\|_{L^2}^2
\geq \frac{1 - \mathfrak{c}}{4} \tilde\epsilon_1^2 = \frac{1 - \mathfrak{c}}{4} d^{-(\gamma-p-\epsilon)}.
\end{displaymath}
Further recalling that $ \mathcal{D} \subset \mathcal{P}$, we have 
\begin{equation}
    \min _{\hat{f}} \max _{\rho \in \mathcal{P}} \mathbb{E}_{(x, \mathbf{y}) \sim \rho^{\otimes n}} \left\|\hat{f} - f_{\rho}^{*}\right\|_{L^2}^2 \ge \min _{\hat{f}} \max _{\rho_{f^{*}} \in \mathcal{D}} \mathbb{E}_{(x, \mathbf{y}) \sim \rho_{f^{*}}^{\otimes n}} \left\|\hat{f} - f^{*}\right\|_{L^2}^2
\geq \frac{1 - \mathfrak{c}}{4} d^{-(\gamma-p-\epsilon)}.
\end{equation}
We finish the proof of Theorem \ref{theorem lower bound} $\left( \text{\lowercase\expandafter{\romannumeral1}}\right)$.\\

\textit{Proof of Theorem \ref{theorem lower bound} $\left( \text{\lowercase\expandafter{\romannumeral2}}\right)$. } We assume $ \gamma \in \left( p+ps+s,(p+1)+(p+1)s \right]$ for some integer $ p \ge 0$. Let $\tilde\epsilon_{2}^{2} = C_{2} d^{-(p+1)s} \ln{d}$, where we shall choose the constant $C_{2}$ later. Proposition \ref{prop:staircase} implies that there exists a constant $\mathfrak{C} $ such that for any $d \geq \mathfrak{C}$ and any $j$, we have
\begin{equation}
    \mu_{p+1,j}^{s} < \tilde \epsilon_{2}^{2} < \mu_{p,j}^{s}.
\end{equation}
Let $\tilde\epsilon_{1}^{2} = C_{1} d^{-(p+1)s} $. Using Proposition \ref{prop:staircase}, we can choose $C_{1} < \mathfrak{C}_{1}^{s}$, where $\mathfrak{C}_{1} $ is the constant in Proposition \ref{prop:staircase}, such that for any $d \ge \mathfrak{C}$ and any $j$, where $ \mathfrak{C} $ is a constant only depending on $s$ and $p$, we have
\begin{equation}
    \tilde \epsilon_{1}^{2} < \mu_{p+1,j}^{s}.
\end{equation}
Therefore, by Lemma \ref{lemma_entropy_of_RKHS} and Proposition \ref{prop:staircase}, for any $d \ge \mathfrak{C}$, we have
\begin{align}\label{lower case 2 eq 1}
    V_{2}\left( \tilde \epsilon_{1},\mathcal{B} \right) &\ge K(\tilde \epsilon_{1}) \\
    &\ge \frac{1}{2}\sum_{j=1}^{N(d,p)}\ln\frac{\mu_{p,j}^s}{\tilde\varepsilon_1^2} \notag \\
    &\ge \frac{1}{2} N(d,p+1) \ln{\left(\frac{\mathfrak{C}_{1} d^{-(p+1)s}}{C_{1} d^{-(p+1)s}}\right)} \notag \\
    &= \frac{1}{2} N(d,p+1) \ln{\frac{\mathfrak{C}_{1}}{C_{1}}}.
\end{align}
In addition, using Proposition \ref{prop:staircase}, we have the following claim.
\begin{claim}\label{claim_2}
Suppose that $ \gamma \in \left( p+ps+s, (p+1)+(p+1)s \right]$ for some integer $ p \ge 0$. Let $\tilde\epsilon_2^2$ be defined as above. For any $\epsilon_{0} > 0$, there exists a sufficiently large constant $\mathfrak{C}$ only depending on $s, p$ and $\epsilon_{0}$, such that for any $d \geq \mathfrak{C}$, we have
\begin{equation*}
\begin{aligned}
&K\left( \sqrt{2}\sigma \tilde\epsilon_2 / 6 \right) \leq
 (1+\varepsilon_0)\frac{1}{2}\sum_{j=1}^{N(d,p)}\ln\frac{18\mu_{p,j}^s}{\sigma^2\tilde\varepsilon_2^2}.
\end{aligned}
\end{equation*}
\end{claim}
Therefore, for any $d \geq \mathfrak{C}$, where $\mathfrak{C}$ is a constant only depending on $s, p$ and $\{a_{j}\}_{j \le p+1}$, we have
\begin{align}\label{lower case 2 eq 2}
    V_{K}\left(\tilde \epsilon_{2}, \mathcal{D} \right) &= V_{2}\left( \sqrt{2} \sigma \tilde \epsilon_{2}, \mathcal{B} \right) \\
    &\le K\left(\frac{\sqrt{2} \sigma \tilde \epsilon_{2}}{6}\right) \notag \\
    &\le (1+\varepsilon_0)\frac{1}{2}\sum_{j=1}^{N(d,p)}\ln\frac{18\mu_{p,j}^s}{\sigma^2\tilde\varepsilon_2^2} \notag \\
    &\le \left(1+\epsilon_{0} \right) \frac{1}{2} N(d,p)\ln\left(\frac{18 \mathfrak{C}_{2} d^{-ps}}{C_{2} \sigma^2  d^{-(p+1)s} }\right) \notag \\
    &= \left(1+\epsilon_{0} \right) \frac{1}{2} N(d,p) \left( \ln{\frac{18 \mathfrak{C}_{2}}{C_{2} \sigma^2 }} + s \ln{d} \right),
\end{align}
where we use Lemma \ref{claim:d_K_and_d_2} and Lemma \ref{lemma_entropy_of_RKHS} for the first line and Proposition \ref{prop:staircase} for the third line.

Using \eqref{lower case 1 eq 1} and \eqref{lower case 1 eq 2}, also recalling that we assume $ c_{1} d^{\gamma} \le n \le c_{2} d^{\gamma}$, we have
\begin{equation}\label{lower case 2 eq 3}
    \frac{V_K(\tilde\epsilon_2, \mathcal{D}) + n\tilde\epsilon_2^2 + \ln{2}}{V_2(\tilde\epsilon_1, \mathcal{B})} \le \frac{ \left(1+\epsilon_{0} \right) \frac{1}{2} N(d,p) \left( \ln{\frac{18 \mathfrak{C}_{2}}{C_{2} \sigma^2 }} + s \ln{d} \right) + c_{2} d^{\gamma} \cdot C_{2} d^{-(p+1)s} + \ln{2} }{ \frac{1}{2} N(d,p+1) \ln{\frac{\mathfrak{C}_{1}}{C_{1}}} }.
\end{equation}
The dominant terms in \eqref{lower case 2 eq 3} are:
\begin{equation}
    \frac{ c_{2} C_{2} d^{\gamma-(p+1)s}}{ \frac{1}{2} \ln{\frac{\mathfrak{C}_{1}}{C_{1}}} N(d,p+1)  \ln{d} }.
\end{equation}
Further noticing that $\gamma-(p+1)s \le p+1$ for any $ \gamma \in (p+ps+s, (p+1)+(p+1)s]$, so we can choose $ C_{2}$ small enough and only depending on $s, \sigma, \gamma, \kappa, c_{1}, c_{2} $, such that 
\begin{equation}
    \frac{V_K(\tilde\epsilon_2, \mathcal{D}) + n\tilde\epsilon_2^2 + \ln{2}}{V_2(\tilde\epsilon_1, \mathcal{B})} \le \ref{lower case 2 eq 3} := \mathfrak{c}< 1.
\end{equation}
Then using Lemma \ref{thm_lower_ultimate_tech} again, we have 
\begin{displaymath}
    \min _{\hat{f}} \max _{\rho_{f^{*}} \in \mathcal{D}} \mathbb{E}_{(x, \mathbf{y}) \sim \rho_{f^{*}}^{\otimes n}}
\left\|\hat{f} - f^{*}\right\|_{L^2}^2
\geq \frac{1 - \mathfrak{c}}{4} \tilde\epsilon_1^2 = \frac{1 - \mathfrak{c}}{4} C_{1} d^{-(p+1)s}.
\end{displaymath}
Further recalling that $ \mathcal{D} \subset \mathcal{P}$, we have 
\begin{equation}
    \min _{\hat{f}} \max _{\rho \in \mathcal{P}} \mathbb{E}_{(x, \mathbf{y}) \sim \rho^{\otimes n}} \left\|\hat{f} - f_{\rho}^{*}\right\|_{L^2}^2 \ge \min _{\hat{f}} \max _{\rho_{f^{*}} \in \mathcal{D}} \mathbb{E}_{(x, \mathbf{y}) \sim \rho_{f^{*}}^{\otimes n}} \left\|\hat{f} - f^{*}\right\|_{L^2}^2
\geq \frac{1 - \mathfrak{c}}{4} C_{1} d^{-(p+1)s}.
\end{equation}
We finish the proof of Theorem \ref{theorem lower bound} $\left( \text{\lowercase\expandafter{\romannumeral2}}\right)$.\\

\section{Experiments of Large Dimensional Gaussian Kernel}\label{sec:experiment}
We consider the distribution as $N(0,I_d)$, $k_d(x,x')=\exp(-\|x-x'\|^2_2/(2\ell^2d))$ for some positive constant $\ell$ and  the following two cases of true functions $f_\rho^*$:
\begin{itemize}
    \item[(i)] $f_\rho^*\in \mathcal{H}$. Following the setting of \cite{lu2023optimal}, we choose $f_\rho^*(x)=k_d(x,x_1)+k_d(x,x_2)+k_d(x,x_3)$, where $x_1,x_2,x_3$ are chosen randomly from $N(0,I_d)$.
    
    \item[(ii)] $f_\rho^* \in [\mathcal{H}]^s$, $s$ can be arbitrarily large. We choose $f_\rho^*(x)$ as the first eigenfunction of the large dimensional Gaussian kernel. We may denote $f_\rho^* \in [\mathcal{H}]^\infty$ without causing ambiguity.
\end{itemize}

We consider the following regression model $y=f_\rho^*(x)+\epsilon$, where $\epsilon\sim N(0,\sigma^2 I_d)$, $\sigma=0.1$. We draw $n=\lceil d^{1.5}\rceil$ samples from the model. The results are as follows.

\begin{figure}[h]
    \centering
    \subfigure[$f_\rho^*\in \mathcal{H}$, theoretical rate $=-1$]{\includegraphics[width=0.45\columnwidth]{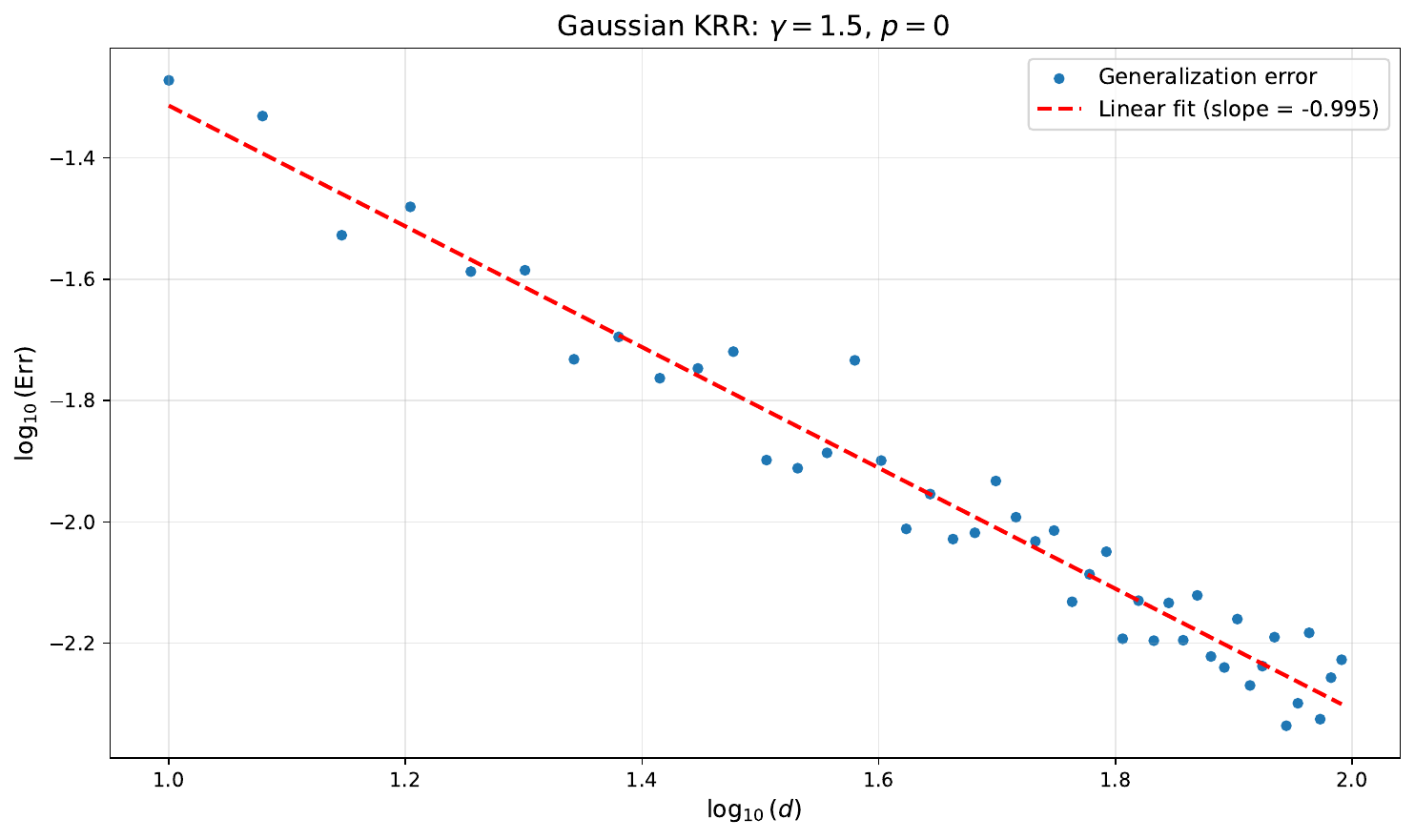}\label{experiment a}}
    \subfigure[$f_\rho^* \in \lbrack \mathcal{H} \rbrack^\infty$, theoretical rate $=-1.25$ ]{\includegraphics[width=0.45\columnwidth]{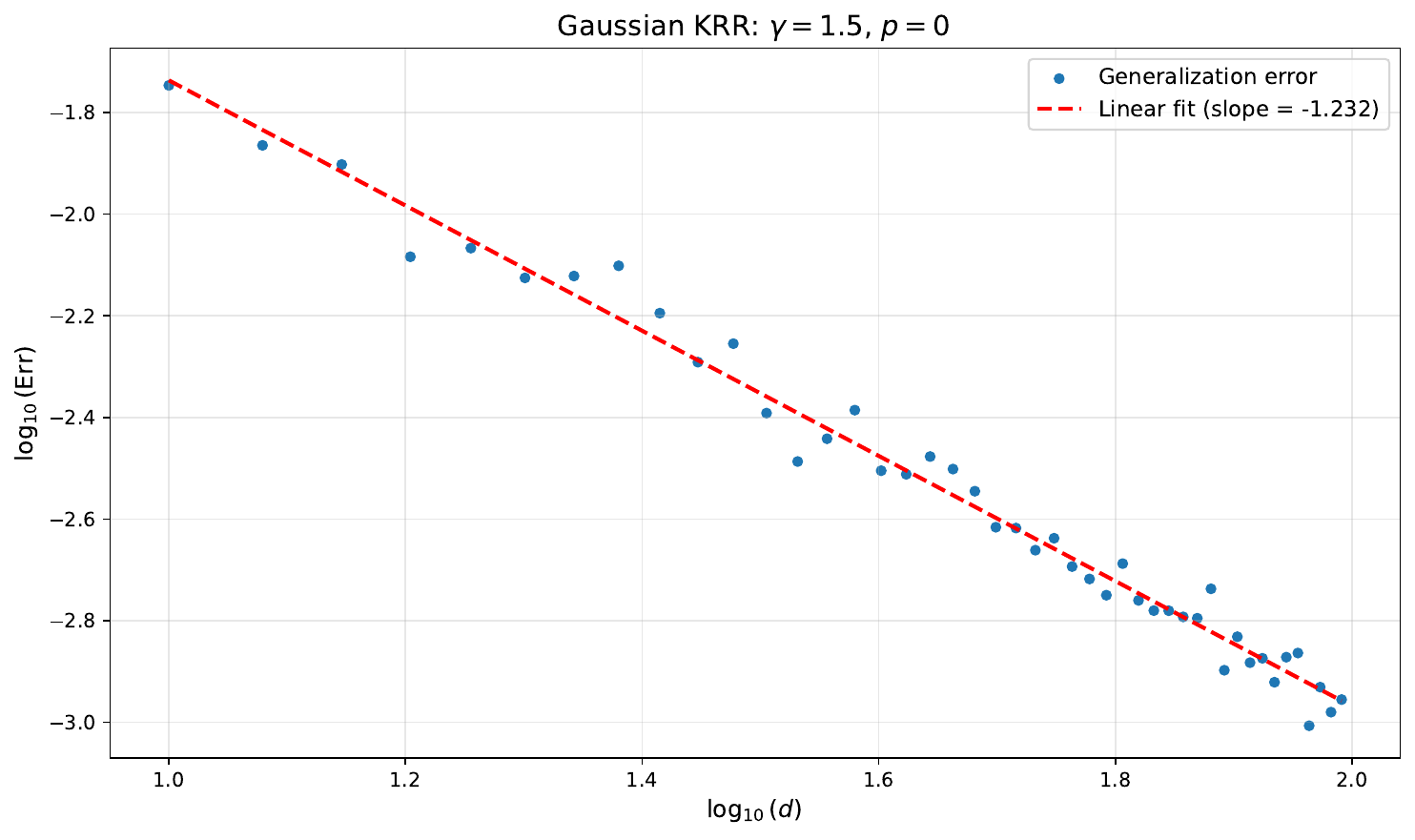}\label{experiment b}}
    \caption{Generalization error of KRR under large dimensional Gaussian kernel. Each data point of generalization error is the average of $30$ trials.}
    \label{experiment}
\end{figure}
As is shown in Figure \ref{experiment}, the empirical convergence rates align closely with the theoretical predictions in both cases. It can be observed that the convergence in Figure \ref{experiment a} is slower than that in Figure \ref{experiment b}, indicating that improved smoothness indeed enhances the rate of convergence. Moreover, the empirical rate in Figure \ref{experiment b} more closely approximates the theoretical exact rate of $-1.25$, rather than the minimax rate of $-1.5$, which suggests the presence of a saturation effect under this setting.

\section{Auxiliary Results}\label{section auxiliary}

The following proposition about estimating the $L^{2}$ norm with empirical norm is from \citet[Proposition C.9]{li2023_SaturationEffect}, which dates back to \cite{caponnetto2010cross}.
\begin{proposition}
    \label{prop:SampleNormEstimation}
    Let $\mu$ be a probability measure on $\mathcal{X}$, $f \in L^2(\mathcal{X},\mu)$ and $\|{f}\|_{L^\infty} \leq M$.
    Suppose we have ${x}_1,\dots,{x}_n$ sampled i.i.d.\ from $\mu$.
    Then for $\delta \in (0,1)$, the following holds with probability at least $1-\delta$:
    \begin{equation}\begin{aligned}
      \frac{1}{2}\|{f}\|_{L^2}^2 - \frac{5M^2}{3n}\ln \frac{2}{\delta} \leq \|{f}\|_{L^2,n}^2 \leq
      \frac{3}{2}\|{f}\|_{L^2}^2 + \frac{5M^2}{3n}\ln \frac{2}{\delta}.
    \end{aligned}\end{equation}
  \end{proposition}

The following concentration inequality about self-adjoint Hilbert-Schmidt operator
valued random variables is frequently used in related literature, e.g., \citet[Theorem 27]{fischer2020_SobolevNorm} and \citet[Lemma 26]{lin2020_OptimalConvergence}.
\begin{lemma}\label{lemma concentration of operator}
   Let $(\Omega, \mathcal{B}, P)$ be a probability space, $\mathcal{H}$ be a separable Hilbert space. Suppose that $ A_{1}, \cdots, A_{n}$ are i.i.d. random variables with values in the set of self-adjoint Hilbert-Schmidt operators. If  $\mathbb{E} A_{i} = 0$, and the operator norm $ \| A_{i} \| \le L, P \text {-a.e.}$, and there exists a self-adjoint positive semi-definite trace class operator $V$ with $\mathbb{E} A_{i}^{2} \preceq V $. Then for $\delta \in (0,1)$, with probability at least $1 - \delta$, we have 
   \begin{equation}\begin{aligned}
        \left\| \frac{1}{n}\sum_{i=1}^n A_i \right\|
        \leq \frac{2L\beta}{3n} + \sqrt {\frac{2 \| V \| \beta}{n}},\quad
        \beta = \ln \frac{4 \rm{tr} V}{\delta \| V \|}. \notag
   \end{aligned}\end{equation}
\end{lemma}

The following Bernstein inequality about vector-valued random variables is frequently used, e.g., \citet[Proposition 2]{caponnetto2007optimal} and \citet[Theorem 26]{fischer2020_SobolevNorm}.
\begin{lemma}[Bernstein inequality]\label{bernstein}
   Let $(\Omega,\mathcal{B},P)$ be a probability space, $H$ be a separable Hilbert space, and $\xi: \Omega \to H$ be a random variable with 
   \begin{displaymath}
     \mathbb{E}\|\xi\|_H^m \leq \frac{1}{2} m ! \sigma^2 L^{m-2},
   \end{displaymath}
   for all $m>2$. Then for $\delta \in (0,1)$, $\xi_{i}$ are i.i.d. random variables, with probability at least $1 - \delta$, we have
   \begin{displaymath}
       \left\|\frac{1}{n} \sum_{i=1}^n \xi_{i} - \mathbb{E} \xi\right\|_{H} \le 4\sqrt{2} \ln{\frac{2}{\delta}} \left(\frac{L}{n} + \frac{\sigma}{\sqrt{n}}\right).
   \end{displaymath}
\end{lemma}

\begin{lemma}\label{due embedding bound}
We have
   \begin{equation}\begin{aligned}\label{bound of Tk}
      \|T_{\lambda}^{-\frac{1}{2}} k({x},\cdot) \|_{\mathcal{H}}^{2} \le \frac{\kappa^2}{\lambda},~~ \mu \text {-a.e. } {x} \in \mathcal{X}.
  \end{aligned}\end{equation}
\end{lemma}
\begin{proof}
   \begin{equation}\begin{aligned}
      \|T_{\lambda}^{-\frac{1}{2}} k({x},\cdot) \|_{\mathcal{H}}^{2} &= \Big\| \sum\limits_{i\in N} ( \frac{1}{\lambda_{i} + \lambda})^{\frac{1}{2}} \lambda_{i} e_{i}({x}) e_{i}(\cdot)  \Big\|_{\mathcal{H}}^{2} \notag \\
      &=  \sum_{i \in N}  \frac{\lambda_{i}}{\lambda_{i} + \lambda} e_{i}^{2}({x}) \notag \\
      &\leq \sum_{i\in N}\frac{1}{\lambda}\lambda_ie_i^2(x)\\
      & \le \frac{\kappa^2}{\lambda}, \quad \mu \text {-a.e. } {x} \in \mathcal{X}. \notag  
  \end{aligned}\end{equation}
  The last inequality is due to the fact that $\sum_{i \in N}\lambda_ie_i^2(x)=k(x,x)\leq \sup_{x\in \mathcal{X}}k(x,x)\leq \kappa^2.$
\end{proof}

Lemma \ref{due embedding bound} has a direct corollary.
\begin{lemma}\label{emb norm}
Given the definition of $\mathcal{N}_{1}(\lambda)$ as in \eqref{n1 n2 m1 m2}. We have
\begin{displaymath}
    \| T_{\lambda}^{-\frac{1}{2}} T_{{x}} T_{\lambda}^{-\frac{1}{2}}\| \le \frac{\kappa^2}{\lambda}, \quad \mu \text {-a.e. } {x} \in \mathcal{X}.
\end{displaymath}
\end{lemma}
\begin{proof}
    Note that for any $f \in \mathcal{H}$,
  \begin{equation}\begin{aligned}
      T_{\lambda}^{-\frac{1}{2}} T_{{x}} T_{\lambda}^{-\frac{1}{2}} f &= T_{\lambda}^{-\frac{1}{2}} K_{{x}} K_{{x}}^{*}  T_{\lambda}^{-\frac{1}{2}} f \notag \\
      &= T_{\lambda}^{-\frac{1}{2}} K_{{x}} \langle k({x},\cdot), T_{\lambda}^{-\frac{1}{2}} f \rangle_{\mathcal{H}} \notag \\
      &= T_{\lambda}^{-\frac{1}{2}} K_{{x}} \langle T_{\lambda}^{-\frac{1}{2}} k({x},\cdot),  f \rangle_{\mathcal{H}} \notag \\
      &=  \langle T_{\lambda}^{-\frac{1}{2}} k({x},\cdot),  f \rangle_{\mathcal{H}} \cdot T_{\lambda}^{-\frac{1}{2}} k({x},\cdot). \notag
  \end{aligned}\end{equation}
  So $\| T_{\lambda}^{-\frac{1}{2}} T_{{x}} T_{\lambda}^{-\frac{1}{2}} \| = \sup\limits_{\| f\|_{\mathcal{H}}=1} \| T_{\lambda}^{-\frac{1}{2}} T_{{x}} T_{\lambda}^{-\frac{1}{2}} f\|_{\mathcal{H}} = \sup\limits_{\| f\|_{\mathcal{H}}=1} \langle T_{\lambda}^{-\frac{1}{2}} k({x},\cdot),  f \rangle_{\mathcal{H}} \cdot \|T_{\lambda}^{-\frac{1}{2}} k({x},\cdot) \|_{\mathcal{H}} = \|T_{\lambda}^{-\frac{1}{2}} k({x},\cdot) \|_{\mathcal{H}}^{2}$. 
  Using Lemma \ref{due embedding bound}, we finish the proof.
\end{proof}

The following theorem is borrowed from \citet[Theorem 42]{zhang2024optimal}. While they provided the proof of the theorem on compact set $\mathcal{X}$, we check their proof carefully and find that their proof can be easily extended to unbounded domains.
\begin{theorem}[$L^{q}$-embedding property]\label{integrability of Hs constants}
  Suppose that $\mathcal{H}$ is the RKHS associated with a continuous, positive-definite and symmetric kernel $k$ on $\mathcal{X}\subset \mathbb{R}^{d}$ and the probability distribution on $\mathcal{X}$ is $\mu$. Further suppose that $ \sup\limits_{x \in \mathcal{X}} |k(x,x)| \le \kappa^{2}$, where $ \kappa$ is an absolute constant. Then for any $0 < s < 1$, we have 
  \begin{align}
      [\mathcal{H}]^{s} \hookrightarrow L^{q_{s}}(\mathcal{X}, \mu),\quad \forall q_{s} < \frac{2}{1 - s}, 
  \end{align}
  and there exists a constant $C_{s,\kappa}$ only depending on $s$ and $\kappa$, such that the operator norm of the embedding operator satisfies
  \begin{equation}
      \left\| [\mathcal{H}]^{s} \hookrightarrow L^{q_{s}}(\mathcal{X}, \mu) \right\| \le C_{s,\kappa}.
  \end{equation}
\end{theorem}

\end{document}